%% file: main.tex
\documentclass[twoside]{article}

\usepackage[accepted]{aistats2022}
\usepackage[round]{natbib}
\renewcommand{\bibname}{References}
\renewcommand{\bibsection}{\subsubsection*{\bibname}}

\input{math_commands.tex}

\usepackage{microtype}
\usepackage{graphicx}
\usepackage{subfigure}
\usepackage{booktabs} 
\usepackage{xcolor}
\usepackage[utf8]{inputenc} 
\usepackage[T1]{fontenc}    
\usepackage{hyperref}       
\usepackage{url}            
\usepackage{booktabs}       
\usepackage{amsfonts}       
\usepackage{amsmath}
\usepackage{mathtools}
\usepackage{amsthm}
\usepackage{amssymb}
\usepackage{nicefrac}       
\usepackage{microtype}      
\usepackage{multirow}
\usepackage{colortbl}
\usepackage{array}
\usepackage{booktabs}
\usepackage{algorithm}
\usepackage[noend]{algpseudocode}
\usepackage{cleveref}
\usepackage{xspace}
\newcommand{\calS}{\mathcal{S}}
\usepackage{wrapfig}
\usepackage[textwidth=6cm]{todonotes}
\usepackage{soul}

\algrenewcommand\algorithmicrequire{\textbf{Input:}}
\algrenewcommand\algorithmicensure{\textbf{Output:}}

\newtheorem{definition}{Definition}
\newtheorem{theorem}{Theorem}
\newtheorem*{theorem*}{Theorem}
\newtheorem{lemma}{Lemma}

\newtheorem{corollary}{Corollary}
\newtheorem{assumption}{Assumption}
\DeclarePairedDelimiterX{\inp}[2]{\langle}{\rangle}{#1, #2}
\DeclareMathOperator*{\lrnlocal}{LRN}

\newcommand{\norm}[1]{\left\lVert#1\right\rVert}
\newcommand{\order}[1]{\mathcal{O}\left(#1\right)}
\newcommand\Tstrut{\rule{0pt}{2.6ex}}         

\newcommand{\algname}{{FedBuff}\xspace}
\newcommand{\lrn}{{LR-Norm}\xspace}
\newcommand{\concurrency}{{concurrency}\xspace}
\newcommand{\asyncfl}{{AsyncFL}\xspace}
\newcommand{\syncfl}{{SyncFL}\xspace}
\newcommand{\secagg}{{SecAgg}\xspace}

\definecolor{babyblue}{rgb}{0.63, 0.79, 0.95}

\begin{document}

%

%
\runningauthor{J. Nguyen, K. Malik,  H. Zhan, A. Yousefpour, M. Rabbat, M. Malek, D. Huba}

\twocolumn[

\aistatstitle{Federated Learning with Buffered Asynchronous Aggregation}

\aistatsauthor{ John Nguyen \And Kshitiz Malik  \And  Hongyuan Zhan \And Ashkan Yousefpour}
\aistatsauthor{Mike Rabbat \And Mani Malek \And Dzmitry Huba }

\aistatsaddress{ Meta AI} ]

\begin{abstract}
Scalability and privacy are two critical concerns for cross-device federated learning (FL) systems. In this work, we identify that synchronous FL --- synchronized aggregation of client updates in FL --- cannot scale efficiently beyond a few hundred clients training in parallel. It leads to diminishing returns in model performance and training speed, analogous to large-batch training. On the other hand, asynchronous aggregation of client updates in FL (i.e., asynchronous FL) alleviates the scalability issue. However, aggregating individual client updates is incompatible with Secure Aggregation, which could result in an undesirable level of privacy for the system. To address these concerns, we propose a novel buffered asynchronous aggregation method, \algname, that is agnostic to the choice of optimizer, and combines the best properties of synchronous and asynchronous FL. We empirically demonstrate that \algname is 3.3$\times$ more efficient than synchronous FL and up to 2.5$\times$ more efficient than asynchronous FL, while being compatible with privacy-preserving technologies such as Secure Aggregation and differential privacy. We provide theoretical convergence guarantees in a smooth non-convex setting. Finally, we show that under differentially private training, \algname can outperform FedAvgM at low privacy settings and achieve the same utility for higher privacy settings.
\end{abstract}

\section{Introduction}
\label{sec:intro}
\input{introduction}

\section{Background}
\label{sec:background}
\input{background}

\section{\algname{}: Federated Learning with Buffered Asynchronous Aggregation}
\label{sec:alg}
\input{protocol}

\section{Convergence Analysis}
\label{sec:convergence}
\input{convergence}

\section{Practical Improvements}
\label{sec:improvement}
\input{improvements}

\section{Experiments}
\label{sec:experiment}
\input{experiments}


\section{Related Work}
\label{sec:related}
\input{related}

\section{Conclusions}
\label{sec:conclusions}
\input{conclusions}

\subsubsection*{Acknowledgements}
We would like to thank Ilya Mironov, Maziar Sanjabi, Graham Cormode, Samuel Horvath and Luca Melis for the meaningful discussions and their valuable suggestions which significantly improved the quality of this paper. We would like to also thank the anonymous reviewers for their insightful feedback. 

\bibliography{reference}
\bibliographystyle{abbrvnat}

\clearpage
\onecolumn
\appendix
\section*{Appendix}
\input{appendix}


\end{document}

%% file: math_commands.tex
\usepackage{amsmath,amsfonts,bm}

\def\eqref#1{equation~\ref{#1}}

\def\1{\bm{1}}

\DeclareMathAlphabet{\mathsfit}{\encodingdefault}{\sfdefault}{m}{sl}
\SetMathAlphabet{\mathsfit}{bold}{\encodingdefault}{\sfdefault}{bx}{n}

\newcommand{\E}{\mathbb{E}}

\newcommand{\R}{\mathbb{R}}



%% file: introduction.tex
Federated Learning (FL) is a distributed learning paradigm that aims to train a shared model across participants while training data stays on the participant devices. In this work, we focus on cross-device FL where participants are edge devices (\cite{fl-survey}), and in particular, aim to address the following two challenges:



\textbf{Challenge 1: Scalability.} In large-scale cross-device FL settings, the number of clients can be in the millions, and only a small fraction of the client population may be available at any given time for training (\cite{fl_field_guide}). Additionally, client devices may have limited communication bandwidth and compute power. In these settings, an important parameter is \emph{\concurrency}: the number of clients training concurrently (i.e., \emph{clients-per-round} or \emph{cohort size}). 
There is a fundamental limitation when increasing concurrency in synchronous FL training: a diminishing return in the speed and quality of training. In this paper, we propose a novel buffered asynchronous aggregation optimization that makes it possible to train using significantly higher concurrency, improving the performance and efficiency of~FL.

\textbf{Challenge 2: Privacy.}
Inference attacks, methods trying to recover information from gradients, can expose sensitive information about the participating clients \citep{melis2019exploiting, geiping2020inverting}. Given this privacy concern, secure aggregation (SecAgg) (\cite{secagg-sgx, secagg}) and differential privacy (DP) (\cite{dpftrl, google-dp-fl}) provide protection against inference attacks \citep{lauren_alex, extraction}. Using SecAgg, an honest-but-curious server cannot see the individual client updates, while DP can protect clients’ data from observations based on the inputs and the output of the computation. With SecAgg, DP clipping and noise addition can be performed on server, providing a better privacy-utility trade-off. For many real-world cross-device FL applications, compatibility with such privacy enhancing technologies is vital.  




\textbf{Our proposal: \algname{}.} Motivated by these challenges, we propose and analyze \algname, a novel asynchronous federated optimization framework using buffered asynchronous aggregation. In \algname, clients train and communicate asynchronously with the server. Unlike other asynchronous methods, the server aggregates $K$ client updates in a secure buffer before performing a server update. This secure buffer can be implemented by using Trusted Execution Environments (TEEs) \citep{secagg-sgx, tee_ppfl}. 




\textbf{Contributions.}
We highlight the main contributions: 

    $\bullet$ We propose \algname{}, a novel asynchronous federated optimization framework with buffered asynchronous aggregation to achieve scalability and privacy against the \textbf{honest-but-curious} threat model through secure aggregation and differential privacy.
     
    $\bullet$ We provide a convergence analysis for \algname in the smooth non-convex setting. When clients take $Q$ local SGD steps, \algname requires $\order{1/(\epsilon^2 Q)}$ server iterations to reach $\epsilon$ accuracy (Section \ref{sec:convergence}).
    
    $\bullet$ Empirically, we show that \algname{} is up to 3.8$\times$ more efficient than competing synchronous FL algorithms,  even without penalizing synchronous FL algorithms for stragglers. We also demonstrate that \algname{} is up to 2.5$\times$ more efficient than the closest asynchronous FL algorithm in the literature, FedAsync \citep{fedasync}. Our extensive empirical evaluation finds that $K = 10$ is a good setting across benchmarks and does not require tuning. 

    $\bullet$ To the best of our knowledge, we are the first to propose an asynchronous federated optimization framework that is compatible with SecAgg and global user-level DP. Under differentially private training, \algname can outperform both synchronous FL with amplified DP-SGD and DP-FTRL (differentially private Follow-the-Regularized-Leader) at low privacy settings, and be competitive for high privacy settings. 


\begin{figure*}[t]
     \centering
     \begin{minipage}[t]{0.32\linewidth}
         \centering
         \includegraphics[width=\linewidth]{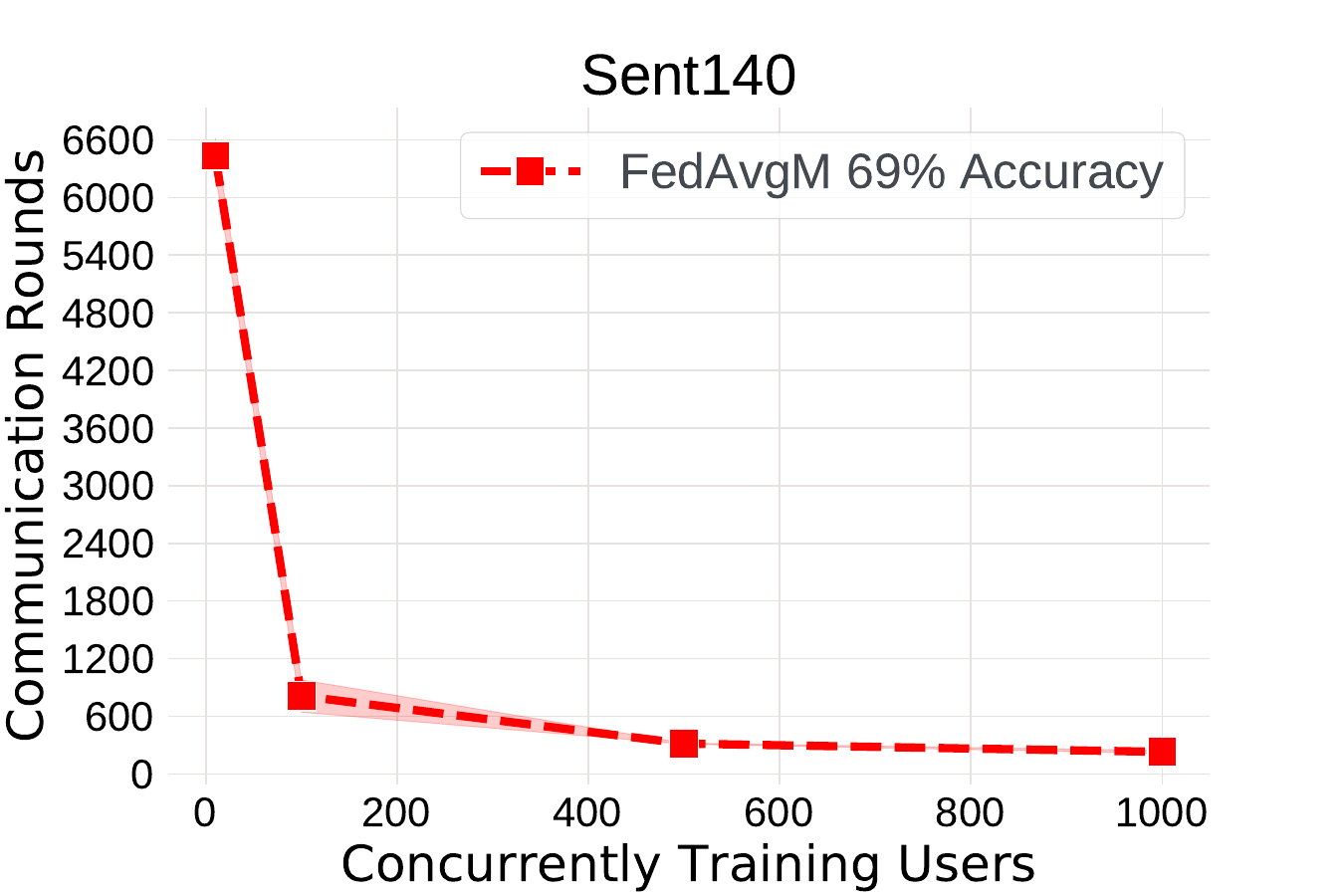}
     \end{minipage}
     \hspace{1em}
     \begin{minipage}[t]{0.32\linewidth}
        \centering
        \includegraphics[width=\linewidth]{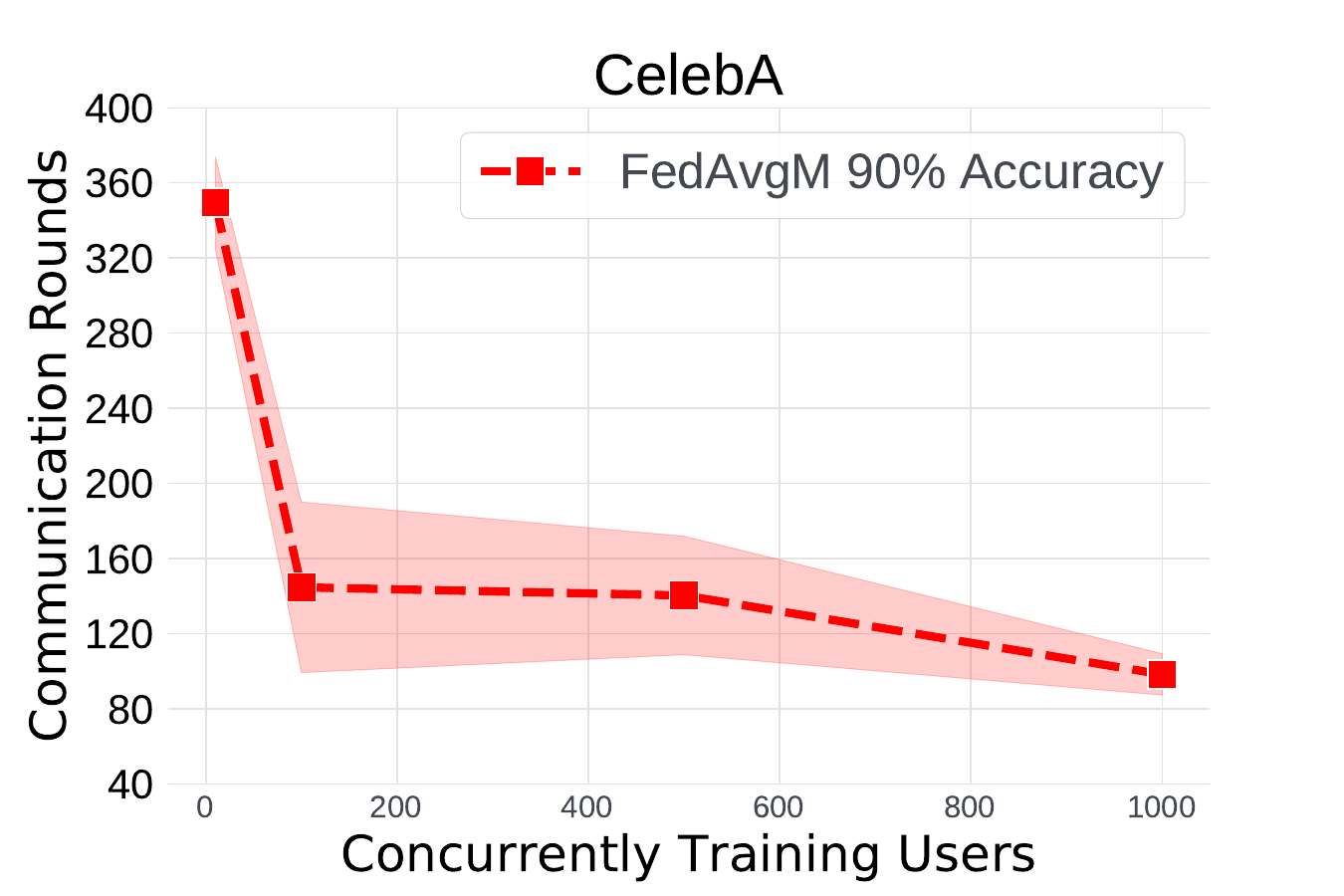}
     \end{minipage}
        \begin{minipage}[t]{0.32\linewidth}
        \centering
        \includegraphics[width=\linewidth]{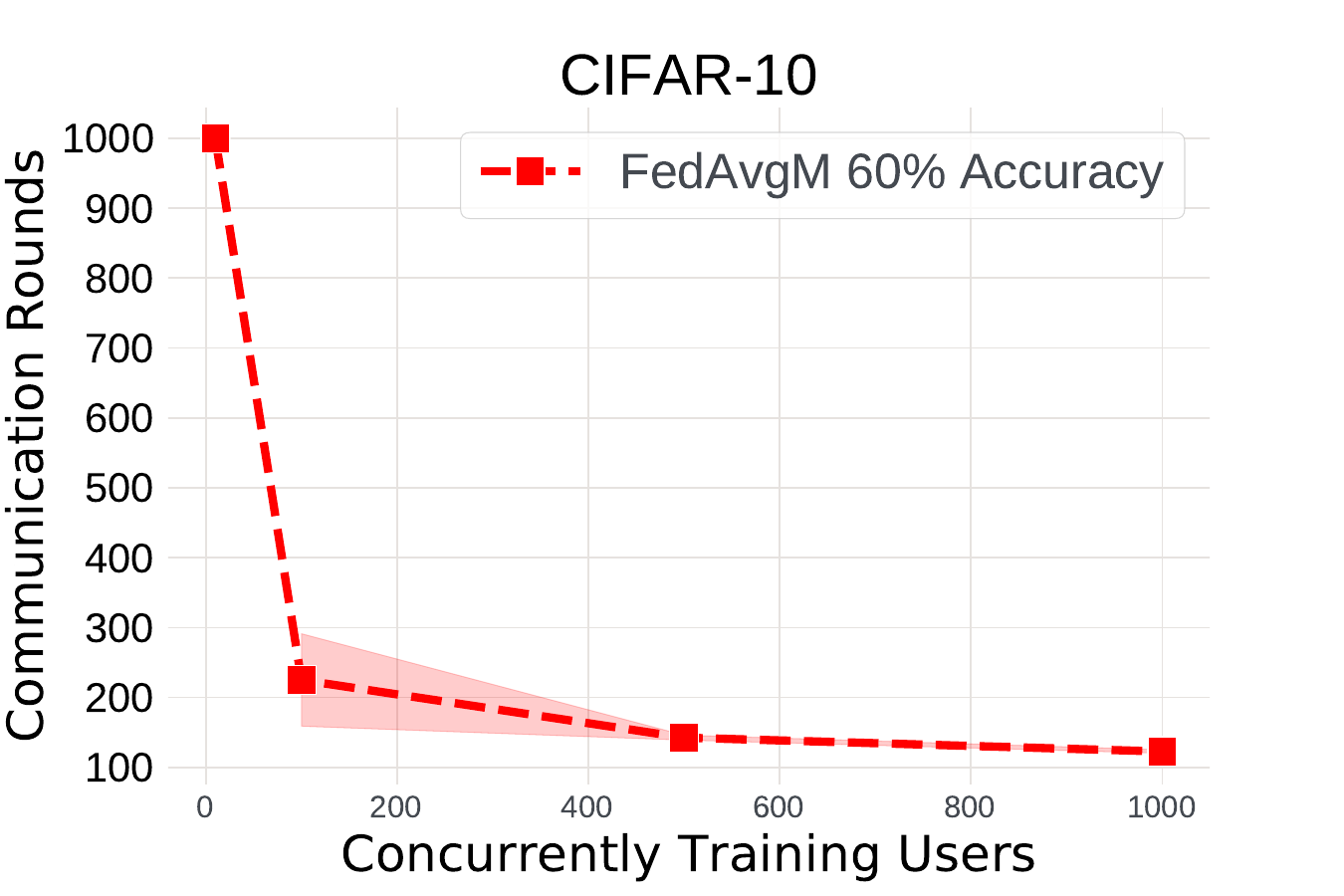}
     \end{minipage}
     \caption{The number of communication rounds to reach a target accuracy with varying levels of concurrency.
        SyncFL algorithms such as FedAvgM \citep{fedavgm} shows diminishing returns from
        increasing concurrency beyond 100. For example, increasing concurrency by 10x (100 –> 1000)
        decreases the number of communication rounds by less than 2x. This is analogous to large-batch
        training, where increasing the batch size eventually gives diminishing returns.}
     \label{fig:sent140_rounds}
\end{figure*}

\begin{figure}[t]
    \begin{minipage}[t]{\linewidth}
        \centering
	\includegraphics[width=\linewidth]{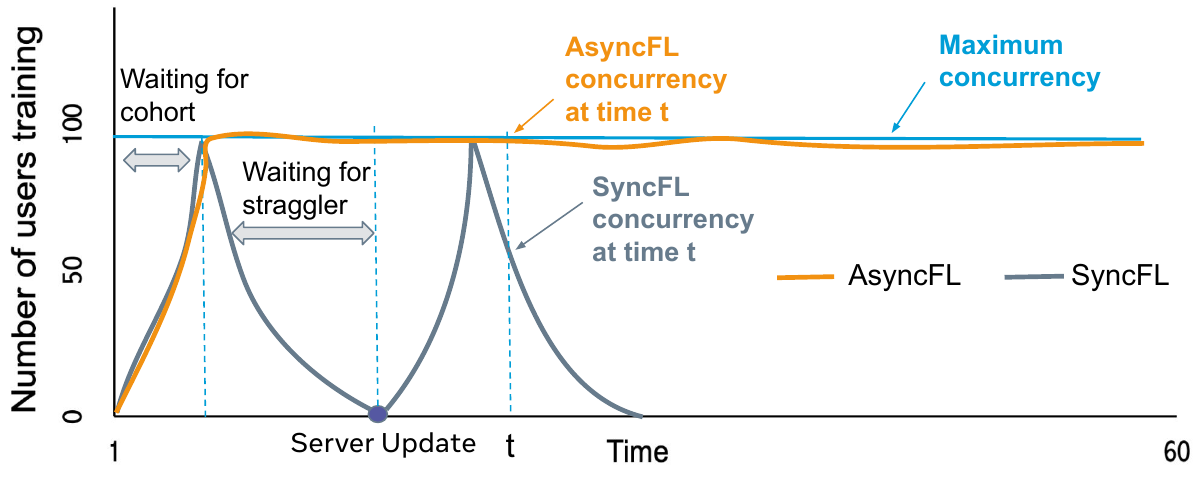}
    \end{minipage}

    \caption{Training progress for asynchronous and synchronous FL, and the associated delays. Synchronous FL proceeds in \emph{rounds}. The number of active clients increases at the beginning of a round as clients join the cohort, and it falls gradually towards the end of the round due to stragglers. In asynchronous FL, the number of active clients stays relatively constant over time; as clients finish training and upload their results, other clients take their place.}
	\label{fig:concurrency_sync_async}	
\end{figure}

%% file: background.tex
\textbf{Synchronous FL. }
Significant attention has been paid towards synchronous FL methods (\syncfl), as they are perhaps easier to
analyze and implement. \syncfl methods are also better suited for privacy -- training and aggregating updates over a large number of clients render most inference attacks ineffectual \citep{melis2019exploiting, zhu2020deep, geiping2020inverting, lam2021gradient}. However, synchronous FL methods are prone to stragglers, proceeding at the pace of the slowest client. \cite{google-fl} proposed using over-selection to tap 30\% more clients than the target cohort size and wait for the fastest replies to overcome this issue. However, over-selection comes at the cost of wasting clients' resources and introduces selection bias. We study these problems in Appendix~\ref{sec:appendix_stragglers} and~\ref{sec:appendix_diversity}.

In \syncfl optimization, FedAvg, a generalization of local SGD, has been shown to work well empirically \citep{google-fedavg}. FedProx \citep{fedprox} improves upon FedAvg by adding a proximal term~$\mu$ to the local SGD optimizer. FedAvgM \citep{fedavgm} further improves convergence by adding server-side momentum. Adaptive methods such as FedAdam \citep{adaptive-fl-optimization} are effective in cross-device FL settings and have comparable performance to FedAvgM. These optimizers often focus on heterogeneity, asymptotic convergence, and communication efficiency in \emph{low \concurrency} settings. In this paper, we focus on \emph{high \concurrency} settings. In these settings, \syncfl is not scalable and is inefficient. Typically, the optimal server learning rate increases with \concurrency{}; aggregating over more users has a variance-reducing effect, enabling the server to take larger steps. Consequently, higher \concurrency reduces the number of rounds needed to reach a target accuracy because of a larger server learning rate. However, to have stable, convergent training dynamics, the server learning rate cannot be increased indefinitely; eventually it saturates, resulting in a sub-linear speed-up similar to in large-batch training (\cite{imagenet-1hr, scaling_nmt, bert_in_76, imagenet_minutes, lars, measuring_data_parallelism}). 
As a result, \syncfl systems cannot accelerate training through parallelism beyond a few hundred clients and exhibit decreasing efficiency with increasing \concurrency (Figure \ref{fig:sent140_rounds}). 

\textbf{Asynchronous FL.}
Asynchronous FL methods are a good match for cross-device FL settings, where clients have different compute power and intermittent availability \citep{fl_field_guide}. Most asynchronous FL (\asyncfl) works, such as the works \citep{fedasync, paper3, paper2, paper1, safa, li2021stragglers} have been focused on solving the straggler problem by designing asynchronous FL algorithms. However, these proposals include aspects that make them impractical for real-world FL deployment at scale. For instance, \cite{paper2, li2021stragglers} profile client speed, \cite{paper1} broadcast the model updates to all clients, \cite{fedasync} update the server model on every client, placing a significant burden on the clients and server, and \cite{paper3} assume all clients have the same speed. 

In \emph{``fully''} \asyncfl methods (e.g., \cite{fedasync}), every client update results in a server model update. This has implications for privacy and scalability. Considering privacy, when every client update forces a server update, \secagg cannot be used; secure aggregation's benefit is in hiding individual updates by combining them in an aggregate. Additionally, providing user-level DP in \asyncfl is only feasible with local differential privacy (LDP), where the client clips the model update and adds noise locally to it before sending it to the server. LDP for high dimensional data has been criticized for poor privacy-utility trade-off (\cite{esa_empirical, prochlo}). 


\textbf{Secure Aggregation.}
SecAgg is a privacy enhancing technology based on cryptographic primitives \citep{secagg, bell2020secure, so2021turbo} or hardware-based Trusted Execution Environment (TEE)~\citep{secagg-sgx}. SecAgg enhances privacy by obfuscating a client's update with many other clients' updates, protecting against the honest-but-curious server threat model~\citep{secagg}. \algname is compatible with SecAgg.

\textbf{Differential Privacy.}
DP \citep{dp} provides a rigorous formulation of the release of information derived from private data. In the context of machine learning, differentially private training \citep{dp-sgd} limits what can be learned about the original training data. 
\begin{definition}
    A randomized mechanism M: $U \mapsto R$ satisfies ($\epsilon, \delta$)-DP if for all adjacent datasets $D, D' \in U$ and for any subset of outputs $S \subseteq R$, the following holds:
    \[
    Pr[M(D) \in S] \leq e^{\epsilon} \cdot Pr[M(D') \in S] + \delta.
    \]
\end{definition}
The definition of adjacent datasets $D,D'$ is domain and application dependent. In the context of cross-device FL, we consider $D,D'$ to be two datasets of training examples, where each example is associated with a client. Then, $D$ and $D'$ are adjacent if $D'$ can be formed by adding or removing all of the examples associated with a single client from $D$ (i.e., \textit{user-level privacy} \citep{google-dp-fl}). In this paper, we consider the \emph{global DP} (GDP) setting, where a trusted server collects, clips, and aggregates the client updates. The server then adds noise to the aggregated updates. Compared to LDP, GDP provides a better privacy-utility trade-off for high dimensional data. This setting of DP relies on using SecAgg, as the server is responsible for implementing DP.


%% file: protocol.tex
We consider the following optimization problem:
\begin{equation}
    \min_{w \in \mathbb{R}^d} f(w) := \frac{1}{m} \sum_{i=1}^m p_i F_i(w)
\end{equation}
where $m$ is the total number of clients and the function $F_i$ measures the loss of a model with parameters $w$ on the $i$th client's data, and $p_i > 0$ weighs the importance of the data from client $i$. The goal is to find a model that fits all clients' data well on (weighted) average. In FL, $F_i$ is only accessible by client $i$. 

\syncfl methods need to aggregate and synchronize clients after each round. Hence, concurrency in \syncfl is equal to the number of clients that participate in a given round. In asynchronous methods, concurrency is the number of clients training at a given point in time (Figure \ref{fig:concurrency_sync_async}). In \algname{} (Algorithm \ref{alg:server}), clients enter and finish local training asynchronously. However, the server model is not updated immediately upon receiving every client update. Instead, client updates are stored in a \textit{buffer}. A server update only takes place once $K$ client updates are in the buffer, where $K$ is the size of the buffer and is a tunable parameter. However, we find that $K = 10$ is a good choice and does not require tuning. The buffer can be implemented by using a Trusted Execution Environment (TEE) \citep{secagg-sgx, tee_ppfl} or through a cryptographic algorithm \citep{so_buffsecagg}. 
Note that $K$ is independent of concurrency --- the extra degree of freedom introduced by the buffer allows the server to choose the model update frequency instead of coupling concurrency with the server model update as in \syncfl. The extra degree of freedom  allows \algname{} to achieve data efficiency at high concurrency while being compatible with secure aggregation and DP.

\algname is compatible with SecAgg, because with $K>1$ updates in the buffer, SecAgg provides its promise by hiding individual updates in the aggregate. Since \algname supports SecAgg, it can be easily extended to provide global DP. In asynchronous FL settings, the server has no control over which clients participate in a particular model update and client availability is dynamic. For such settings, privacy amplification by sampling is not feasible. DP-FTRL \citep{dpftrl} has emerged as a suitable solution to address this issue. \algname with DP-FTRL is straightforward and we show in Algorithm \ref{alg:server} how one can extend \algname to provide global DP. The three functions, \textit{InitializeTree}, \textit{AddToTree}, and \textit{GetSum} in Algorithm \ref{alg:server} correspond to those of the DP-FTRL algorithm. We defer to Section B.1 in \citep{dpftrl} for more in-depth descriptions of the functions. 



%% file: convergence.tex
In this section, we provide a convergence guarantee for \algname{} in the smooth, non-convex setting. Most previous works analyze synchronous federated learning methods, such as the works in \citep{dont_use_minibatch, fedprox, adaptive-fl-optimization, fedavg_conv_Li, Stitch_LocalSGD, restarted_SGD_Yu, li2019communication, haddadpour_localdescent, SCAFFOLD}. In contrast, in \algname{}, clients train asynchronously, and the client updates are first aggregated in a buffer before producing a server model update. Hence, it is essential to understand the relationship between client computation and server communication under asynchrony with buffered aggregation. 

\textbf{Notation.} We use the following notation throughout: $[m]$ represents the set of all client indices, $\nabla F_i(w)$ denotes the gradient with respect to the loss on client $i$'s data, $f^*$ denotes the minimum of $f(w)$, $g_i(w; \zeta_i)$ denotes the stochastic gradient on client $i$, $K$ is the buffer size for aggregation before producing each server update, and $Q$ denotes the number of local steps taken by each client. We make the following assumptions throughout.

\begin{algorithm}[t]
  \caption{\texttt{\algname{}-server}}\label{alg:server}
    \begin{algorithmic}[1]
    \Require server learning rate $\eta_g$, client learning rate $\eta_{\ell}$, client SGD steps $Q$, buffer size $K$ 
    \colorbox{babyblue}{$n$ dataset size, noise scale $\sigma^2$, clip norm $L$ (DP) }
    \Ensure FL-trained global model
    \State \colorbox{babyblue}{$\mathcal{T} \leftarrow InitializeTree(n, \sigma^2, L)$ (DP)}
    \Repeat
    \State{$c \leftarrow$ sample available clients \Comment{async}}
    \State{run \texttt{\algname{}-client}$(w^t, \eta_{\ell}, Q)$ on $c$ \Comment{async}} 
    \If{receive client update}
    \State{$\Delta_i \leftarrow$ received update from client $i$}
    \State{\colorbox{babyblue}{$\Delta_i \leftarrow Clip(\Delta_i, L)$ (DP)} \Comment{in TEE}}
    \State{$\overline{\Delta}^t \leftarrow \overline{\Delta}^t + \Delta_i$ \Comment{in TEE}}
    \State{$k \leftarrow k + 1$} 
    \EndIf
    \If{$k == K$} 
    \State {$\overline{\Delta}^t \leftarrow \frac{\overline{\Delta}^t}{K}$}
    \State \colorbox{babyblue}{$\mathcal{T} \leftarrow \text{AddToTree}(\mathcal{T}, t, \overline{\Delta}^t)$ (DP) \Comment{in TEE}}
    \State \colorbox{babyblue}{$\overline{\Delta}^t \leftarrow \overline{\Delta}^t + \text{GetSum}(\mathcal{T}, t)$ (DP) \Comment{in TEE}}
    \State $w^{t+1} \leftarrow w^t - \eta_g \overline{\Delta}^t$
    \State  $\overline{\Delta}^t \leftarrow 0, k \leftarrow 0, t \leftarrow t + 1$ \Comment{reset buffer}
    \EndIf
    \Until{Convergence}
    \end{algorithmic}
\end{algorithm}

\begin{algorithm}[t]
  \caption{\texttt{\algname{}-client}}\label{alg:client}
    \begin{algorithmic}[1]
    \Require server model $w$, client learning rate $\eta_{\ell}$, number of client SGD steps $Q$
    \Ensure client update $\Delta$
     \State{$y_0 \leftarrow w$ }
     \For{$q=1:Q$}
     \State{$y_q \leftarrow y_{q-1} - \eta_{\ell} g_q(y_{q-1} )  $}
     \EndFor
     \State{$\Delta \leftarrow y_0 - y_q$}
     \State{Send $\Delta$ to server}
    \end{algorithmic}
\end{algorithm}

\begin{assumption}
(Unbiasedness of client stochastic gradient) $\mathbb{E}_{\zeta_i}[g_i(w ; \zeta_i))] = \nabla F_i(w)$.
\label{assumption:unbiased}
\end{assumption}
\begin{assumption}
(Bounded local and global variance) for all clients $i \in [m]$, $$\mathbb{E}_{\zeta_i|i}[\norm{g_i(w; \zeta_i) - \nabla F_i(w)}^2] \leq \sigma^2_{\ell},$$
and 
$$\frac{1}{m} \sum_{i=1}^m \norm{\nabla F_i(w) - \nabla f(w)}^2 \leq \sigma^2_{g}.$$ 
\label{assumption:bounded_var}
\end{assumption}

\begin{assumption}
(Bounded gradient) $\norm{\nabla F_i}^2 \leq G$ for all $i \in [m]$.
\label{assumption:bounded_grad}
\end{assumption}

\begin{assumption}
(Lipschitz gradient) for all client $i \in [m]$, the gradient is $L$-smooth, 
$$\norm{\nabla F_i(w) - \nabla F_i(w')}^2 \leq L \norm{w - w'}^2. $$
\label{assumption:lipz}
\end{assumption}
Assumptions~\ref{assumption:unbiased}--\ref{assumption:lipz} are commonly made in analyzing federated learning algorithms (\cite{adaptive-fl-optimization, fedavg_conv_Li, Stitch_LocalSGD, restarted_SGD_Yu}). We make an additional assumption on the staleness under asynchrony.

\begin{assumption}
(Bounded Staleness when $K=1$) For all clients $i \in [m]$ and for each server step $t$, the staleness $\tau_i(t)$ between the model version in which \texttt{FedBuff-client} uses to start local training, and the model version in which $\Delta^i$ is used to modify the global model is not larger than $\tau_{\max,1}$ when $K=1$. 
\label{assumption:staleness1}
\end{assumption}

\textbf{Remark.} More generally, the staleness upper-bound depends on the buffer size $K$. When the buffer size increases, the server iterates are updated less frequently, hence reducing the number of server steps in between the initialization of client training and when the client updates are used for modifying the server model. Specifically, if Assumption 5 holds, then for any execution of FedBuff with $K>1$, the maximum delay $\tau_{\max,K}$ is at most $\lceil \tau_{\max, 1} / K \rceil$; see Appendix~\ref{appendix:staleness-K}.

\begin{theorem}
Let $\eta^{(q)}_{\ell}$ be the local learning rate of client SGD in the $q$-th step, and define $\alpha(Q):=\sum_{q=0}^{Q-1} \eta^{(q)}_{\ell}$, $\beta(Q):=\sum_{q=0}^{Q-1} (\eta^{(q)}_{\ell})^2$. Choosing $\eta_g \eta^{(q)}_{\ell} Q \leq \frac{1}{L}$ for all local steps $q=0,\cdots,Q-1$, the global model iterates in Algorithm \ref{alg:server} achieves the following ergodic convergence rate
\begin{equation}
    \begin{aligned}
        \frac{1}{T} \sum_{t=0}^{T-1} \norm{\nabla f(w^t)}^2 
         &\leq \frac{2 \Big(f(w^0) - f^* \Big)}{\eta_g \alpha(Q) T} +\frac{L}{2}\frac{\eta_g \beta(Q) }{ \alpha(Q)}   \sigma^2_{\ell} \\ + 3 L^2 Q  & \beta(Q)  \Big(\eta^2_g \tau_{\max, K}^2 + 1 \Big)  \Big(\sigma^2_{\ell} + \sigma^2_g +  G\Big).
    \end{aligned}
\end{equation}
\label{thm:main_ergodic_rate}
\end{theorem}
The proof of Theorem~\ref{thm:main_ergodic_rate} is provided in Appendix \ref{appendix:proof}, and leverages ideas from the perturbed iterate framework~\citep{AsyncSGD_purturbed}.
\begin{corollary}
Choosing constant local learning rate $\eta_{\ell}$ and $\eta_g$ such that $\eta_g \eta_{\ell} Q \leq \frac{1}{L},
$ the global model iterates in \algname{} (Algorithm \ref{alg:server}) are bounded by 
\begin{equation}
\begin{aligned}
        \frac{1}{T} \sum_{t=0}^{T-1} \mathbb{E}\left[\norm{\nabla f(w^t)}^2\right]
        &\leq \frac{2 F^*}{\eta_g \eta_{\ell} QT} +\frac{L}{2}\eta_g \eta_{\ell} \sigma_{\ell}^2 \\  &+ 3 L^2 Q^2\eta_{\ell}^2\Big(\eta^2_g  \tau_{\max, K}^2 + 1 \Big)  \sigma^2, 
         \end{aligned}
\end{equation}
where $F^* := f(w^0) - f^*$ and $\sigma^2 := \sigma^2_{\ell} + \sigma^2_g + G$. Further, choosing $\eta_{\ell} = \order{ 1/\left(K\sqrt{TQ}\right)}$, $\eta_g = \order{K}$, for all $\eta_g, \eta_{\ell}$ satisfying $\eta_g \eta_{\ell} Q \leq \frac{1}{L}$ and sufficiently large $T$, we have
\begin{equation}
\begin{aligned}
    & \frac{1}{T} \sum_{t=0}^{T-1} \mathbb{E}\left[\norm{\nabla f(w^t)}^2\right]
    \leq \order{ \frac{F^*}{\sqrt{TQ}} } \\ &+ \order{ \frac{ \sigma_{\ell}^2}{\sqrt{TQ}} }+ \order{ \frac{Q \sigma^2}{TK^2} } + \order{ \frac{ Q \sigma^2 \tau^2_{\max, 1}}{TK^2} },
     \end{aligned}
    \label{eq:constantLR_bound}
\end{equation}
\label{corollary:constant_LR_bound}
where we use the relation $\tau_{\max, K} \leq \lceil \tau_{\max, 1} / K \rceil$.
\end{corollary}

Corollary \ref{corollary:constant_LR_bound} yields several insights:

\textbf{Worst-case iteration complexity. } Theorem \ref{thm:main_ergodic_rate} bounds the ergodic norm-squared of the gradient, a standard quantity studied in non-convex stochastic optimization. If this becomes small as $T$ grows, then it must be that the norm-squared of the gradient at later iterations is vanishing, implying the algorithm is converging towards a first-order stationary point. The bound in equation (\ref{eq:constantLR_bound}) contains three terms. The first is standard, expressing how the initialization impacts convergence, and it decreases at a rate of  $\order{1 / T}$ as is standard for SGD. The second two terms depend on different sources of variance, due to heterogeneity of functions at different clients ($\sigma^2_{g}$), stochasticity of gradients ($\sigma^2_{\ell} + G$), and stale gradients due to delays in asynchronous execution ($\tau_{\max, K}$). Corollary \ref{corollary:constant_LR_bound} is derived from Theorem \ref{thm:main_ergodic_rate} under a specific choice of constant learning rate, and yields interpretation of trade-offs between the convergence of loss, local and global variance, effect of client drift due to local steps, effect of staleness and effect of buffer size. We summarize the these trade-offs next. 


\textbf{Total communication cost.} Since each server step in \algname{} involves $K$ client trips between the server and the clients,  the total communication cost is $\order{K/\epsilon^2Q} + \order{Q\sigma^2 / K \epsilon} + \order{K Q\sigma^2 \tau^2_{\max, K} / \epsilon}$ in order to achieve $\frac{1}{T} \sum_{t=0}^{T-1} \mathbb{E}\left[\norm{\nabla f(w^t)}^2\right] \leq \epsilon$. This suggests a trade-off in the communication cost introduced by the effect of the buffer. We empirically investigate different values of $K$ in Section~\ref{sec:experiment} and observe this tradeoff in Table~\ref{tab:communication_cost}. 

\textbf{Relation between communication and local computation. } Note that in equation (\ref{eq:constantLR_bound}), increasing the number of local steps $Q$ improves the first term related to $F^*$ and the second term related to the local variance $\sigma^2_{\ell}$, but increases the third and fourth term. The first term with constant $F^*$ characterizes the distance to optimal loss. Hence, increasing local computation $Q$ reduces the loss faster, but it also leads to more drift, enlarging the effect of the local and global variance sum $\sigma^2$ and the impact of the worst-case staleness $\tau_{\max, K}$.


\textbf{Effect of staleness.} The effect of staleness between the initialization of \texttt{FedBuff-client} and the server update dissipates at the rate of $\order{1/T}$ according to the fourth term in equation $(\ref{eq:constantLR_bound})$. In addition, the maximum staleness $\tau_{\max, K}$ reduces as the buffer size $K$ grows. (see Appendix~\ref{appendix:staleness-K})



%% file: improvements.tex
\textbf{Staleness scaling.}
To control the effect of staleness $\tau_i(t)$ in client $i$'s contribution to the $t$-th server update, we down-weight stale updates using the following function: $s(\tau_i(t)) := 1 / (1 + \tau_i(t))^{0.5}$, similar to (\cite{fedasync}).

\textbf{Learning rate normalization.}
In practical FL implementations, each client is typically asked to perform a fixed number of \emph{epochs} over their local training data, rather than a fixed number $Q$ of steps, using a server-prescribed batch size $B$ which is the same for all clients. Because different clients have different amounts of data, some clients may only have a fraction of a batch. Previous work has suggested that increasing batch size and learning rate are complementary~\citep{imagenet-1hr,increase_batchsize,three-factors}. When a client performs a local update with a batch size smaller than $B$, we have it linearly scale the learning rate used for that local step; i.e., $\eta_{\lrnlocal} := \eta_{\ell} \cdot n^t_{i,q} / B$, where $n^t_{i,q} \le B$ is actual batch size used for the step. We find that this small change can improve \algname. A theoretical justification is provided in Appendix~\ref{sec:appendix_lrn}.




%% file: experiments.tex
\begin{figure*}[t]
     \centering
     \begin{minipage}[t]{0.32\linewidth}
         \centering
         \includegraphics[width=\linewidth]{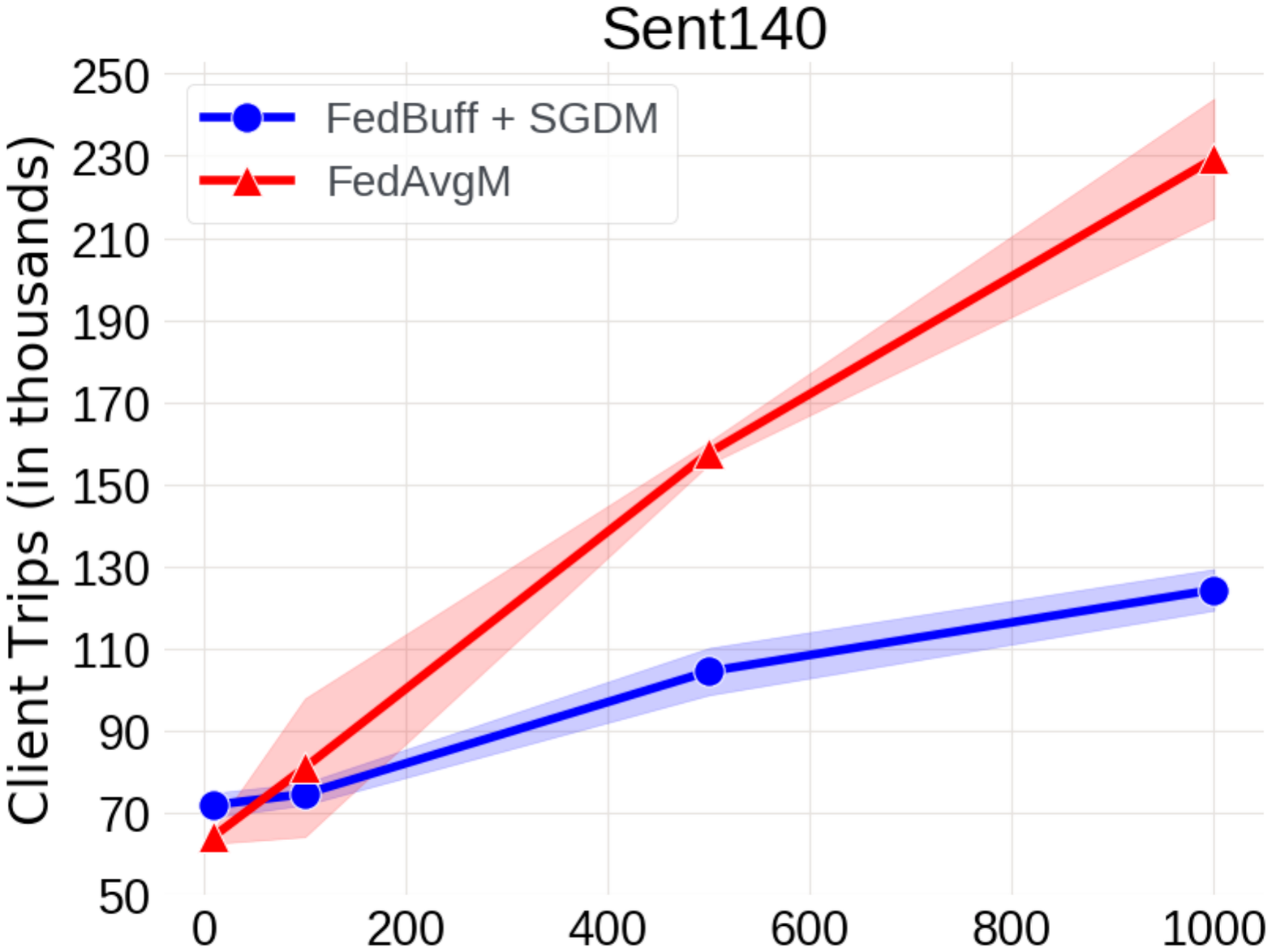}
     \end{minipage}
     \begin{minipage}[t]{0.32\linewidth}
        \centering
        \includegraphics[width=\linewidth]{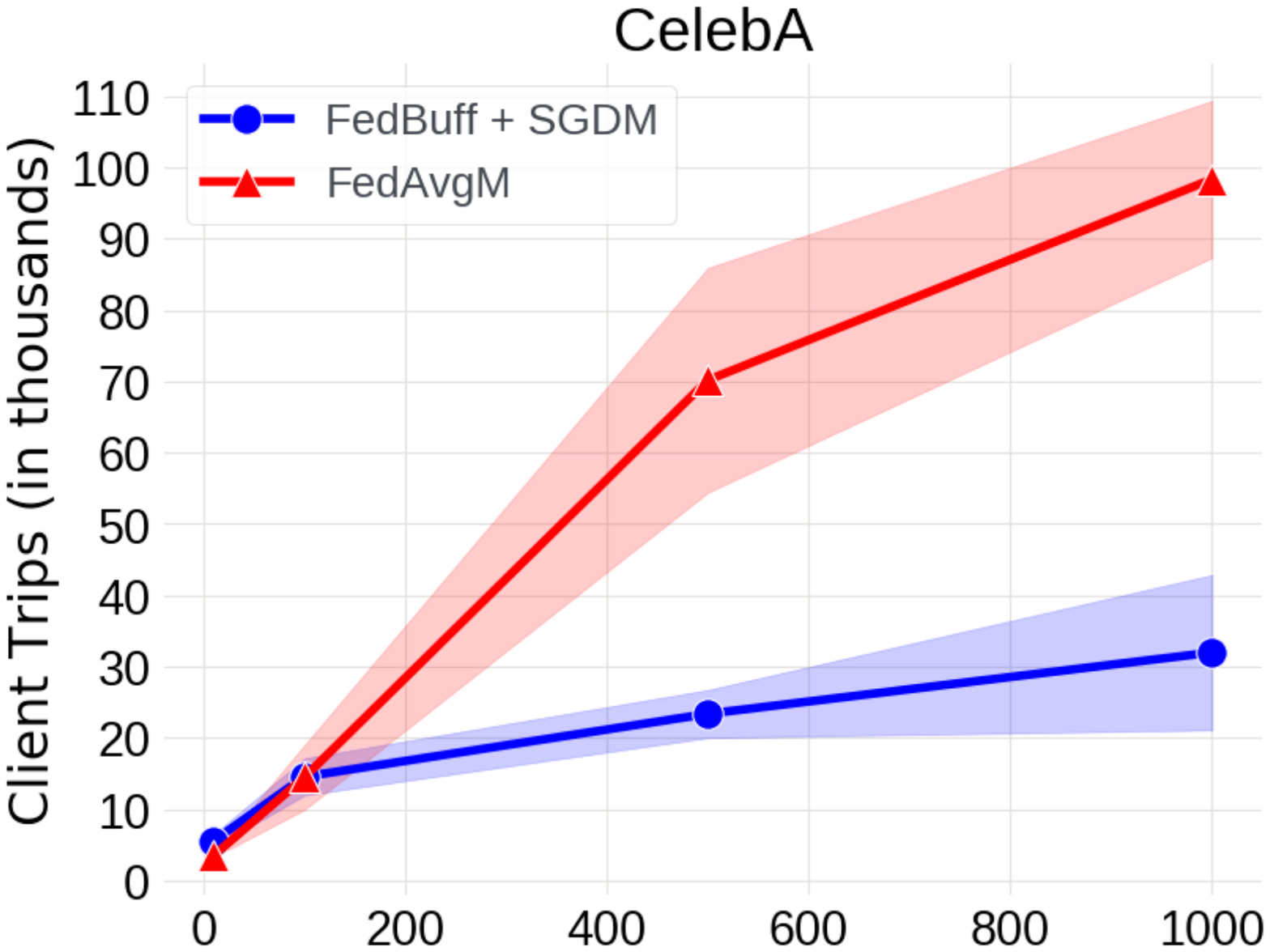}
     \end{minipage}
        \begin{minipage}[t]{0.32\linewidth}
        \centering
        \includegraphics[width=\linewidth]{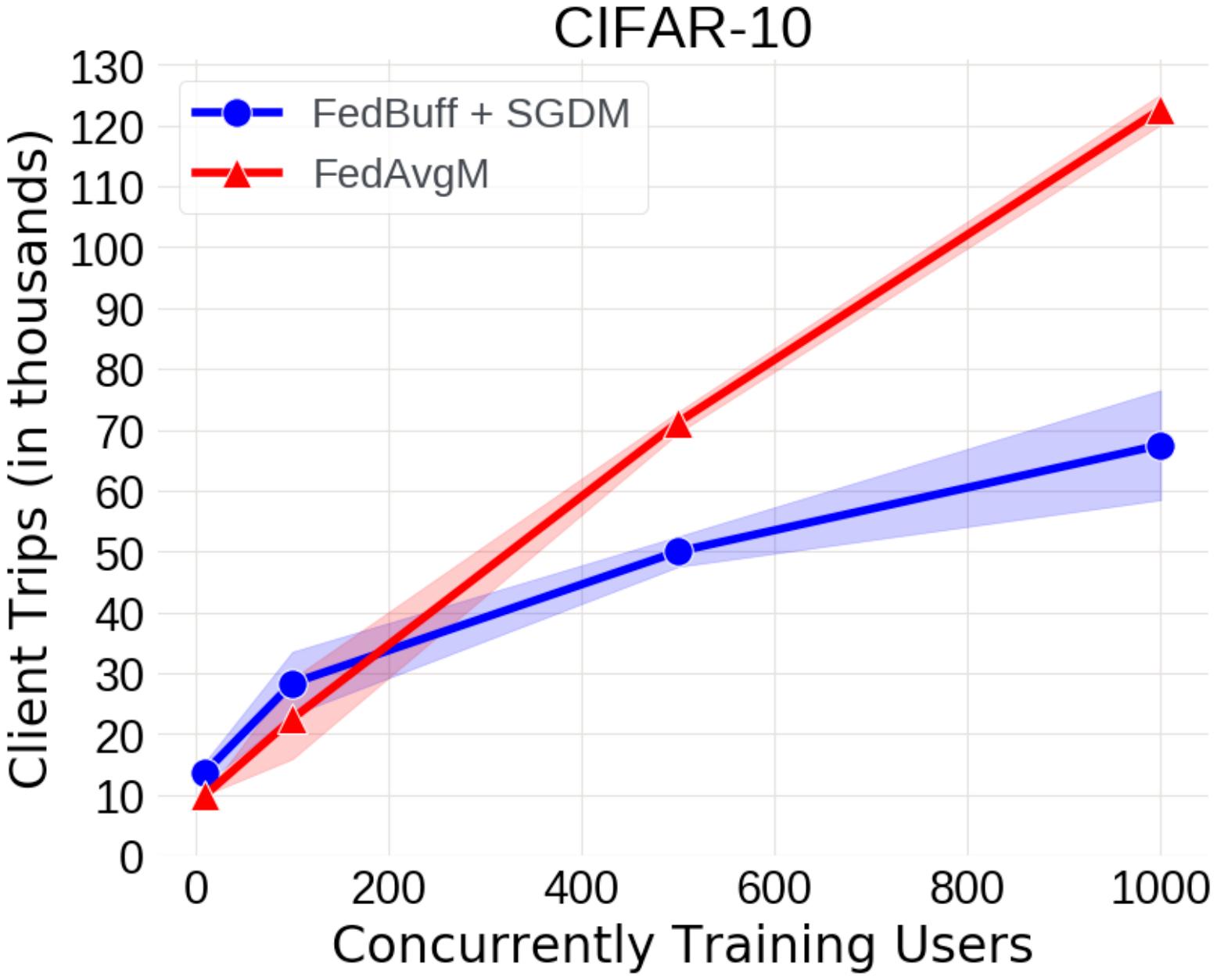}
     \end{minipage}
     \caption{Number of client trips to reach target validation accuracy for \algname + SGD with momentum at the server and FedAvgM. At low concurrency, \algname{} and FedAvgM perform similarly. However, as concurrency increases, \algname{} outperforms FedAvgM by increasingly larger amounts. In contrast to FedAvgM, \algname{}'s data-efficiency and communication-efficiency degrade less with concurrency.}
     \label{fig:large_upr}
\end{figure*}

\begin{table*}
\centering
\caption{\textbf{Average (speedup)} number of client trips (in units of 1000 updates) to reach target validation accuracy on CelebA and Sent140 (lower is better). We set concurrency = 1000 for all methods, $K=10$ for FedBuff and ran all methods for 600k client trips. ``$>600$'' indicates the target accuracy was not reached.}
\label{tab:baseline_convergence}
\begin{tabular}{lrrrrrr} 
\toprule
Dataset &  Accuracy &   FedBuff    & FedAsync       & FedAvgM    & FedAvg     & FedProx     \\ 
\midrule
CelebA  & 90\%              & 31.9  & 37.1 (1.2$\times$)    & 104 (3.3$\times$) & 231 (8.5$\times$) & 228 (8.4$\times$)  \\
Sent140 & 69\%             & 124.7 & 308.9 (2.5$\times$) & 216 (1.7$\times$) & $>600$   & $ >600$    \\
CIFAR-10 & 60\%             & 67.5 & 73.3 (1.1$\times$) & 122.7 (1.8$\times$) & 386.7 (5.7$\times$) &  292.7 (4.3$\times$)    \\
\bottomrule
\end{tabular}
\end{table*}

In this section, we compare the efficiency and scalability of \algname{} with other synchronous and asynchronous FL methods from the literature via simulation. We wish to understand how \algname{} behaves under different values of $K$, its scalability, and data efficiency.  

\textbf{Evaluation metrics.}
The standard evaluation metric for FL is the number of communication rounds to reach a target accuracy. However, asynchronous and synchronous methods do not have the same notion of rounds. For this reason, we compare different synchronous and asynchronous methods by the \emph{number of client trips} needed to reach a target accuracy. 
One client trip corresponds to one client round-trip communication. A client trip involves a client pulling the latest model from the server (download communication), performing one epoch of training on the local dataset (computation), then communicating the model update to the server (upload communication). Since the number of client trips measures both communication and computation costs, we use this as a proxy for wall-clock training time. We show wall-clock time simulation with stragglers in Appendix \ref{sec:appendix_stragglers}. 


\textbf{Datasets, models, and tasks.}
In order to provide a comparison with other work in the literature, we run experiments on three datasets: CelebA (\cite{celeba}), Sent140 (\cite{sent140}), and CIFAR-10 (\cite{cifar10}). Sent140 is a text classification dataset (binary sentiment analysis), whereas CelebA and CIFAR-10 are image classification datasets (multi-class classification).
For Sent140 and CelebA, we use the natural non-iid client partitions and models from the LEAF benchmark (\cite{leaf}). 
For Sent140, we train an LSTM classifier over 660,120 clients, where each Twitter account corresponds to one client. For CelebA, we train the same convolutional neural network classifier as LEAF over 9,343 clients, but with batch normalization layers replaced by group normalization layers (\cite{groupnorm}) as suggested in (\cite{batchnorm_with_groupnorm}). For CIFAR-10, we generate 5000 non-iid clients using a Dirichlet distribution with parameter 0.1, the same approach as in (\cite{fedavgm}).
More details about datasets, models, and tasks are provided in Appendix~\ref{sec:appendix_data_model}.

\textbf{Experimental setup.}
We implement all algorithms in PyTorch (\cite{pytorch}). We repeat each experiment with three different seeds and report the average. For asynchronous FL methods, we assume that clients arrive at a constant rate. We sample the delay distribution, the time delay between a client's download and upload operation, from a half-normal distribution. We choose this distribution because it best matches the delay distribution observed in our production FL system (See Appendix ~\ref{sec:other_training_distributions}). 
We also report results with two other delay distributions (uniform and exponential) in Appendix \ref{sec:other_training_distributions}. We find that \algname's performance improvements are consistent across different delay distributions.

\textbf{Baselines.}
We compare \algname with three SyncFL baselines, namely FedAvg (\cite{google-fedavg}), FedProx (\cite{fedprox}), FedAvgM (\cite{fedavgm}), and one AsyncFL baseline, FedAsync (\cite{fedasync}). 
For more details about the algorithms used and the experimental setup, see Appendix~\ref{sec:appendix_imple}.

\textbf{Hyperparameters.}
\label{sec:hyperparameter}
For all algorithms, we run hyperparameter sweeps to tune client and server learning rates $\eta_{\ell}$ and $\eta_g$, server momentum $\beta$, and the proximal term $\mu$ for FedProx. We set $\beta = 0$ for FedAvg. Each client update entails running one local epoch with batch size $B = 32$, rather than a fixed number of local steps.
See Appendix~\ref{sec:appendix_hyper} for additional details on hyperparameter tuning.

\textbf{Concurrency and $K$.}
In at-scale cross-device FL, only a small fraction of all clients participate in training at any point in time. As discussed earlier, concurrency --- the maximum number of clients that train in parallel --- significantly impacts the performance of FL algorithms. 
For a fair comparison between synchronous and asynchronous algorithms, we keep \concurrency{} the same across all configurations. Recall the example in Figure~\ref{fig:concurrency_sync_async} where \concurrency{}=100. For synchronous algorithms, this implies that 100 clients are training and contributing in each round. For asynchronous algorithms, this implies that 100 clients can train concurrently, and we can still vary the buffer size $K$, which will control how frequently updates occur.

\subsection{Results}
\label{subsec:results}

\textbf{Comparison of Methods. }
Table \ref{tab:baseline_convergence} shows the number of client trips needed to converge to the target accuracy on Sent140, CelebA and CIFAR-10 for each method considered. In Table \ref{tab:different_ks}, we show results with other values of $K$ and present the learning curves in Appendix \ref{sec:learning_curves}. Compared to FedBuff, the best synchronous method in the experiments (FedAvgM) requires 1.7-3.3$\times$ more updates, and FedAsync requires 1.1-2.5$\times$ more updates.




\textbf{Scalablility of \algname{}.}
Figure~\ref{fig:large_upr} shows that \algname{} scales much better to larger values of \concurrency{} than FedAvgM. \algname{} with $K=10$ scales better because it updates the server model more frequently than FedAvgM in high concurrency. When \concurrency{} is 10, both FedAvgM and \algname{} update the server model after every 10 client updates. However, when \concurrency{} is 1000, \algname{} with $K=10$ updates the server model after every 10 client updates, while FedAvgM updates the server model after 1000 client updates. One might argue that FedAvgM should run at lower concurrency, e.g. 10. However, that leads to longer wall-clock training time because less parallelism is exploited. We discuss this problem in Appendix \ref{sec:appendix_stragglers}. For synchronous FL methods, larger concurrency reduces training time but is also less efficient.
On the other hand, taking server model steps more frequently is not free; \algname{} has to deal with staleness as a consequence. Our empirical results show that the benefits from frequent updating of the server model outweigh the cost of staleness in client model updates. 



\begin{table}
\centering
\caption{\textbf{Average $\pm$ standard deviation} number of client trips to reach validation accuracy on CelebA (90\%), Sent140 (69\%) and CIFAR-10 (60\%) (lower is better, Units = 1000 updates) in \algname. We set concurrency = 1000 for all methods. }
\label{tab:different_ks}
\begin{tabular}{llrrllll} 
\toprule
Dataset & $K$ & Client trips  \\ 
\hline
\Tstrut{}
        &   1                      & 32.4 $\pm$  2.0       \\
CelebA  &   10                     & 31.9 $\pm$  8.9         \\       
        &   100                    & 71.6 $\pm$  6.8        \\ 
\hline
\Tstrut{}
        &    1                      & 190.0   $\pm$  11.9      \\
Sent140 &   10                      & 124.7   $\pm$  25.8      \\
        &   100                     & 178.2   $\pm$  13.1        \\
\bottomrule
\Tstrut{}
        &    1                      & 76.7 $\pm$ 9.6    \\
CIFAR-10 &   10                     & 67.5 $\pm$ 7.4         \\
        &    100                    & 102.5 $\pm$ 2.0          \\
\bottomrule
\end{tabular}
\label{tab:communication_cost}
\end{table}

\textbf{Choice of $K$.} Table \ref{tab:different_ks} presents the number of client trips to reach validation accuracy for different values of $K$, with fixed concurrency. We find that $K$=10 is a good setting across benchmarks. We analyze FedBuff with even larger values of $K$ in Table \ref{tab:large_k}, and show the training curves of \algname{} and other algorithms in Appendix \ref{sec:learning_curves}.

\begin{figure}
 \centering
 \begin{minipage}[t]{\linewidth}
     \centering
     \includegraphics[width=0.80\linewidth]{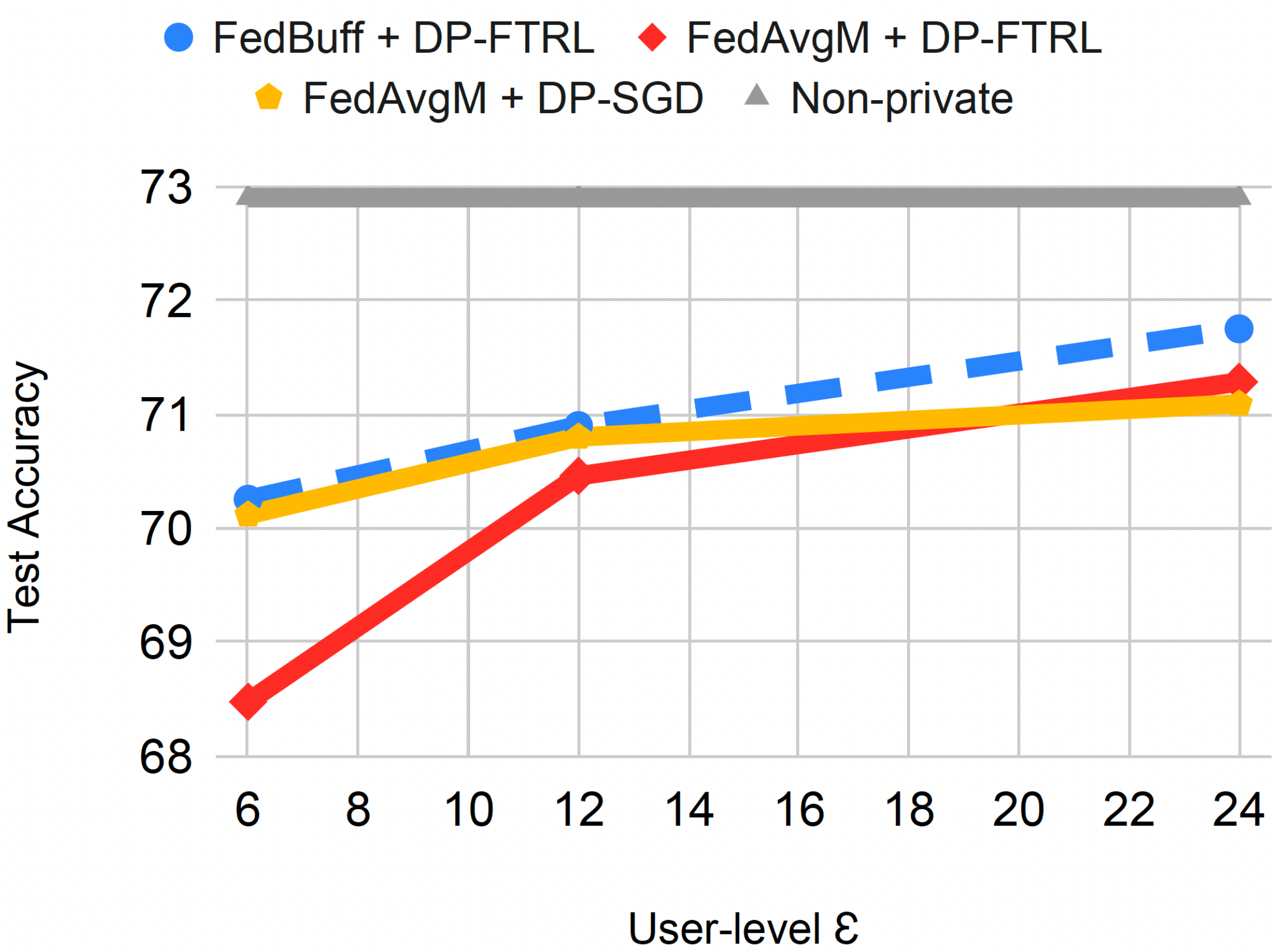}
 \end{minipage}
\caption{Accuracy on Sent140 under different levels of privacy ($\delta = 1e^{-7}$) for FedBuff with DP-FTRL versus FedAvgM with DP-FTRL and FedAvgM with amplified DP-SGD. For all methods, we use momentum at the server and fix the communication cost at 600 thousands. For \algname, we use $K = 10$. As for FedAvgM we use clients-per-round = 1000.}
\label{fig:dp_results}
\end{figure}

\textbf{\algname with Differential Privacy.} To evaluate the privacy-utility trade-off of \algname, we compare the final test accuracy of \algname with synchronous baselines after 600 thousands client trips, one pass over the dataset. Figure \ref{fig:dp_results} illustrates that \algname can outperform both FedAvgM with amplified DP-SGD and FedAvgM with DP-FTRL at high values of $\epsilon$, and be competitive for lower values of $\epsilon$. This result illustrates \algname's flexibility to be adapted for privacy. Even with DP, we find that $K=10$ is good setting. We find that a small $L$ can counteract the additional noise from taking more steps. For more details see Appendix \ref{sec:appendix_dp}. 






%% file: related.tex
In addition to the discussion in Section \ref{sec:background} on related works, we discuss the other efforts in related domains.

\textbf{Asynchronous stochastic optimization. }
Asynchronous stochastic optimization in shared-memory and distributed-memory systems has been extensively studied (\cite{bertsekasTsitsiklis, Duchi_AsyncSGD_convex, hogwild, lian2015asynchronous_nonconvex, lian2018asynchronous, Revisit_dist_SGD_google, AsyncSGDDelay, AsyncSGD_purturbed, leblond17a_ASAGA, reddi2015_ASVRG, advance_async_Mike}). Asynchronous training is resilient to stragglers in both centralized and federated settings. The idea of aggregating $K$ asynchronous updates for convex objectives has been studied in \citep{dutta2018slow}. Although \citet{dutta2020slow} provide a guarantee for non-convex objectives, their assumption on the relationship between staleness and gradient moments is difficult enforce, in contrast to the bounded staleness assumption we consider which can be easily enforced. In this work, we consider heterogeneous objectives and show that in a federated environment with a large number of clients, the source of speed-up is not only due to avoiding stragglers but also achieving better efficiency at high concurrency.
 
\textbf{Large-batch training. }
Many proposals aim to understand and characterize conditions under which linear speed-up for distributed SGD and local SGD is achievable (\cite{dont_use_minibatch, yu2019linear, is_local_sgd_better, adaptive_sync_localSGD}). It is well accepted that increasing \concurrency{} eventually saturates beyond a certain batch size in synchronous methods (\cite{gradient-diversity, scaling_nmt, imagenet-1hr, scaling_nmt, bert_in_76, imagenet_minutes, lars, measuring_data_parallelism}). However, most existing research focuses on scalability across tens of server workers, each having iid-data - very different from the FL setting.

%% file: conclusions.tex
In this paper, we propose \algname{}, an asynchronous FL training scheme with buffered aggregation. Compared to \syncfl proposals, \algname{} scales to large values of concurrency. Compared to \asyncfl proposals, \algname{} is more private as it is compatible with SecAgg and differential privacy. At high levels of $\epsilon$, we demonstrate that \algname can outperform major \syncfl proposals.
We analyze the convergence behavior of \algname{} in the non-convex setting. Empirical evaluation shows that \algname{} is up to 3.3$\times$ more efficient than FedAvgM, and up to 2.5$\times$ more efficient than FedAsync. As for future work, we are aware that our analyses is on standard SGD. We leave extending the analysis to include momentum or adaptive learning rates as future work.

%% file: appendix.tex
\section{Relationship Between Maximum Staleness and $K$}
\label{appendix:staleness-K}

Recall Assumption~\ref{assumption:staleness1}, that the staleness when executing \algname with $K=1$ is always bounded as $\tau_i(t) \le \tau_{\max, 1}$. In this section we will show that this implies the staleness bound $\tau_i(t) \le \lceil \tau_{\max, K} \rceil \le \lceil \tau_{\max, 1}/K \rceil$ when running \algname with $K > 1$.

Consider an execution of \algname. Let $r_i$ denote the time when the $i$'th client update is received by the server, and let $s_i < r_i$ denote the time when the client downloaded the serve model before performing local steps that resulted in the model update received at $r_i$.

When $K=1$, the staleness $\tau_i^{(1)}$ of the $i$'th update corresponds to the number of updates that occurs between when the client downloaded the model and when it completed local training and uploaded the model to the server,
\[
\tau_i^{(K=1)} = |\{j \colon s_i < r_j < r_i \}|.
\]
If Assumption~\ref{assumption:staleness1} holds, then $\max_{i} \tau_i^{(1)} \le \tau_{\max,1}.$

When $K > 1$, the server waits to aggregate $K$ client updates before stepping the global model. Thus, if $\tau_i^{(K=1)}$ client updates are received between the times $s_i$ and $r_i$, then at most $\tau_i^{(1)} / K$ server updates occur during this time. Hence $\tau_i^{(K)} \le \lceil \tau_i^{(1)} / K \rceil$, and therefore
\[
\tau_{\max, K} = \max_i \tau_i^{(K)} \le \max_i \lceil \tau_i^{(1)} / K \rceil \le \lceil \tau_{\max, 1} / K \rceil.
\]

Note that the times $s_i$ and $r_i$ only depend on the number of clients training concurrently, and the distribution of client execution times (the time it takes a client to complete one round of local updates; i.e., the distribution of $r_i - s_i$). These times are not impacted by the choice of $K$; rather $K$ only affects how frequently the server performs an update. Thus, the arguments above hold regardless of the distribution of client execution times, and only depend on  Assumption~\ref{assumption:staleness1}.

We also remark that the same relationship holds for the average delay; i.e., increasing $K$ reduces average delay. In particular, let
\newcommand{\taubar}{\overline{\tau}}
\[
\taubar_1 = \lim_{N \rightarrow \infty} \frac{1}{N} \sum_{i=1}^N \tau_i^{(1)},
\]
and suppose the limit exists. Clearly, if the limit exists and Assumption~\ref{assumption:staleness1} holds, then $\taubar_1 \le \tau_{\max, 1}$. Furthermore, then
\begin{align*}
\taubar_K &= \lim_{N \rightarrow \infty} \frac{1}{N} \sum_{i=1}^N \tau_i^{(K)} \\
&\le \lim_{N \rightarrow \infty} \frac{1}{N} \sum_{i=1}^N \tau_i^{(1)} / K \\
&= \taubar_1 / K.
\end{align*}

\section{Experiment Details}
\label{sec:appendix_exp}

\subsection{Datasets and Models}
\label{sec:appendix_data_model}



\textbf{Sent140.} We train a sentiment classifier on tweets from the Sent140 dataset~\citep{leaf, sent140} with a two-layer LSTM binary classifier. The dataset has 660,120 clients where each client is a Twitter account. The LSTM binary classifier contains 100 hidden units with a top 10,000 pretrained word embedding from 300D GloVe~\citep{glove}. The model has a max sequence length of 25 characters. The model first embeds each of the characters into a 300-dimensional space by looking up GloVe, passes through 2 LSTM layers and a 128 hidden unit linear layer to output labels 0 or 1. We set the dropout rate to 0.1. We split the data into 80\% training set, 10\% validation set, and 10\% test set using script provided by \cite{leaf}. Due to memory constraint, we use 15\% of the entire dataset using the script provided by \cite{leaf}, with split seed = 1549775860.

\textbf{CelebA.} We study an image classification problem on the CelebA dataset~\citep{celeba, leaf} using a four layer CNN binary classifier with dropout rate of 0.1, stride of 1, and padding of 2. As it is standard with image datasets, we preprocess train, validation, and test images; we resize and center crop each image to $32 \times 32$ pixels, then normalize by $0.5$ mean and $0.5$ standard deviation. The dataset has 9,343 clients where each client is a unique celebrity. 

\textbf{CIFAR-10.} We evaluate a multi-class image classification problem on CIFAR-10 \citep{cifar10} using a four layer CNN binary classifier with dropout rate of 0.1, stride of 1, and padding of 2. We normalize the images by the dataset mean and standard deviation. Following \citet{fedavgm}, we partition the dataset into 5,000 clients using a Dirichlet distribution with parameter 0.1 and split seed = 0.

\subsection{Implementation Details}
\label{sec:appendix_imple}
We implemented all algorithms in Pytorch \citep{pytorch} and evaluated them on a cluster of machines, each with eight NVidia V100 GPUs. Independently, we built a simulator to simulate large-scale federated learning environments. The simulator can realistically simulate clients, server, communication channels between clients and server, model aggregation schemes, and local training of clients. We intend to open-source the simulator, making it available for the research community.

For our experiments, we assume clients arrive to the FL system at a constant rate. To simulate device heterogeneity, we sample each client training duration from a half-normal, uniform, or exponential distribution. Moreover, our implementation has two other important distinctions. First, each client does one epoch of training over its local data; this distinction stems from two observations in our production stack: that our FL production stack has plenty of users to train on, and that we train small capacity models in FL (e.g., less than 10 million parameters) because of bandwidth and client compute. Second, we use the weighted sum of the client updates instead of the weighted average. This is because each client update has different levels of staleness; taking the average cannot capture the true contribution for each client. 

\subsection{Hyperparameters}
\label{sec:appendix_hyper}
For all experiments, we tune hyperparameters using Bayesian optimization \citep{bayesian-opt}. For optimizer on clients, we use minibatch SGD for all tasks. We select the best hyperparameters based on the number of rounds to reach target validation accuracy for each dataset. 

\begin{table}
\centering
\caption{The best performing hyperparameters for \cref{fig:dp_results}}
\label{tab:dp_hp}
\begin{tabular}{llllll} 
\toprule
 & \algname + DP-FTRL & SyncFL + DP-SGD & SyncFL + DP-FTRL\\
\hline
\Tstrut{}
               &  $\eta_{\ell}=1.0$ &  $\eta_{\ell}=1.0\cdot10^{-1}$ &  $\eta_{\ell}=1.0\cdot10^{-3}$  \\     
$\epsilon = 6$ & $\eta_g=4.3$ & $\eta_g=5.9\cdot10^{2}$ & $\eta_g=2.6\cdot10^{4}$ \\ 
               & $\beta=9.9\cdot10^{-1}$ & $\beta=3.0\cdot10^{-1}$ & $\beta=1.0\cdot10^{-1}$ \\
               & $L=1.2\cdot10^{-4}$ & $L=1.1\cdot10^{-2}$& $L=2.7\cdot10^{-4}$ \\
\hline
\Tstrut{}
               &  $\eta_{\ell}=1.0\cdot10^{-2}$ &  $\eta_{\ell}=1.0$ &  
               $\eta_{\ell}=1.0$  \\     
$\epsilon = 12$ & $\eta_g=5.4\cdot10^{1}$ & 
                $\eta_g=1.0\cdot10^{4}$ & $\eta_g=1.0\cdot10^{2}$ \\ 
                & $\beta=0$ & 
                $\beta=5.0\cdot10^{-1}$ & $\beta=9.0\cdot10^{-1}$ \\
                & $L=1.1\cdot10^{-3}$ & 
                $L=7.6\cdot10^{-3}$  & 
                $L=2.7\cdot10^{-1}$ \\
\hline
\Tstrut{}
               &  $\eta_{\ell}=1.0\cdot10^{-1}$ &  $\eta_{\ell}=1.0\cdot10^{-1}$ &  $\eta_{\ell}=1.0$  \\     
$\epsilon = 24 $ & $\eta_g=8.7\cdot10^{2}$ & 
                $\eta_g=5.1\cdot10^{2}$ & $\eta_g=8.0\cdot10^{3}$ \\ 
               & $\beta=3.0\cdot10^{-1}$ & $\beta=3.0\cdot10^{-1}$ & $\beta=5.0\cdot10^{-1}$ \\
               & $L=1.0\cdot10^{-4}$ & 
               $L=1.4\cdot10^{-2}$ & 
               $L=1.0\cdot10^{-4}$ \\
\bottomrule
\end{tabular}
\end{table}

\subsubsection{Hyperparameter Ranges}
Below, we show the range for the client learning rate ($\eta_{\ell}$), server learning rate ($\eta_g$), server momentum ($\beta$), proximal term ($\mu$) sweep ranges. 
\begin{align*}
\beta &\in \{0, 0.1, 0.2, 0.3, 0.4, 0.5, 0.6, 0.7, 0.8, 0.9, 0.99\} \\
\eta_{\ell} &\in [1\cdot10^{-8}, 10000] \\
\eta_g &\in [1\cdot10^{-8}, 10000] \\
\mu &\in \{0.001, 0.01, 0.1, 1\} \\
\end{align*}

\subsubsection{Best Performing Hyperparameters}
Table \ref{tab:best_hyper} illustrates the best value for client and server learning rates ($\eta_{\ell}$, $\eta_g$), server momentum ($\beta$), and proximal term ($\mu$) for tasks in Table \ref{tab:baseline_convergence}. For experiments in Table \ref{tab:large_k}, we set staleness exponent $\alpha = 10$. We set $\alpha = 0.5$ for all other experiments. 

\begin{table}
\centering
\caption{The best performing hyperparameters for Table \ref{tab:baseline_convergence}}
\label{tab:best_hyper}
\begin{tabular}{llllll} 
\toprule
 & \algname{} & FedAsync &  FedAvgM & FedAvg & FedProx  \\
\hline
\Tstrut{}
  & $\eta_{\ell}=4.7\cdot10^{-6}$ & $\eta_{\ell}=5.7$ &  $\eta_{\ell}=1.1\cdot10^{-1}$ & $\eta_{\ell}=1.0\cdot10^2$  & $\eta_{\ell}=4.9\cdot10^{-4}$ \\     
CelebA  & $\eta_g=1.0\cdot10^{3}$  & $\eta_g=2.8\cdot10^{-3}$ & $\eta_g=2.4\cdot10^{-1}$ & $\eta_g=1.6\cdot10^{-3}$ & $\eta_g=1.0\cdot10^2$   \\ 
  & $\beta=3.0\cdot10^{-1}$ &  &  $\beta=8.3\cdot10^{-1}$ &  & $\mu=1.0\cdot10^{-2}$   \\ 
\hline
\Tstrut{}
  & $\eta_{\ell}=1.3\cdot10^1$ & $\eta_{\ell}=1.7\cdot10^1$ &  $\eta_{\ell}=1.5$ & $\eta_{\ell}=2.6\cdot10^{-3}$  & $\eta_{\ell}=2.0\cdot10^{-3}$ \\     
Sent140  & $\eta_g=4.9\cdot10^{-2}$  & $\eta_g=1.5\cdot10^{-2}$ & $\eta_g=3.4\cdot10^{-1}$ & $\eta_g=1.0\cdot10^3$ & $\eta_g=1.0^3$   \\ 
  & $\beta=5.0\cdot10^{-1}$ &  &  $\beta=9.0\cdot10^{-1}$ &  & $\mu=1.0\cdot10^{-3}$   \\ 
\hline
\Tstrut{}

  & $\eta_{\ell}=1.95\cdot10^{-4}$  & $\eta_{\ell}=1.0\cdot10^{2}$  & $\eta_{\ell}=1.0\cdot10^1$  & $\eta_{\ell}=1.0\cdot10^1$  & $\eta_{\ell}=1.0\cdot10^1$  \\     
CIFAR-10  & $\eta_g=4.09\cdot10^1$ & $\eta_g=6.4\cdot10^{-5}$ & $\eta_g=1.02\cdot10^{-3}$ & $\eta_g=1.02\cdot10^{-3}$ & $\eta_g=1.02\cdot10^{-3}$ \\ 
& $\beta=0$ & & $\beta=9.0\cdot10^{-1}$  &  & $\mu=1.0\cdot10^{-3}$ \\
\bottomrule
\end{tabular}
\end{table}

\section{Additional Experiments}
\label{sec:appendix_additional}
\input{extra_experiments}

\subsection{Learning Rate Normalization (\lrn{})}
\label{sec:appendix_lrn}
\subsubsection{Theoretical Justification}

Recall that \lrn{} described in Section \ref{sec:improvement} aims to address the situation where a client performing local updates may need to perform an update using a batch size $b$ smaller than the server-prescribed batch size $B$. This may occur when processing a batch at the end of one epoch, including the first batch if the client has fewer than $B$ samples in total. Since this only pertains to the local updates performed at clients, let us simply write such an update as
\begin{equation} \label{eq:sgd_varying_batch}
y_q = y_{q-1} - \eta_q g_q^{(b_q)},
\end{equation}
without referring to any specific client index $i$ or global iteration index $t$. Here $g_q^{(b_q)}$ denotes a stochastic gradient of $F$ (the client's local objective) evaluated at $y_q$ using batch size $b_q$.

Assume that $F$ is $L$-smooth, i.e.,
\[
\norm{\nabla F(y) - \nabla F(y')} \le L \norm{y - y'}.
\]
Also assume that the stochastic gradients are unbiased and have variance satisfying a weak growth condition. Specifically, assume that with batch size $b_q=1$,
\begin{align*}
\mathbb{E}[g_q^{(1)} | y_q] &= \nabla F(y_q), \\
\mathbb{E}[ \norm{g_q^{(1)} - \nabla F(y_q)}^2] &\le \sigma_\ell^2 + M \norm{\nabla F(y_q)}^2.
\end{align*}
Note that in the proof of Theorem~\ref{thm:main_ergodic_rate}, we make the stronger assumption of bounded variance, corresponding to $M=0$.

Furthermore, suppose that a mini-batch stochastic gradient $g_q^{(n)}$ with batch size $b_q > 1$ is obtained by averaging the gradients evaluated at $b_q$ independent and identically distributed samples. Thus,
\begin{align*}
\mathbb{E}[g_q^{(b_q)} | y_q] &= \nabla F(y_q), \\
\mathbb{E}[ \norm{g_q^{(b_q)} - \nabla F(y_q)}^2] &\le \frac{\sigma_\ell^2}{b_q} + \frac{M}{b_q} \norm{\nabla F(y_q)}^2.
\end{align*}

\textbf{Uniform batch sizes.} If all steps use the same batch size $b_q=B$ with constant step-size $\eta_q = \eta_\ell$ satisfying
\[
0 < \eta_\ell \le \frac{1}{L(M/B + 1)},
\]
then it is well-known that the SGD iterates satisfy
\[
\mathbb{E}\left[\frac{1}{Q} \sum_{q=1}^Q \norm{\nabla F(y_q)}^2\right] \le \frac{2(F(y_1) - F^*)}{\eta_\ell Q} + \frac{\eta_\ell L \sigma_\ell^2}{B};
\]
see, for example, Theorem~4.8 in L.~Bottou, F.~Curtis, and J.~Nocedal, ``Optimization methods for large-scale machine learning,'' \textit{SIAM Review}, 2019.

\textbf{Non-uniform batch sizes.} Now suppose that some steps will use batch size $1 < b_q \le B$. In this case one can show the following result.

\begin{theorem*}
Consider updates as in \eqref{eq:sgd_varying_batch} with per-iteration batch size
\[
\eta_q = \eta_\ell \frac{b_q}{B},
\]
and let $A_Q = \sum_{q=1}^Q \eta_q = \frac{\eta_\ell}{B} \sum_{q=1}^Q b_q$.
Suppose that $\eta_\ell$ satisfies
\[
0 < \eta_\ell \le \frac{1}{L (M/B + 1)}.
\]
Then
\[
\mathbb{E}\left[ \frac{1}{A_Q} \sum_{q=1}^Q \norm{\nabla F(y_q)}^2\right] \le \frac{2(F(y_1) - F^*)}{A_Q} + \frac{\eta_\ell L \sigma_\ell^2}{B}.
\]
\end{theorem*}

First, note that $A_Q$ is strictly increasing in $Q$, since $1 \le b_q \le B$. In the special case where $b_q = B$ for all $q$ we exactly recover the result above for uniform batch sizes. More generally, when $b_q < B$ for some steps, the asymptotic residual is identical to the case with uniform-batch size. This justifies using the \lrn{} step-size rule $\eta_q = \eta_\ell b_q / B$ when encountering batches of size $b_q < B$. The proof follows from similar arguments to those of Theorem~4.8 in L.~Bottou, F.~Curtis, and J.~Nocedal, ``Optimization methods for large-scale machine learning,'' \textit{SIAM Review}, 2019.

\newcommand{\ip}[2]{\left\langle #1, #2 \right\rangle}
\begin{proof}
Let $\mathbb{E}_q$ denote expectation with respect to all randomness up to step $y_q$. Because $F$ is $L$-smooth,
\[
\mathbb{E}_q[F(y_{q+1})] - F(y_q) \le - \eta_q \ip{ \nabla F(y_q) }{ \mathbb{E}_q[g_q^{(b_q)}] } + \frac{\eta_q^2 L}{2} \mathbb{E}_k[ \norm{g_q^{(b_q)}}^2].
\]
From the weak growth assumption, it follows that
\[
\mathbb{E}_k[\norm{g_q^{(b_q)}}^2] \le \frac{\sigma_\ell^2}{b_q} + \left(\frac{M}{b_q} + 1\right) \norm{\nabla F(y_q)}^2,
\]
and thus
\begin{align*}
\mathbb{E}_q[F(y_{q+1})] - F(y_q) &\le - \eta_q \norm{ \nabla F(y_q) }^2 + \frac{\eta_q^2 L}{2} \left( \frac{\sigma_\ell^2}{b_q} + \left(\frac{M}{b_q} + 1\right) \norm{\nabla F(y_q)}^2 \right) \\
&= - \eta_q \left(1 - \frac{\eta_q L}{2} \left(\frac{M}{b_q} + 1\right)\right) \norm{\nabla F(y_q)}^2 + \frac{\eta_q^2 L \sigma_\ell^2}{2 b_q}.
\end{align*}
Based on the relationship $\eta_q = \eta_\ell b_q / B$ and the upper-bound assumed on $\eta_\ell$, we have
\begin{align*}
\frac{\eta_q L }{2} \left( \frac{M}{b_q} + 1\right) &\le \frac{1}{2}.
\end{align*}
Consequently,
\begin{align*}
\mathbb{E}_q[F(y_{q+1})] - F(y_q) &\le - \frac{\eta_q}{2} \norm{\nabla F(y_q)}^2 + \frac{\eta_q^2 L \sigma_\ell^2}{2 b_q}.
\end{align*}
Rearranging, we get
\begin{align*}
\frac{\eta_q}{2} \norm{\nabla F(y_q)}^2 &\le F(y_q) - \mathbb{E}_q[F(y_{q+1})] + \frac{\eta_q^2 L \sigma_\ell^2}{2 b_q}.
\end{align*}
Summing both sides over $q=1,\dots,Q$ and taking the total expectation yields
\begin{align*}
\sum_{q=1}^Q \frac{\eta_q}{2} \mathbb{E}[\norm{\nabla F(y_q)}^2] &\le F(y_1) - \mathbb{E}[F(y_Q)] + \sum_{q=1}^Q \frac{\eta_q^2 L \sigma_\ell^2}{2 n_q} \\
&\le F(y_1) - F^* + \sum_{q=1}^Q \frac{\eta_q^2 L \sigma_\ell^2}{2 n_q}.
\end{align*}
Now, multiplying both sides by $2/A_Q$, we obtain
\begin{align*}
\frac{1}{A_Q} \sum_{q=1}^Q \eta_q \mathbb{E}[\norm{\nabla F(y_q)}^2] &\le \frac{2(F(y_1) - F^*)}{A_Q} + \frac{1}{A_Q} \sum_{q=1}^Q \frac{\eta_q^2 L \sigma_\ell^2}{n_q} \\
&= \frac{2(F(y_1) - F^*)}{A_Q} + \frac{\eta_\ell L \sigma_\ell^2}{B}.
\end{align*}
\end{proof}

\subsubsection{Empirical Evaluation}

\begin{table}
\centering
\caption{Number of client updates (lower is better) to reach validation accuracy on CelebA (90\%) and Sent140 (69\%). We set $M = 1000$ for all methods. We compare \lrn{} against two other popular weighting schemes. {\em Example weight} is when the weight is the number of training examples for each client. {\em Uniform weight} is where all clients have weights of 1. (Units = 1000 updates.)}
\label{tab:lrn}
\begin{tabular}{llrrrlll} 
\toprule
Dataset & K & \lrn{} & Example Weight & Uniform Weight  \\ 
\hline
\Tstrut{}
        &   1                      & 20.8                       & 23.9                             & 20.7 \\
CelebA  &   10                     & 27.1                       & 25.5                             & 28.7 \\       
        &   100                    & 57.6                       & 57.6                             & 54.4  \\ 
\hline
\Tstrut{}
        &    1                      & 190.0                      & 201.9                            & 201.9  \\
Sent140 &   10                      & 124.7                      & 207.9                            & 136.6  \\
        &   100                     & 178.2                      & 570.3                            & 231.7 \\
\bottomrule
\end{tabular}
\end{table}

In Table \ref{tab:lrn}, we compare \lrn{} against two other weighting schemes: {\em Example Weight } where the weight is the number of training examples for each client, and {\em Uniform Weight} where all clients have weight of 1. We see that \lrn{} performs competitively on CelebA. For CelebA, all weighting schemes,  Uniform, Example, and \lrn{} perform similarly. This is because all clients in CelebA have one batch of data and number of examples per client is fairly centered around the mean. On the other hand, \lrn{} significantly outperforms Example Weight and Uniform Weight on Sent140. \lrn{} is beneficial when there is a high degree of data imbalance across clients, as in Sent140. Sent140 is more representative of real world FL applications where there is a long tail in the number of examples and number of batches per client. 

\subsection{Learning Curves}
\label{sec:learning_curves}
In this section, we show the learning curves for each algorithm in Figures \ref{fig:sent140_accuracy}, \ref{fig:celeba_accuracy}, and \ref{fig:cifar_accuracy}. These figures demonstrate \algname{}'s robustness to different staleness distributions. Synchronous FL algorithms, FedAvgM, FedAvg and FedProx, are unaffected by the change in staleness distribution because they simply wait for all clients in the round. 

For both CelebA and Sent140, \algname{} with $K=10$ can reach the target validation accuracy quicker than other values of $K$. At $K=10$, \algname{} appears to have the optimal balance between speed and variance reduction. 

\clearpage
\begin{figure*}[t]
     \centering
     \begin{minipage}[t]{0.32\linewidth}
         \centering
         \includegraphics[width=\linewidth]{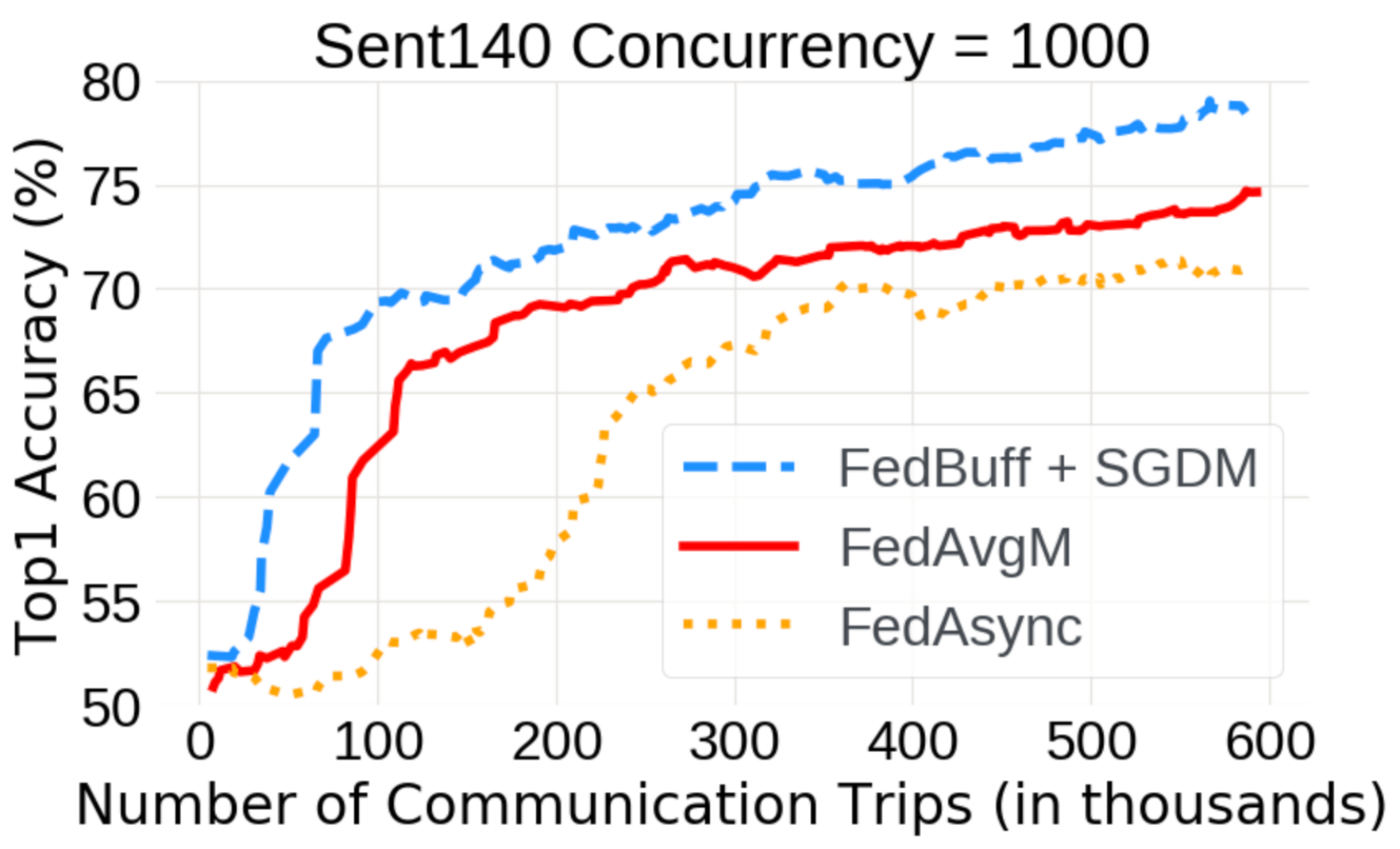}
     \end{minipage}
     \begin{minipage}[t]{0.32\linewidth}
        \centering
        \includegraphics[width=\linewidth]{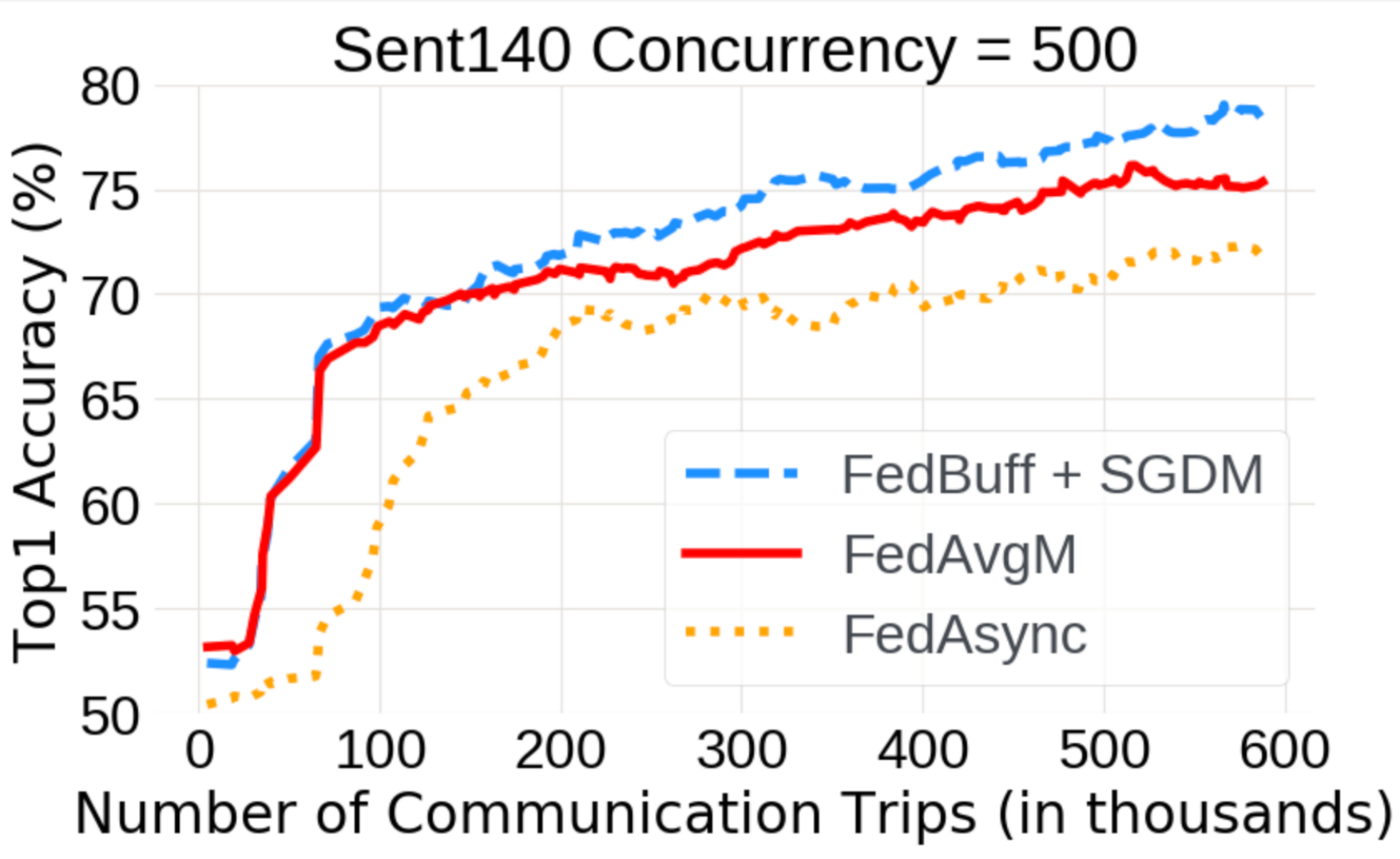}
     \end{minipage}
        \begin{minipage}[t]{0.32\linewidth}
        \centering
        \includegraphics[width=\linewidth]{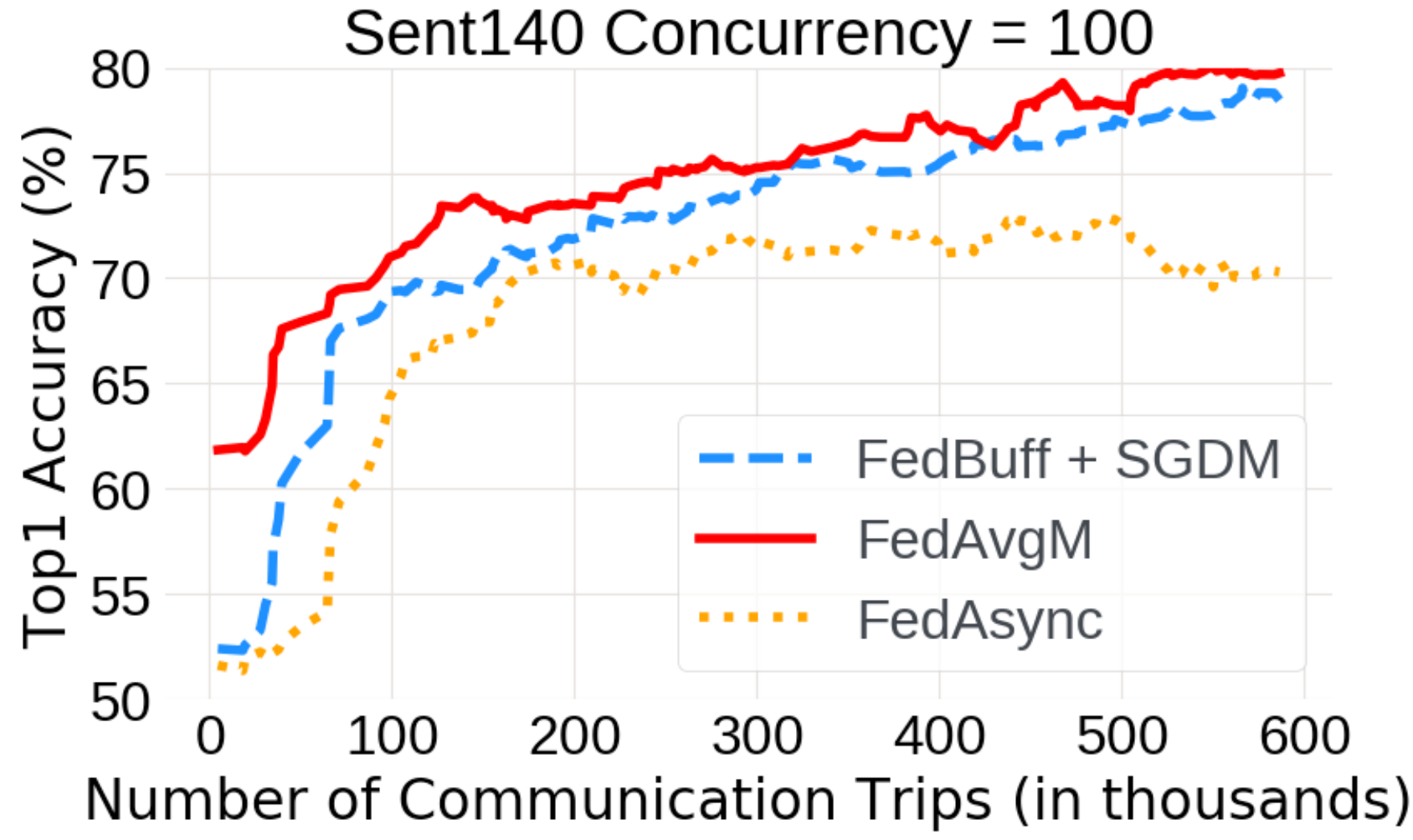}
     \end{minipage}
     \caption{Training accuracy for FedAsync, FedAvgM and \algname on Sent140.}
     \label{fig:sent140_accuracy}
\end{figure*}

\begin{figure*}[t]
     \centering
     \begin{minipage}[t]{0.32\linewidth}
         \centering
         \includegraphics[width=\linewidth]{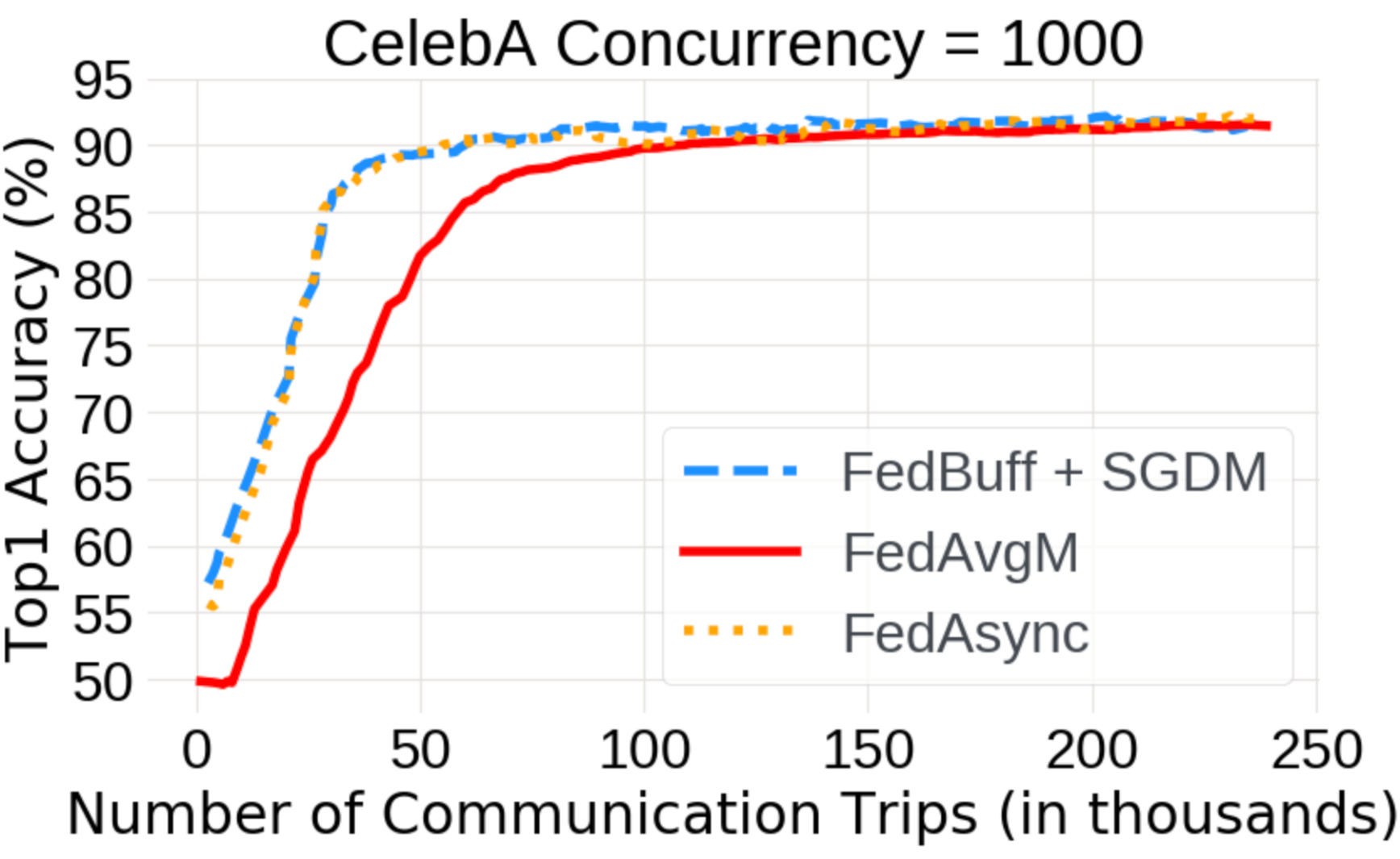}
     \end{minipage}
     \begin{minipage}[t]{0.32\linewidth}
        \centering
        \includegraphics[width=\linewidth]{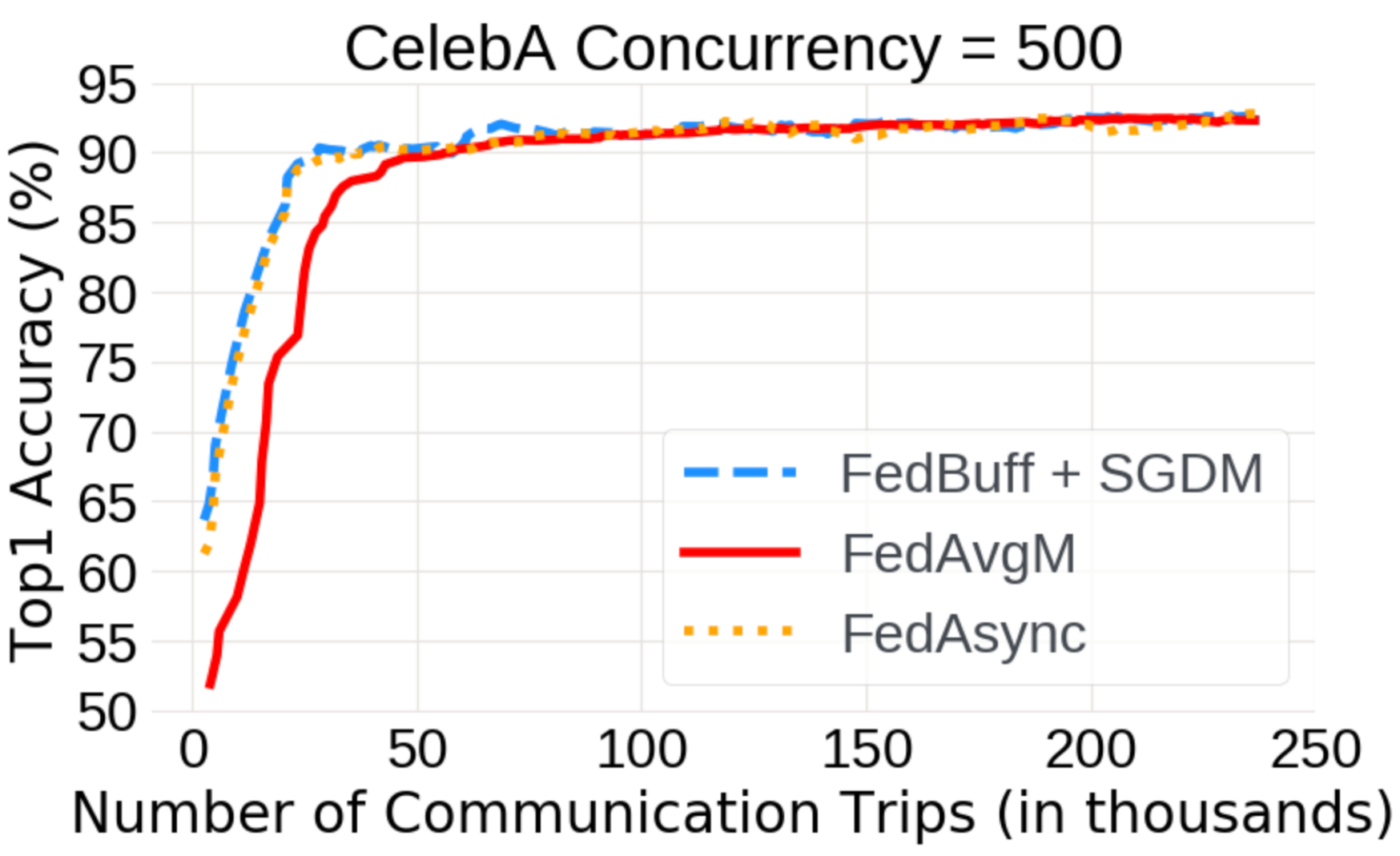}
     \end{minipage}
        \begin{minipage}[t]{0.32\linewidth}
        \centering
        \includegraphics[width=\linewidth]{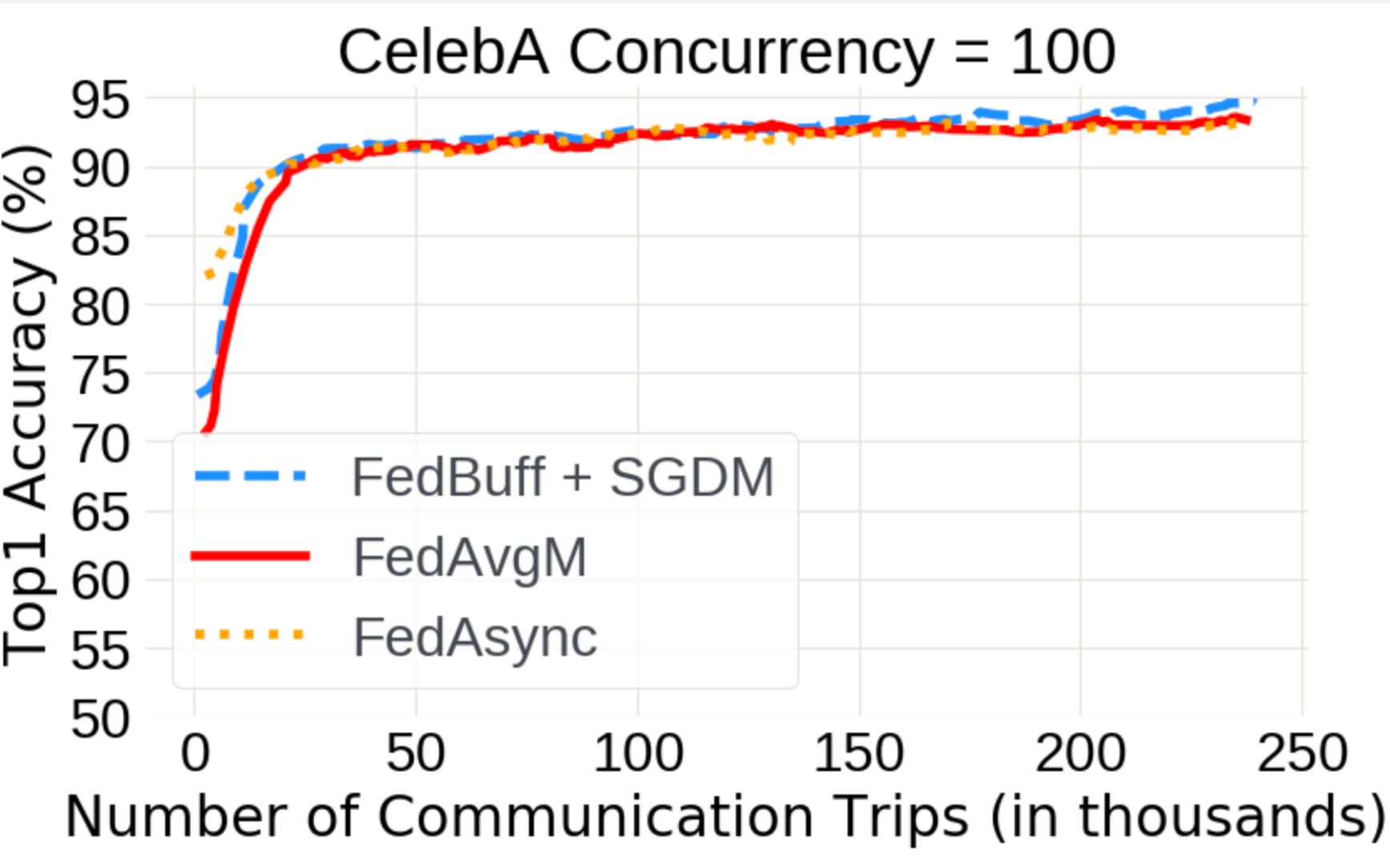}
     \end{minipage}
     \caption{Training accuracy for FedAsync, FedAvgM and \algname on CelebA.}
     \label{fig:celeba_accuracy}
\end{figure*}

\begin{figure*}[t]
     \centering
     \begin{minipage}[t]{0.32\linewidth}
         \centering
         \includegraphics[width=\linewidth]{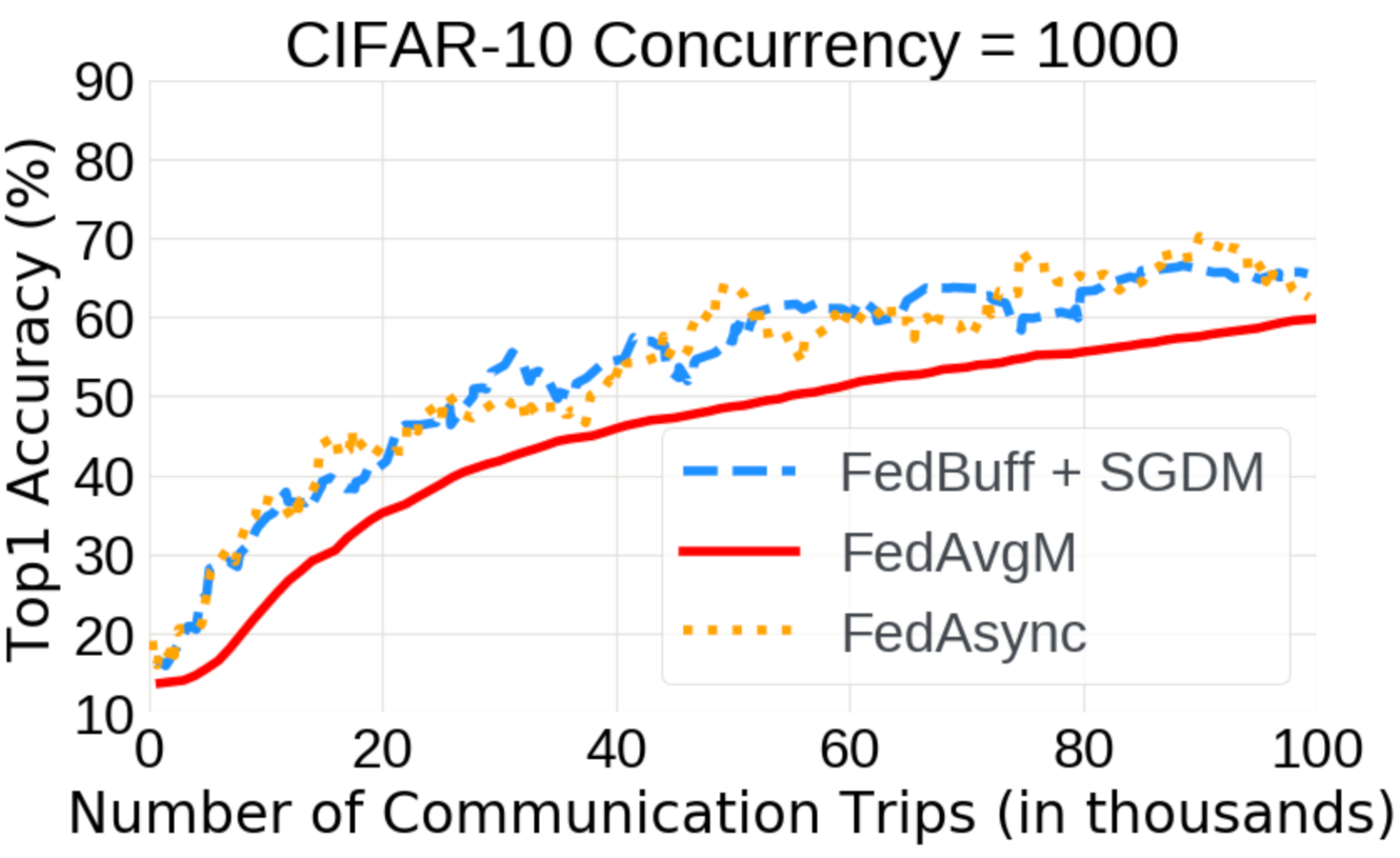}
     \end{minipage}
     \begin{minipage}[t]{0.32\linewidth}
        \centering
        \includegraphics[width=\linewidth]{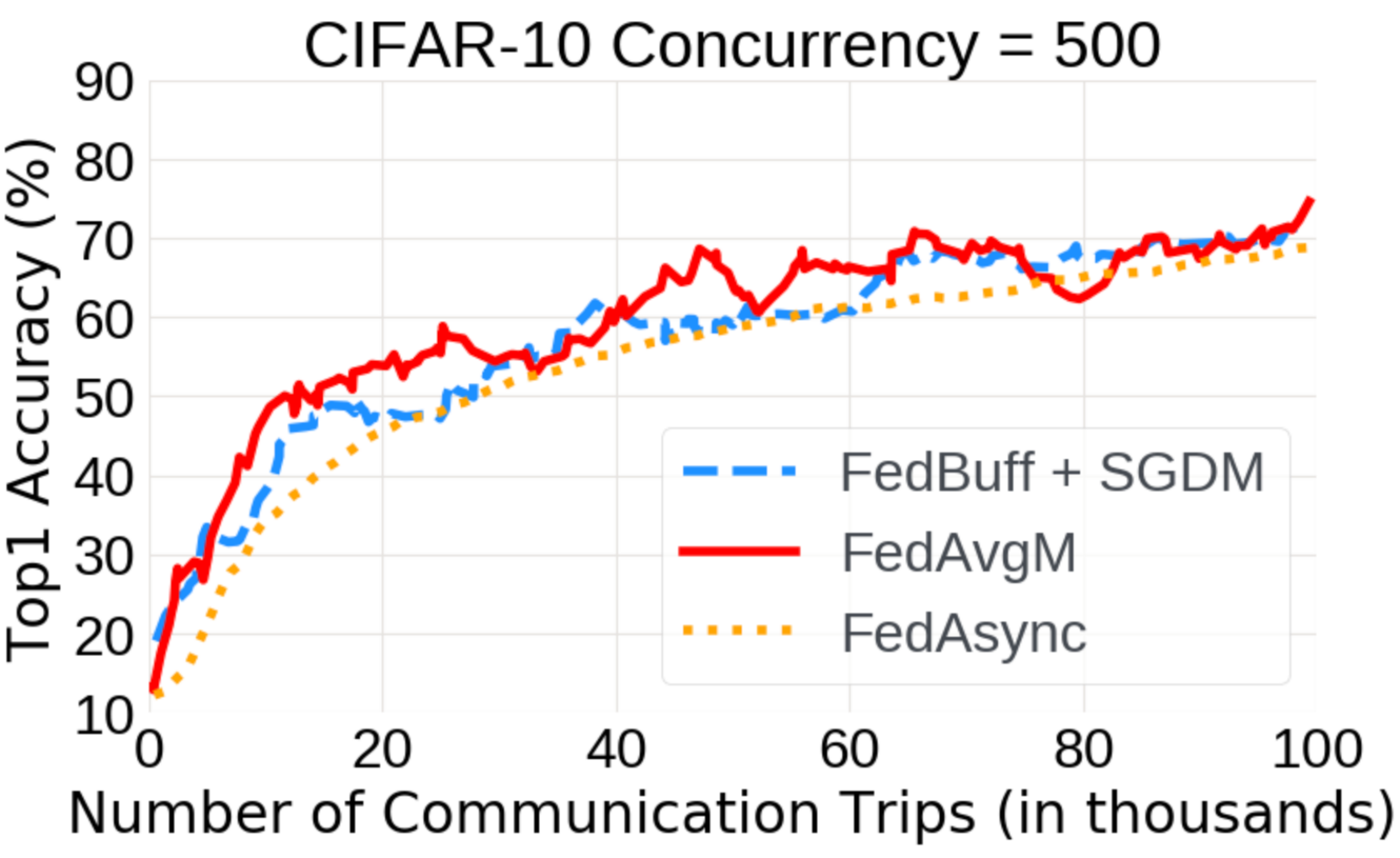}
     \end{minipage}
        \begin{minipage}[t]{0.32\linewidth}
        \centering
        \includegraphics[width=\linewidth]{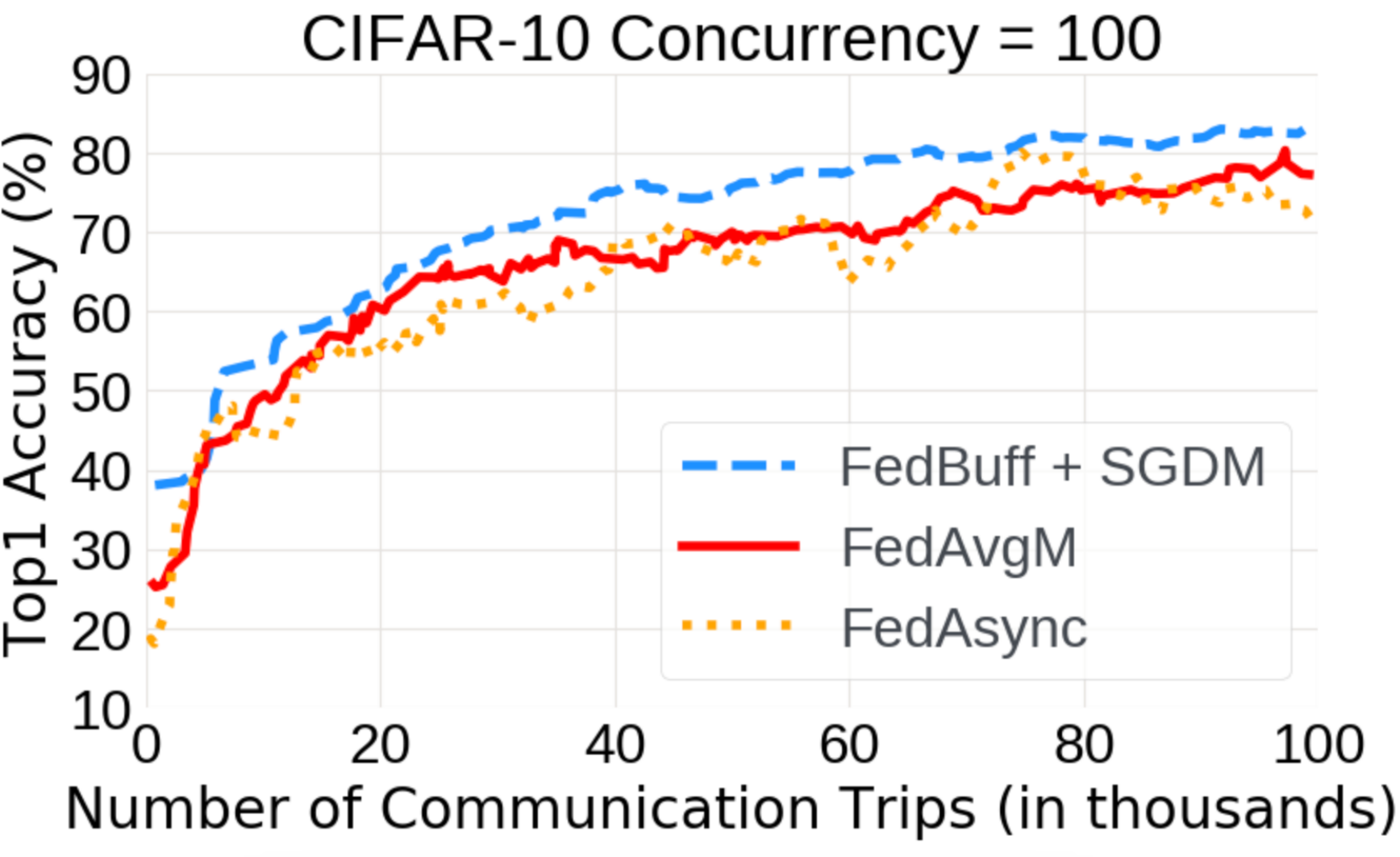}
     \end{minipage}
     \caption{Training accuracy for FedAsync, FedAvgM and \algname on CIFAR-10.}
     \label{fig:cifar_accuracy}
\end{figure*}

\clearpage
\section{Proof of Convergence Rate}
\label{appendix:proof}

\begin{table}[t]
    \caption{Summary of notation}
    \label{tab:notation}
\begin{center}
\begin{tabular}{r c} \toprule
Description & Symbol \\ \midrule
number of server updates, server update index &  $T$, $t$ \\
set of clients updates used in server update $t$ & $\calS^t$ \\
number of clients, client index & $m$, $i$ or $k$ \\
number of local steps per round, round index & $Q$, $q$ \\
server model after $t$ steps & $w^t$ \\
stochastic gradient at client $i$ & $g_i(w; \zeta_i):=g_i(w)$ \\
local learning rate &  $\eta_l$ \\
global learning rate  & $\eta_g$ \\
number of clients in update & $K$ \\
local and global gradient variance & $\sigma_{\ell}^2$, $\sigma_g^2$\\
delay/staleness of client $i$'s model update for the $t$th server update & $\tau_i(t)$ \\
maximum staleness for buffer of size $K$ & $\tau_{\max, K}$ \\
\bottomrule
\end{tabular}
\end{center}
\end{table}

In this appendix, we prove the main convergence result for \algname{}. A summary of the notation used is provided in Table~\ref{tab:notation}.

Observe that \algname{} updates can be described succinctly as
\begin{align*}
w^{t+1} &= w^t + \eta_g \overline{\Delta}^t \\
&= w^t + \eta_g \frac{1}{K} \sum_{k \in \calS^t} \left(- \eta_\ell \sum_{q=1}^Q g_k(y_{k,q}^{t - \tau_k(t)}) \right),
\end{align*}
where $\calS^t$ denotes the set of clients that contribute to the $t$'th server update, and $\tau_k(t) \ge 1$ is the staleness of an update contributed by client $k$ to the $t$'th server update. Specifically, when $k \in \calS^t$, the update returned by client $k$ was computed by starting from $w^{t-\tau_k(t)}$ and performing $Q$ local gradient steps. When $\tau_k(t) = 1$ there is no staleness in the update, and more generally $\tau_k(t) > 1$ corresponds to some staleness; i.e., $t - \tau_k(t)$ server updates have taken place between when the client last pulled a model from the server and when the client's update is being incorporated at the server.

In addition to the assumptions stated in Section~\ref{sec:convergence}, in the proof below we assume that $\calS^t$ is a uniform subset $[n]$; i.e., in any given round any client is equally likely to contribute. This can be justified in practice as follows. To avoid having any client contribute more than once to any update, after the client returns an update contributing to $\overline{\Delta}^t$, the server can only sample that client after the server has performed another update. 

We first state a useful lemma.
\begin{lemma}
$\mathbb{E}\Big[\norm{g_k}^2\Big] \leq 3(\sigma^2_{\ell} + \sigma^2_g + G)$, where the total expectation $\mathbb{E}[\cdot]$ is evaluated over the randomness with respect to client participation and the stochastic gradient taken by a client. 
\label{ref:lemma_totalexpectation}
\end{lemma}

\begin{proof}
From the law of total expectation we have $\mathbb{E} = 
\mathbb{E}_{k\sim[m]}  \mathbb{E}_{\zeta_k | k}$. Hence,
\begin{equation}
    \begin{aligned}
            \mathbb{E}\Big[\norm{g_k(w)}^2\Big] &= \mathbb{E}_{k\sim[m]}  \mathbb{E}_{g | k}\Big[\norm{g_k(w) - \nabla F_k(w) + \nabla F_k(w) - \nabla f(w) + \nabla f(w)}^2\Big] \\
            &\leq 3 \mathbb{E}_{k\sim[m]} \mathbb{E}_{g | k} \Big[\norm{g_k(w) - \nabla F_k(w)}^2 + \norm{\nabla F_k(w) - \nabla f(w)}^2 + \norm{\nabla f(w)}^2 \Big] \\
            &= 3(\sigma^2_{\ell} + \sigma^2_g + G )
    \end{aligned}
\end{equation}
\end{proof}

\subsection{Proof of Theorem \ref{thm:main_ergodic_rate}}
\begin{theorem}
Let $\eta^{(q)}_{\ell}$ be the local learning rate of client SGD in the $q$-th step, and define $\alpha(Q):=\sum_{q=0}^{Q-1} \eta^{(q)}_{\ell}$, $\beta(Q):=\sum_{q=0}^{Q-1} (\eta^{(q)}_{\ell})^2$. Choosing $\eta_g \eta^{(q)}_{\ell} Q \leq \frac{1}{L}$ for all local steps $q=0,\cdots,Q-1$, the global model iterates in Algorithm \ref{alg:server} achieves the following ergodic convergence rate
\begin{equation}
        \frac{1}{T} \sum_{t=0}^{T-1} \norm{\nabla f(w^t)}^2
        \leq \frac{2 \Big(f(w^0) - f(w^*) \Big)}{\eta_g \alpha(Q) T} + 3 L^2 Q  \beta(Q)\Big(\eta^2_g \tau_{\max, K}^2 + 1 \Big)  \Big(\sigma^2_{\ell} + \sigma^2_g +  G\Big)
        +\frac{L}{2}\frac{\eta_g \beta(Q) }{ \alpha(Q)}   \sigma^2_{\ell} .
\end{equation}
\end{theorem}

\begin{proof}
By $L$-smoothness assumption,
\begin{equation}
    \begin{aligned}
        f(w^{t+1}) &\leq f(w^t) - \eta_g \langle \nabla f(w^t), \overline{\Delta}^t \rangle + \frac{L\eta^2_g}{2} \norm{\overline{\Delta}^t}^2 \\
       &\leq   f(w^t) \underbrace{-\frac{\eta_g}{K}\sum_{k\in \calS_t} \Big\langle \nabla f(w^t),  \Delta_{k}^{t-\tau_k} \Big\rangle}_{T_1} + \underbrace{\frac{L\eta^2_g}{2K^2} \norm{\sum_{k\in \calS_t} \Delta_{k}^{t-\tau_k}}^2}_{T_2},
    \end{aligned}
    \label{eq:lipchitz_expansion}
\end{equation}
where $\Delta_{k}^{t-\tau_k}$ is the client delta which is trained from using the global model after $t-\tau_k$ updates as initialization. We will next derive the upper bounds on $T_1$ and $T_2$. To begin,
\begin{equation}
    \begin{aligned}
        T_1 
        &= -\frac{\eta_g}{K} \sum_{k\in \calS_t}  \Big\langle \nabla f(w^t),  \sum_{q=0}^{Q-1} \eta^{(q)}_{\ell} g_k(y^{t-\tau_k}_{k,q}) \Big\rangle
        = - \frac{\eta_g}{K} \sum_{k\in \calS_t}  \sum_{q=0}^{Q-1} \eta^{(q)}_{\ell} \Big\langle \nabla f(w^t),   g_k(y^{t-\tau_k}_{k,q}) \Big\rangle.
    \end{aligned}
\end{equation}
Using conditional expectation, the expectation operator can be written as 
\begin{equation*}
    \mathbb{E}[\cdot]:= \mathbb{E}_{\mathcal{H}} \mathbb{E}_{i \sim [m]} \mathbb{E}_{g_i | i, \mathcal{H}}[\cdot]
\end{equation*}
where $\mathbb{E}_{\mathcal{H}}$ is the expectation over the history of the iterates, $\mathbb{E}_{i \sim [m]}$ is evaluated over the randomness over the distribution of clients $i \sim [m]$ checking in at time-step $t$, and the inner expectation operates over the stochastic gradient of one step on a client. Hence, following unbiasedness, 
\begin{equation*}
    \begin{aligned}
            \mathbb{E}[T_1] =&  - \mathbb{E} \left[  \frac{\eta_g}{K}  \sum_{k\in \calS_t} \sum_{q=0}^{Q-1} \eta^{(q)}_{\ell}  \Big\langle \nabla f(w^t),   g_k(y^{t-\tau_k}_{k,q}) \Big\rangle \right] \\
            =& - \eta_g \mathbb{E}_{\mathcal{H}} \left[ \frac{1}{m}   \sum_{i=1}^m \sum_{q=0}^{Q-1} \eta^{(q)}_{\ell}  \mathbb{E}_{g_i | i\sim[m]}  \Big\langle \nabla f(w^t),    g_i(y^{t-\tau_i}_{i,q})  \Big\rangle \right] \\
            =& - \frac{\eta_g}{m} \mathbb{E}_{\mathcal{H}} \left[  \sum_{i=1}^m \sum_{q=0}^{Q-1} \eta^{(q)}_{\ell} \Big\langle \nabla f(w^t),   \nabla F_i(y^{t-\tau_i}_{i,q})  \Big\rangle \right] \\
            =& - \eta_g \mathbb{E}_{\mathcal{H}} \left[ \sum_{q=0}^{Q-1} \eta^{(q)}_{\ell}  \Big\langle \nabla f(w^t),   \frac{1}{m}\sum_{i=1}^m \nabla F_i(y^{t-\tau_i}_{i,q}) \Big\rangle \right].
    \end{aligned}
\end{equation*}

From the identity 
$$\langle a, b \rangle = \frac{1}{2}(\norm{a}^2 + \norm{b}^2 - \norm{a-b}^2),$$
we have
\begin{equation}
    \begin{aligned}
        \mathbb{E}[T_1] 
        &= -\frac{ \eta_g}{2}\left(\sum_{q=0}^{Q-1} \eta^{(q)}_{\ell} \right) \norm{\nabla f(w^t)}^2 + \sum_{q=0}^{Q-1} \frac{ \eta_g \eta^{(q)}_{\ell}}{2} \Big(-\mathbb{E}_{\mathcal{H}} \norm{\frac{1}{m}\sum_{i=1}^m \nabla F_i(y^{t-\tau_i}_{i,q})}^2 \\ &+\mathbb{E}_{\mathcal{H}}  \underbrace{\norm{\nabla f(w^t) -  \frac{1}{m}\sum_{i=1}^m \nabla F_i(y^{t-\tau_i}_{i,q})}^2 }_{T_3} \Big).
    \end{aligned}
    \label{eq:T1}
\end{equation}
Now for $T_3$, from the definition of $f(w^t)$,
\begin{equation}
    \begin{aligned}
        \mathbb{E}_{\mathcal{H}}[T_3] &= \mathbb{E}_{\mathcal{H}}\norm{ \frac{1}{m}\sum_{i=1}^m  \nabla F_i(w^t) -  \frac{1}{m}\sum_{i=1}^m \nabla F_i(y^{t-\tau_i}_{i,q})}^2 \\
        &\leq \frac{1}{m} \sum_{i=1}^m \mathbb{E}_{\mathcal{H}}\norm{\nabla F_i(w^t) - \nabla F_i(y^{t-\tau_i}_{i,q})}^2.
    \end{aligned}
\end{equation}
Further, by telescoping, $T_3$ can be decomposed as
\begin{equation}
    \begin{aligned}
        \mathbb{E}[T_3] &= \frac{1}{m}\sum_{i=1}^m \mathbb{E}_{\mathcal{H}} \Big|\Big|\nabla F_i(w^t) - \nabla F_i(w^{t-\tau_i}) + \nabla F_i(w^{t-\tau_i}) - \nabla F_i(y^{t-\tau_i}_{i,q}) \Big|\Big|^2\\
        &\leq \frac{2}{m}\sum_{i=1}^m \mathbb{E}_{\mathcal{H}} \Big( \underbrace{\Big|\Big|\nabla F_i(w^t) - \nabla F_i(w^{t-\tau_i})\Big|\Big|^2}_{\text{staleness}} + \underbrace{\Big|\Big|\nabla F_i(w^{t-\tau_i}) - \nabla F_i(y^{t-\tau_i}_{i,q})\Big|\Big|^2}_{\text{local drift}} \Big) \\
        &\leq \frac{2}{m}\sum_{i=1}^m \Big( L^2 \mathbb{E}_{\mathcal{H}}\norm{w^t - w^{t-\tau_i}}^2 + L^2\mathbb{E}_{\mathcal{H}}\norm{w^{t-\tau_i} - y^{t-\tau_i}_{i,q}}^2 \Big).
    \end{aligned}
\end{equation}

The upper bound on $T_3$ can be understood as sums of bounds on the effect of staleness and local drift during client training, and local variance induced by client-side SGD. Further, we need to produce an upper bound on the staleness of initial model from which the client models are trained.
\begin{equation}
    \begin{aligned}
     \norm{w^t - w^{t-\tau_i}}^2 &= \norm{\sum^{t-1}_{\rho=t-\tau_i} (w^{\rho+1} - w_{\rho}) }^2
= \norm{\sum^{t-1}_{\rho=t-\tau_i} \frac{\eta_g}{K} \sum_{j_\rho \in \calS_{\rho}} \Delta_{j_\rho}^{\rho} }^2 \\
&=  \frac{\eta^2_g}{K^2}\norm{\sum^{t-1}_{\rho=t-\tau_i} \sum_{j_\rho \in \calS_{\rho}} \sum_{l=0}^{Q-1} \eta^{(l)}_{\ell}  g_{j_\rho}(y^{\rho}_{j_\rho, l}) }^2.
\end{aligned}
\label{eq:stalenss_UB}
\end{equation}
Taking the expectation in terms of $\mathcal{H}$,
\begin{equation}
    \begin{aligned}
     \mathbb{E}_{\mathcal{H}}\norm{w^t - w^{t-\tau_i}}^2 &\leq \frac{\eta^2_g Q  \tau_i}{K} \sum^{t-1}_{\rho=t-\tau_i} \sum_{j_\rho \in \calS_{\rho}} \sum_{l=0}^{Q-1} (\eta^{(l)}_{\ell} )^2 \mathbb{E}\norm{ g_{j_\rho}(y^{\rho}_{j_\rho, l}) }^2 \\
     &\leq 3 \eta^2_g Q  \max_{\tau_i} \tau_i^2  \Big(\sum_{l=0}^{Q-1} (\eta^{(l)}_{\ell})^2\Big) \Big(\sigma^2_{\ell} + \sigma^2_g + G\Big) \\
&\leq 3 \eta^2_g Q \tau_{\max, K}^2 \Big(\sum_{l=0}^{Q-1} (\eta^{(l)}_{\ell})^2\Big) \Big(\sigma^2_{\ell} + \sigma^2_g  + G\Big),
\end{aligned}
\end{equation}
where the last inequality follows from the assumption on maximal delay and applying Lemma \ref{ref:lemma_totalexpectation}. Similarly, the local drift term can be upper-bounded by
\begin{equation}
     \mathbb{E}\norm{w^{t-\tau_i} - y^{t-\tau_i}_{i,q}}^2 = \mathbb{E}\norm{y^{t-\tau_i}_{i,0} - y^{t-\tau_i}_{i,q}}^2
    \leq \mathbb{E} \norm{\sum_{l=0}^{q-1}  \eta^{(l)}_{\ell}g_i(y^{t-\tau_i}_{i,l}) }^2 
    \leq 3 q \left(\sum_{l=0}^{q-1}  (\eta^{(l)}_{\ell})^2 \right)\Big(\sigma^2_{\ell} + \sigma^2_g + G\Big).
\end{equation}
Thus, the upper bound on $T_3$ becomes:
\begin{equation}
    \begin{aligned}
        \mathbb{E}[T_3]
        &\leq 6 \Bigg(L^2 \eta^2_g Q  \tau_{\max, K}^2  \Big(\sum_{i=0}^{Q-1} (\eta^{(i)}_{\ell})^2\Big) \Big(\sigma^2_{\ell} + \sigma^2_g + G\Big) + L^2 q \left(\sum_{i=0}^{q-1}  (\eta^{(i)}_{\ell})^2 \right)\Big(\sigma^2_{\ell} + \sigma^2_g+ G\Big) \Bigg) \\
        &\leq 6  L^2 \Big(\sum_{i=0}^{Q-1} (\eta^{(i)}_{\ell})^2\Big)(\eta^2_g Q \tau_{\max, K}^2 + q )  \Big(\sigma^2_{\ell} + \sigma^2_g + G\Big) \\
        &\leq 6  L^2 Q \Big(\sum_{i=0}^{Q-1} (\eta^{(i)}_{\ell})^2\Big)(\eta^2_g \tau_{\max, K}^2 + 1 )  \Big(\sigma^2_{\ell} + \sigma^2_g + G\Big).
    \end{aligned}
    \label{eq:T3_bound_final}
\end{equation}
Inserting the upper bound on $T_3$ into (\ref{eq:T1}), we have,
\begin{equation}
     \mathbb{E}[T_1] \leq -\frac{\eta_g}{2}\left(\sum_{q=0}^{Q-1} \eta^{(q)}_{\ell} \right) \norm{\nabla f(w^t)}^2 +  \sum_{q=0}^{Q-1} \frac{\eta_g\eta^{(q)}_{\ell}}{2} \mathbb{E}[T_3] -\sum_{q=0}^{Q-1} \frac{ \eta_g \eta^{(q)}_{\ell}}{2} \mathbb{E}_{\mathcal{H}} \norm{\frac{1}{m}\sum_{i=1}^m \nabla F_i(y^{t-\tau_i}_{i,q})}^2.
\end{equation}
Let $\alpha(Q):=\sum_{q=0}^{Q-1} \eta^{(q)}_{\ell}$ and $\beta(Q):=\sum_{q=0}^{Q-1} (\eta^{(q)}_{\ell})^2$. Then
\begin{equation}
    \mathbb{E}[T_1] \leq - \frac{\eta_g \alpha(Q)}{2} \norm{\nabla f(w^t)}^2 + 3 \eta_g L^2 Q \alpha(Q)  \beta(Q)\Big(\eta^2_g  \tau_{\max, K}^2 + 1 \Big)  \Big(\sigma^2_{\ell} + \sigma^2_g + G\Big) 
    \underbrace{-\sum_{q=0}^{Q-1} \frac{ \eta_g \eta^{(q)}_{\ell}}{2} \mathbb{E}_{\mathcal{H}} \norm{\frac{1}{m}\sum_{i=1}^m \nabla F_i(y^{t-\tau_i}_{i,q})}^2}_{T_4}.
\label{eq:T1_final}
\end{equation}

To derive the upperbound on the R.H.S. of (\ref{eq:lipchitz_expansion}), we now need to upper bound $\mathbb{E}[T_2]$. We proceed by adding and subtracting the expected gradient within the norm,
\begin{equation}
    \begin{aligned}
        \mathbb{E}[T_2] 
        &= \mathbb{E} \Bigg[ \frac{L\eta^2_g}{2K^2}  \norm{\sum_{k\in \calS_t} \sum_{q=0}^{Q-1} \eta^{(q)}_{\ell} g_k(y^{t-\tau_k}_{k,q})}^2 \Bigg] \\
        &= \mathbb{E} \Bigg[ \frac{L\eta^2_g}{2K^2}  \norm{\sum_{k\in \calS_t} \sum_{q=0}^{Q-1} \eta^{(q)}_{\ell}\left(g_k(y^{t-\tau_k}_{k,q}) - \nabla F_k(y^{t-\tau_k}_{k,q}) \right) + \sum_{k\in \calS_t} \sum_{q=0}^{Q-1} \eta^{(q)}_{\ell} \nabla F_k(y^{t-\tau_k}_{k,q}) }^2 \Bigg] \\
        &\overset{\mathrm{(A.)}}{=}  \frac{L\eta^2_g}{2K^2} \mathbb{E} \norm{\sum_{k\in \calS_t} \sum_{q=0}^{Q-1}  \eta^{(q)}_{\ell}\left(g_k(y^{t-\tau_k}_{k,q}) - \nabla F_k(y^{t-\tau_k}_{k,q}) \right) }^2 +  \frac{L\eta^2_g}{2K^2} \mathbb{E}   \norm{\sum_{k\in \calS_t} \sum_{q=0}^{Q-1}  \eta^{(q)}_{\ell} \nabla F_k(y^{t-\tau_k}_{k,q}) }^2 \\
        &\overset{\mathrm{(B.)}}{=}  \frac{L\eta^2_g}{2} \sum_{k\in \calS_t} \sum_{q=0}^{Q-1} (\eta^{(q)}_{\ell})^2 \mathbb{E} \norm{\left(g_k(y^{t-\tau_k}_{k,q}) - \nabla F_k(y^{t-\tau_k}_{k,q}) \right) }^2 +  \frac{L\eta^2_g}{2K^2} \mathbb{E}   \norm{\sum_{k\in \calS_t} \sum_{q=0}^{Q-1}  \eta^{(q)}_{\ell} \nabla F_k(y^{t-\tau_k}_{k,q}) }^2 \\
        &\leq  \frac{L\eta_g^2 \beta(Q) \sigma^2_{\ell}}{2} + \frac{LQ\eta_g^2}{2K} \sum_{k\in \calS_t} \sum_{q=0}^{Q-1}  (\eta^{(q)}_{\ell})^2\mathbb{E}_{\mathcal{H}} \mathbb{E}_{k\sim[m] | \mathcal{H}}\norm{ \nabla F_k(y^{t-\tau_k}_{k,q}) }^2 \\
        &=  \frac{L\eta_g^2 \beta(Q) \sigma^2_{\ell}}{2} + \frac{LQ\eta_g^2}{2K} \sum_{k\in \calS_t} \sum_{q=0}^{Q-1}  (\eta^{(q)}_{\ell})^2\mathbb{E}_{\mathcal{H}} \Bigg[ \frac{1}{m}\sum_{i=1}^m\norm{ \nabla F_i(y^{t-\tau_i}_{i,q}) }^2 \Bigg] \\
        &=  \frac{L\eta_g^2 \beta(Q) \sigma^2_{\ell}}{2} + \underbrace{\frac{LQ\eta_g^2}{2m}  \sum_{q=0}^{Q-1}  \sum_{i=1}^m (\eta^{(q)}_{\ell})^2 \mathbb{E}_{\mathcal{H}} \Bigg[ \norm{ \nabla F_i(y^{t-\tau_i}_{i,q}) }^2 \Bigg]}_{T_5} \\        
    \end{aligned}
\label{eq:T2_final}
\end{equation}
where (A.) follows the unbiasedness of $g_k$, and (B.) follows from the fact that $g_k - \nabla F_k$ are independent and unbiased for $k \sim [m]$.
To produce an upperbound on $\mathbb{E}[T_1 + T2]$, we need to make sure $T_4 + T_5 \leq 0$.
\begin{equation}
    \begin{aligned}
        &\big(T_4 + T_5\big) \\ =&  -\sum_{q=0}^{Q-1} \frac{ \eta_g \eta^{(q)}_{\ell}}{2} \mathbb{E}_{\mathcal{H}} \norm{\frac{1}{m}\sum_{i=1}^m \nabla F_i(y^{t-\tau_i}_{i,q})}^2 + \frac{LQ\eta_g^2}{2m}  \sum_{q=0}^{Q-1}  \sum_{i=1}^m (\eta^{(q)}_{\ell})^2 \mathbb{E}_{\mathcal{H}}  \norm{ \nabla F_i(y^{t-\tau_i}_{i,q}) }^2  \\      
        =& -\sum_{q=0}^{Q-1}\sum_{i=1}^m  \frac{\eta_g \eta^{(q)}_{\ell}}{2m} \mathbb{E}_{\mathcal{H}}  \norm{ \nabla F_i(y^{t-\tau_i}_{i,q})}^2 + \frac{LQ\eta_g^2}{2m}  \sum_{q=0}^{Q-1}  \sum_{i=1}^m (\eta^{(q)}_{\ell})^2 \mathbb{E}_{\mathcal{H}}  \norm{ \nabla F_i(y^{t-\tau_i}_{i,q}) }^2  \\
        =& \sum_{q=0}^{Q-1}\sum_{i=1}^m \Big( -\frac{\eta_g \eta^{(q)}_{\ell}}{2m} + \frac{LQ\eta_g^2 (\eta^{(q)}_{\ell})^2}{2m}  \Big) \mathbb{E}_{\mathcal{H}}  \norm{ \nabla F_i(y^{t-\tau_i}_{i,q})}^2 
    \end{aligned}
    \label{eq:local_lr_condition}
\end{equation}
To ensure $T_4 + T_5 \leq 0$, it is sufficient to choose $\eta_g \eta^{(q)}_{\ell} Q \leq \frac{1}{L}$ for all local steps $q=0,\cdots,Q-1$.

Now, plugging (\ref{eq:T1_final}), (\ref{eq:T2_final}) and (\ref{eq:local_lr_condition}) into     (\ref{eq:lipchitz_expansion}),

\begin{equation}
        \mathbb{E}[f(w^{t+1})] \leq  \mathbb{E}[f(w^t)] - \frac{\eta_g \alpha(Q)}{2} \norm{\nabla f(w^t)}^2 + 3\eta_g L^2 Q \alpha(Q)  \beta(Q)\Big(\eta^2_g \tau_{\max, K}^2 + 1 \Big)  \Big(\sigma^2_{\ell} + \sigma^2_g + G\Big) 
        + \frac{L}{2} \eta^2_g \beta(Q)    \sigma^2_{\ell}
    \label{eq:obj_bound}
\end{equation}
Summing up $t$ from $1$ to $T$ and rearrange, yields
\begin{equation}
    \begin{aligned}
        \sum_{t=0}^{T-1} \eta_g  \alpha(Q)\norm{\nabla f(w^t)}^2 &\leq   \sum_{t=0}^{T-1} 2 \Big(\mathbb{E}[f(w^t)]  - \mathbb{E}[f(w^{t+1})] \Big)  + 3\sum_{t=0}^{T-1} \eta_g L^2  Q \alpha(Q)  \beta(Q)\Big(\eta^2_g  \tau_{\max, K}^2 + 1 \Big)  \Big(\sigma^2_{\ell} + \sigma^2_g + G\Big) \\
        &+  \frac{L}{2} \eta^2_g \beta(Q)   \sigma^2_{\ell}\\
        &\leq 2 \Big(f(w^0) - f(w^*) \Big) + 3\sum_{t=0}^{T-1} \eta_g L^2 \alpha(Q)  \beta(Q)\Big(\eta^2_g  \tau_{\max, K}^2 + Q \Big)  \Big(\sigma^2_{\ell} + \sigma^2_g + G\Big) 
        +  \frac{L}{2}\eta^2_g \beta(Q)  \sigma^2_{\ell}\\
    \end{aligned}
\end{equation}
Thus we have
\begin{equation}
        \frac{1}{T} \sum_{t=0}^{T-1} \norm{\nabla f(w^t)}^2
        \leq \frac{2 \Big(f(w^0) - f(w^*) \Big)}{\eta_g \alpha(Q) T} + 3 L^2 Q  \beta(Q)\Big(\eta^2_g \tau_{\max, K}^2 + 1 \Big)  \Big(\sigma^2_{\ell} + \sigma^2_g +  G\Big)       +\frac{L}{2}\frac{\eta_g \beta(Q) }{ \alpha(Q)}   \sigma^2_{\ell} 
\end{equation}
\end{proof}

%% file: extra_experiments.tex
\begin{table}
\centering
\caption{Speed up of FedBuff over FedAvgM and FedAsync with regards to number of client trips to reach target validation accuracy, for different delay distributions. We set concurrency = 1000 for all methods and $K=10$ for FedBuff. FedBuff's speed up is consistent across delay distributions.}
\label{tab:robustness}
\begin{tabular}{llrrlll} 
\toprule
Dataset & Delay Distribution & Speedup over FedAvgM & Speedup over FedAsync  \\ 
\hline
\Tstrut{}
        &   Uniform           & 4.7$\times$                        & 1.6$\times$       \\
CelebA  &   Half-Normal       & 3.3$\times$                        & 1.2$\times$      \\
        &   Exponential       & 4.3$\times$                        & 1.1$\times$      \\ 
\hline
\Tstrut{}
        &   Uniform             & 1.3$\times$                       &   1.2$\times$    \\
Sent140 &   Half-Normal         & 1.7$\times$                       &   2.5$\times$    \\
        &   Exponential         & 1.4$\times$                       &   2.0$\times$    \\
\bottomrule
\end{tabular}
\end{table}

\begin{figure}[t]
\centering
	\includegraphics[width=0.75\textwidth]{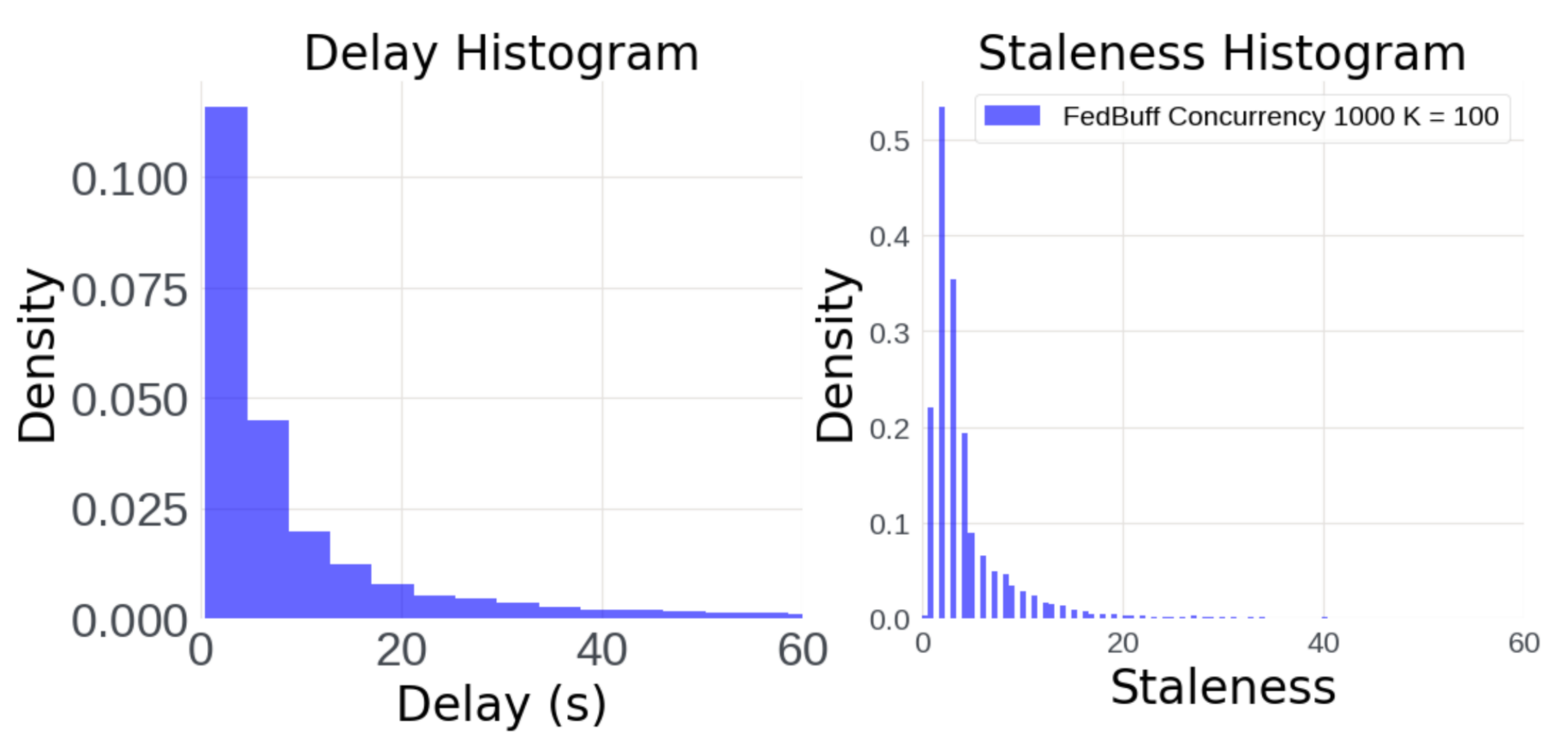}
	\caption{Delay and staleness distributions observed in production when training over millions of real clients for \algname{}.}
	\label{fig:prod_staleness}
\end{figure}

\subsection{Robustness to Delay Distributions.}
\label{sec:other_training_distributions}
In this section, we analyze the sensitivity of \algname{} to different staleness distributions. We compare \algname{} against other competing algorithms with different staleness distributions. Table \ref{tab:robustness} demonstrates that \algname{} is robust and \algname{}'s speed up is consistent. 
To have an accurate view of real-world delays, we observe the delays and their resulting staleness distribution in our production stack when training over millions of clients with concurrency = 1000 and K = 100. \Cref{fig:prod_staleness} demonstrates that a half-normal is a suitable delay distribution.

\subsection{Wall-Clock Time Simulation}
\label{sec:appendix_stragglers}

In this section, we study the speed up of FedBuff over FedAvgM in terms of wall-clock time for various concurrency levels. The results for Sent140 are in Figure \ref{fig:stragglers}. 

In cross-device FL, each client is a mobile phone with limited compute power and communication bandwidth \citep{fl-survey}. Moreover, clients can have vastly different number of examples. Recall, SyncFL methods wait for all the participating clients in a round to finish before updating the server model -- a round proceeds at the pace of the slowest client, the straggler effect. To mitigate the straggler problem, over-selection proposed in \citet{google-fl}, which selects 30\% more clients than the target number of clients to participate and waits for the fastest replies. 

To confirm the speedup gain by \algname in terms of wall-clock time, we simulate training time of FedAvgM and FedBuff using a random exponential time model from \cite{straggler_model}. The random exponential time model has been widely used to simulate the straggler effect in federated learning, e.g. in \cite{reisizadeh2020straggler, tandon2017gradient, yu2020straggler, yu2019lagrange, charles2021large}. 

We assume the time a client requires to perform local training is proportional to the number of examples the client has, same as the assumption in \cite{charles2021large}. Formally, let $n_i$ be the number of examples held by client $i$, and let $T_i$ be the amount of time required by client $i$ to perform local training. Also assume that there is a constant $\lambda > 0$ such that 
\[
T_i \sim Exp( \frac{1}{\lambda n_i}).
\]

In this context, $\lambda$ is the straggler parameter. The larger the $\lambda$, the longer the expected client training time. For a given round $t$, let $C_t$ be number of clients required to close a round, and let $M$ be the total number of clients training concurrently with over-selection. If $T_1, \dots, T_M$ denote the raw times when clients complete the round and $T_{(1)} \le \dots \le T_{(M)}$ denote order statistic of the client training time, then $R_t$ for SyncFL is 
\[
R_t = T_{(C_t)}.
\]

If $T$ is the number of rounds to reach a target accuracy, the expected total training time for SyncFL is
\[
\mathbb{E}[T_{\text{SyncFL}}] = \sum_{t=0}^{T} R_t.
\]

In the case of \algname, the total wall-clock time is when the last client required to reach a target accuracy finishes training. Formally, let $N$ be the number of clients required to reach a target accuracy and $T_{(1)} \le \dots \le T_{(N)}$ denotes the order statistic of the client training time. Then the expected total training time for \algname is

\[
\mathbb{E}[T_{\text{\algname}}] = T_{(N)}.
\]

\begin{figure}[t]
     \centering
     \includegraphics[width=0.8\linewidth]{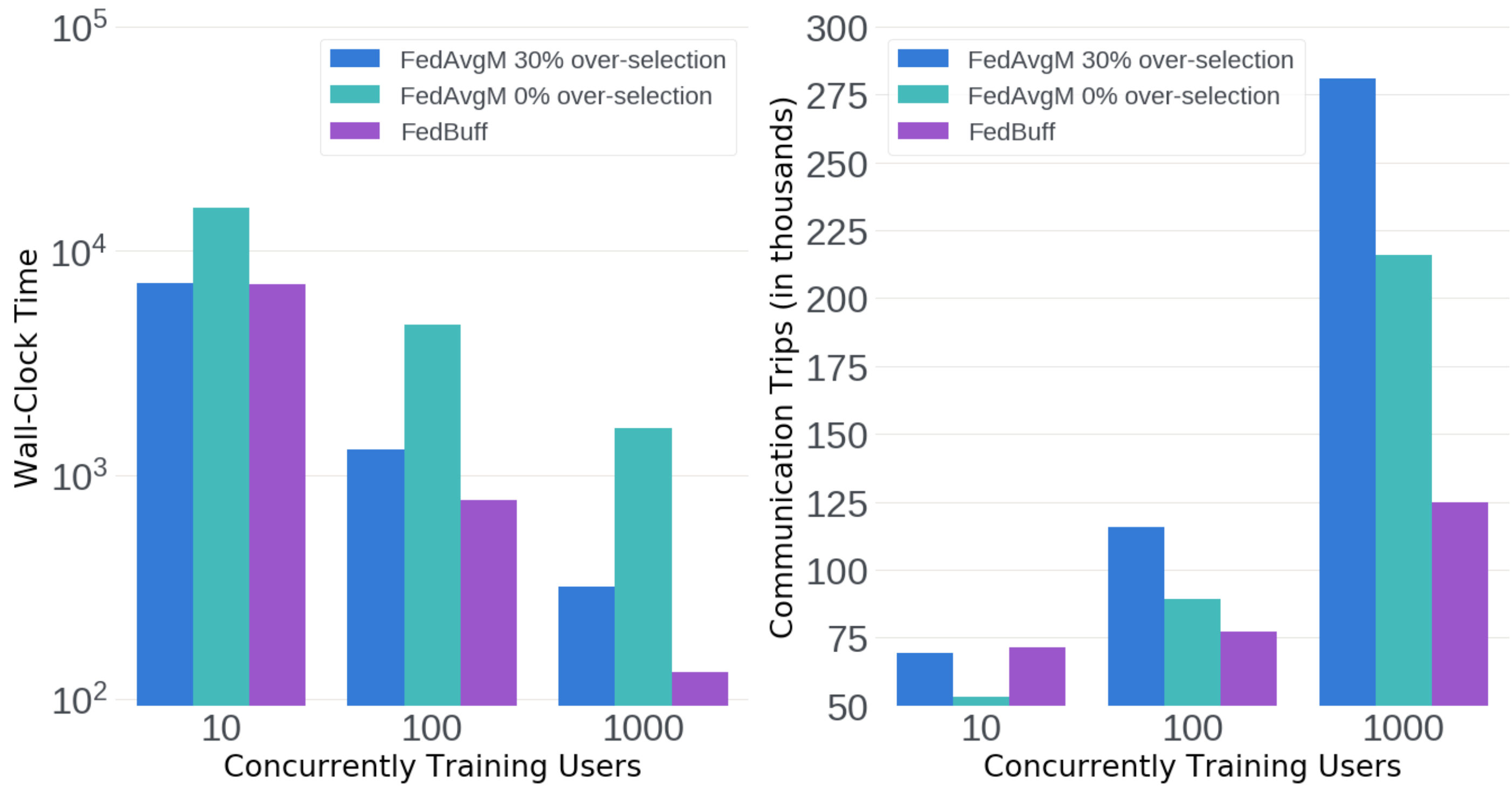}
     \label{fig:sent140_stragglers}
     \caption{(left) The total run time required to reach a target accuracy on Sent140 with $\lambda = 1$ under a random exponential time model. (right) The number of client trips required to reach a target accuracy. This measures the resource efficiency of the three algorithms. In all three configurations, increasing concurrency lowers wall-clock convergence time. However, \syncfl uses much more resources compared with \asyncfl. This highlights the importance of scalability, the ability to efficiently utilize increasing number of clients training in parallel. Both figures illustrates that FedBuff is faster than SyncFL (e.g., FedAvgM) even with over-selection with increasing concurrency, while being up to 3 times more resource efficient.}
    \label{fig:stragglers}
\end{figure}

\subsection{Bias}
\label{sec:appendix_diversity}

\begin{figure}[t]
     \centering
     \includegraphics[width=1\linewidth]{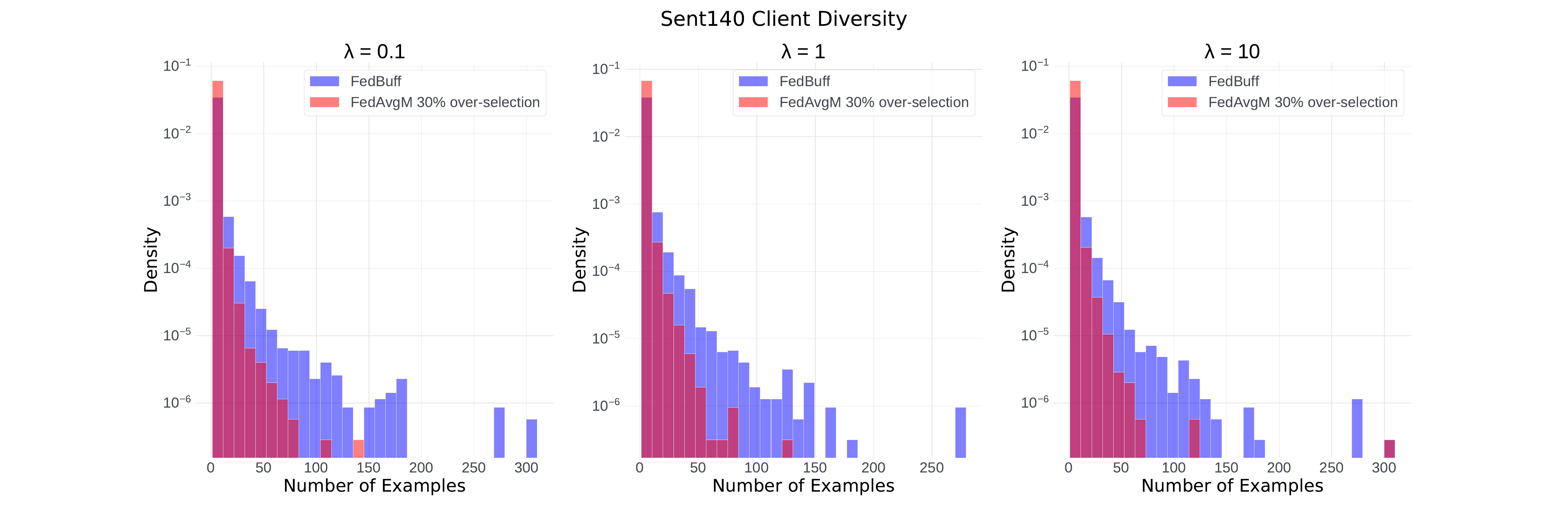}
     \label{fig:sent140_diversity}
     \caption{The distribution of Number of examples held by each participating client on Sent140 with varying $\lambda$. We fix concurrency = 1000 and buffer size $K = 10$. Since \algname does not drop the slowest clients, it can incorporate a more diverse set of clients.}
    \label{fig:diversity}
\end{figure}

In this section, we show the diversity of participating clients in \algname and FedAvgM. We should that SyncFL methods can introduce bias in their selection process while \algname does not. We use the same random exponential time model in Appendix \ref{sec:appendix_stragglers}. The result is given in Figure \ref{fig:diversity}. We find that in all levels of $\lambda$, \algname can incorporate clients with large local datasets. On the other hand, SyncFL (e.g., FedAvgM) with over-selection drops these clients leading to bias in the selection process. \algname does not drop the slowest clients, and can incorporate these clients with many examples.

\subsection{Large values of $K$}
\label{sec:appendix_large_k}

\begin{table}
\caption{Wall-clock time to reach target validation accuracy on CelebA and Sent140 when $K$ is large (Units for wall-clock time: mean training time for one client. Units for client trips updates: 1000 updates). For \algname{}, $K$=1000. FedAvgM with over-selection throws away results from the slowest 30\% of users in each round. These users are included when calculating the number of client trips}
\label{tab:large_k}
\centering
\begin{tabular}{lrrlr} 
\toprule
Dataset & Algorithm & Concurrency & Wall-Clock Time & client trips \\ 
\hline
\Tstrut{}
        &   FedBuff ($K$=1000)             & 1000    & 124            & 124\\
CelebA  &   FedAvgM                        & 1000    & 446 (3.6$\times$)     & 104\\
        &   FedAvgM, over-selection        & 1300    & 155 (1.25$\times$)    & 135\\ 
\hline
\Tstrut{}
        &   FedBuff ($K$=1000)             & 1000    & 228            & 228\\
Sent140 &   FedAvgM                        & 1000    & 927 (4.06$\times$)    & 216\\
        &   FedAvgM. over-selection    & 1300    & 322 (1.41$\times$)    & 281\\
        \bottomrule
\end{tabular}
\end{table}

In this section, we study the performance of \algname with large values of $K$. In Table \ref{tab:different_ks}, we observe that \algname{} trains fast when running with small values of $K$, relative to the concurrency. However, large values of $K$ are useful when providing user-level differential privacy, as essentially the noise is divided among larger number of clients (larger values of $K$)~\citep{dpftrl, google-dp-fl}. 

We compare the training speed of \algname{} and FedAvgM in a setting where both algorithms produce a server update from the same number of aggregated client updates. We fix concurrency at 1000, and have both FedAvgM and \algname{} perform updates after aggregating responses from $K=1000$ clients. In this setting, \algname{}'s main advantage is robustness to stragglers. It cannot take advantage of frequent server updates, yet still needs to deal with staleness. 


The synchronous FL system described in~\cite{google-fl} uses over-selection, typically by 30\%, to address stragglers. For example, if 1000 users are needed to produce a server model update, 1300 users are selected. The round will finish when the fastest 1000 users finish training. Results from the slowest 300 users will be thrown away. Over-selection makes synchronous FL more robust to stragglers, but at the cost of wasting some clients' compute and bandwidth. 

Table~\ref{tab:large_k} reports the wall-clock training time and number of client trips to reach target accuracy for \algname and FedAvgM with and without over-selection. We assume a half-normal training duration distribution since that matches the behavior observed in our production system (see Figure~\ref{fig:prod_staleness}). We find that over-selection reduces the impact of stragglers significantly. However, even with over-selection, \algname{} is 25\%-41\% faster than FedAvgM, despite using 30\% lower concurrency.

\subsection{\algname with Differential Privacy}
\label{sec:appendix_dp}

\begin{figure*}[t]
     \centering
     \begin{minipage}[t]{0.32\linewidth}
         \centering
         \includegraphics[width=\linewidth]{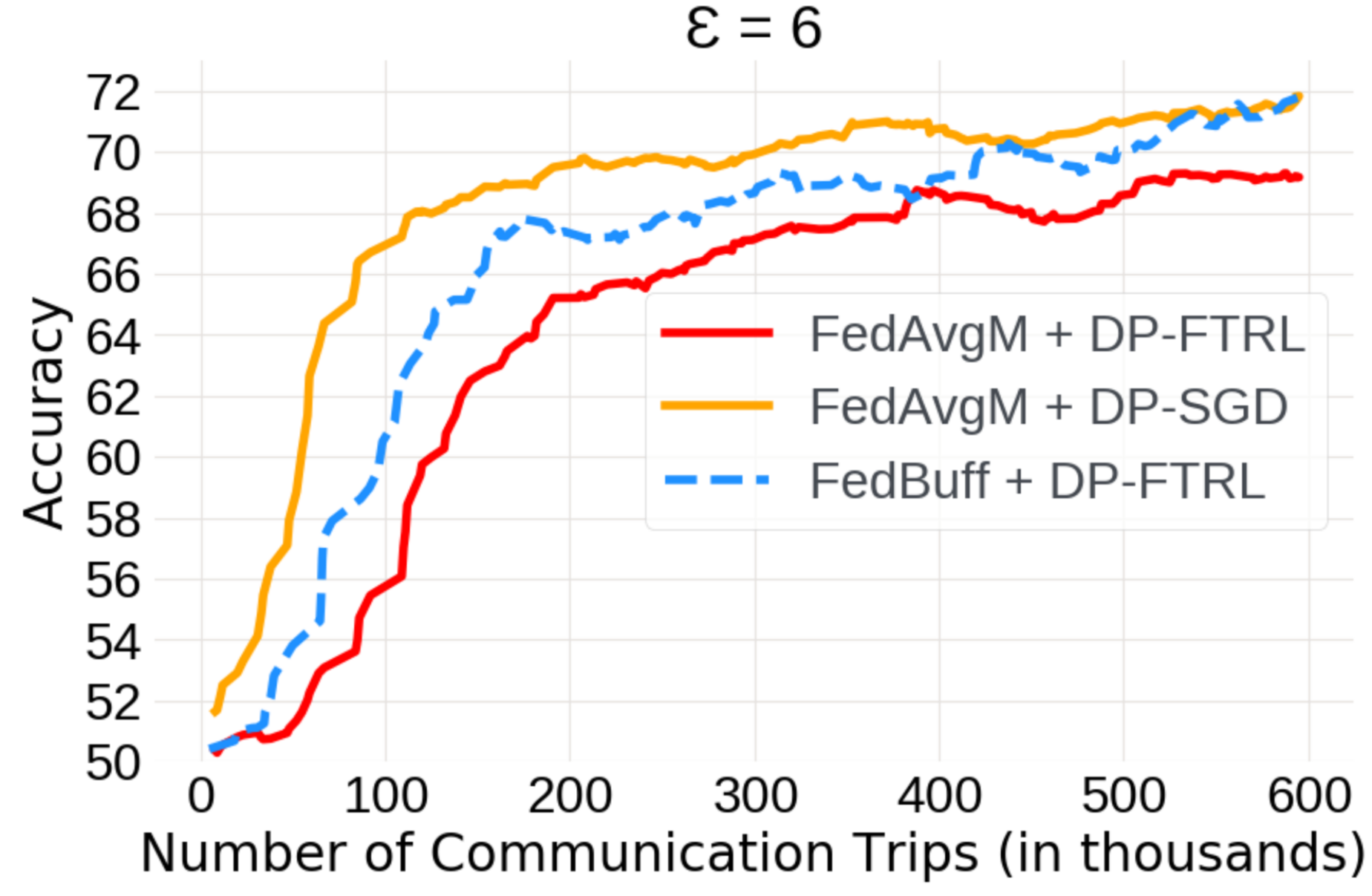}
     \end{minipage}
     \begin{minipage}[t]{0.32\linewidth}
        \centering
        \includegraphics[width=\linewidth]{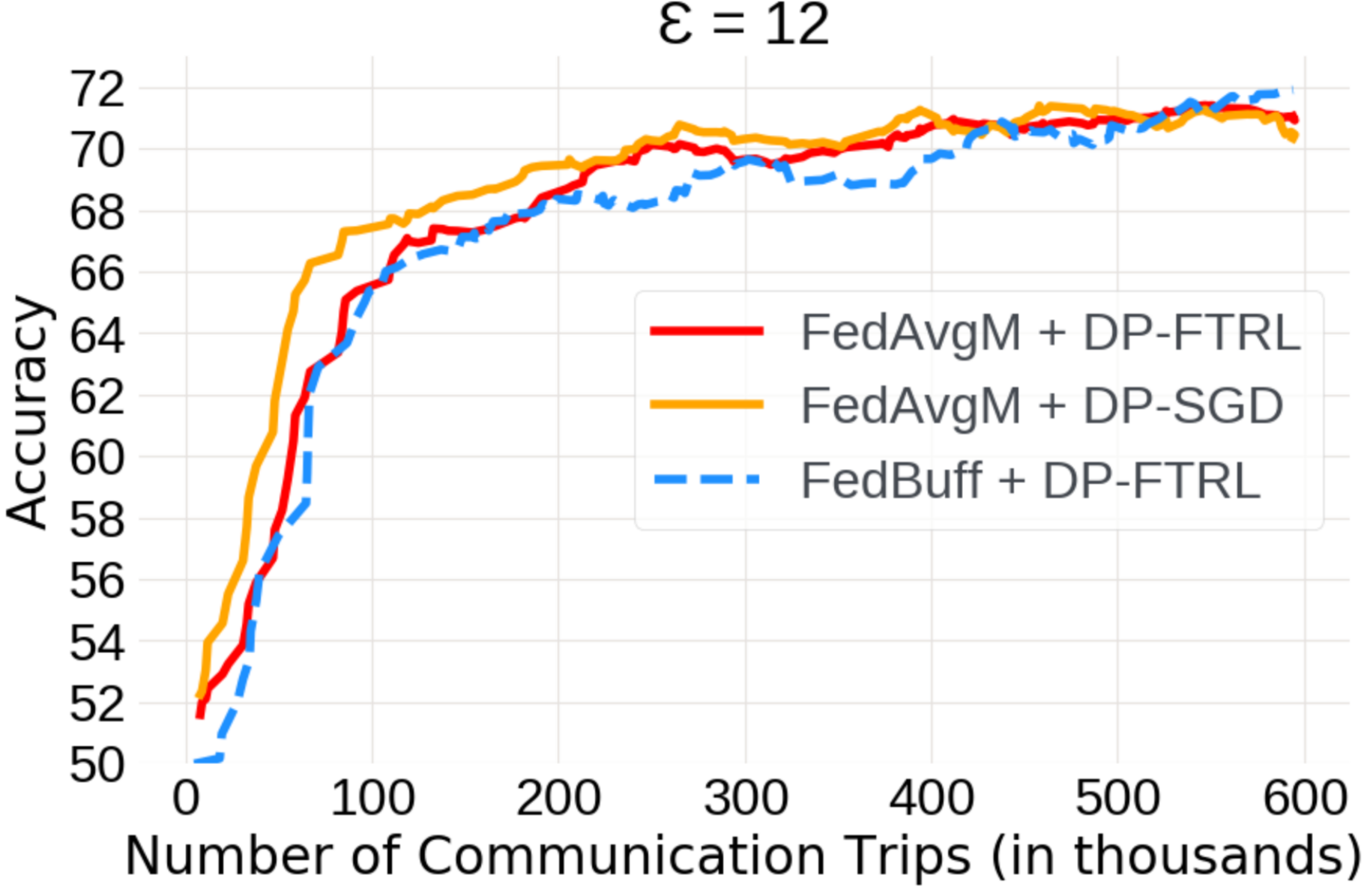}
     \end{minipage}
        \begin{minipage}[t]{0.32\linewidth}
        \centering
        \includegraphics[width=\linewidth]{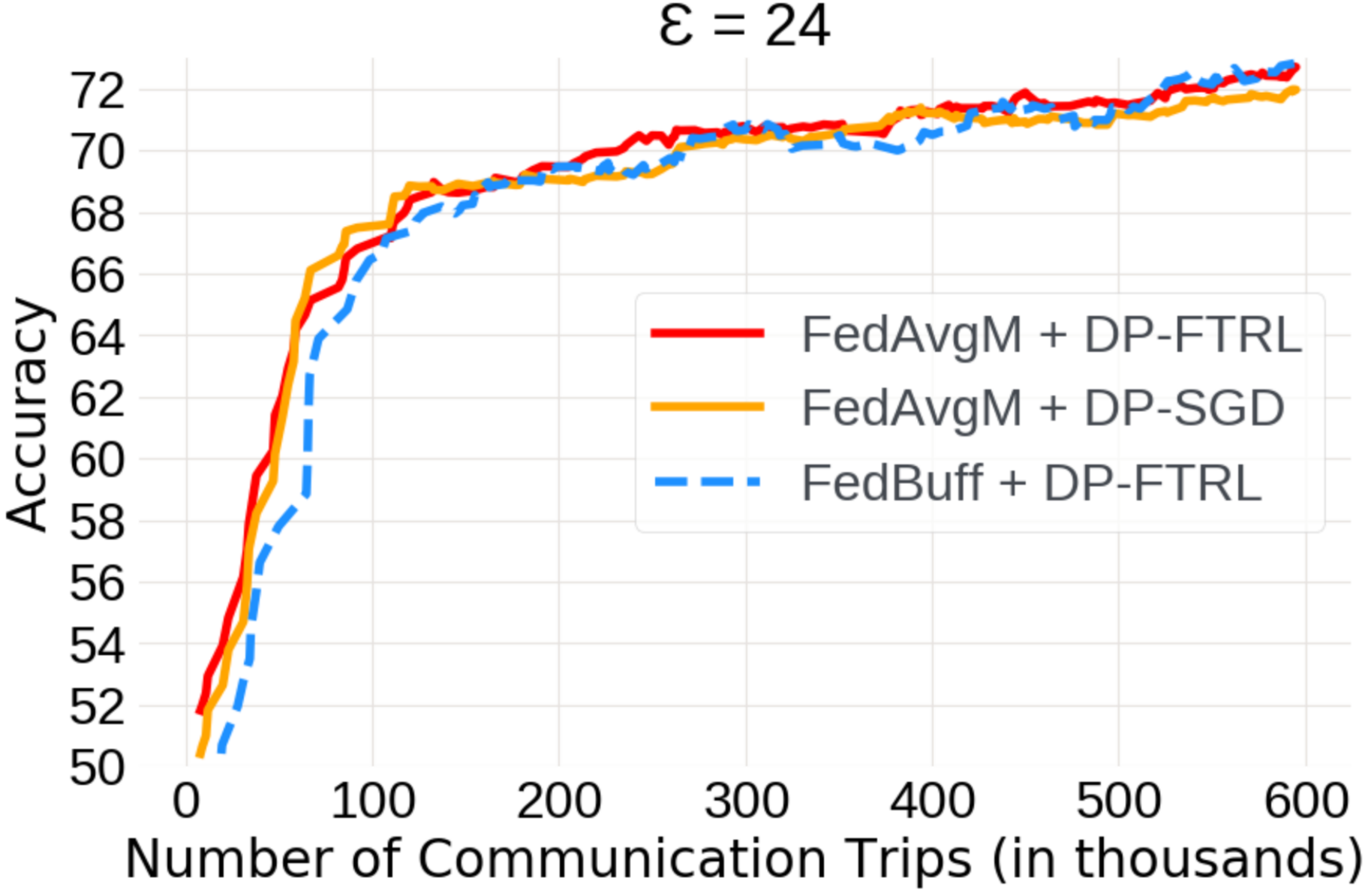}
     \end{minipage}
     \caption{Training curves for different values of $\epsilon$ for three FL configurations.}
     \label{fig:dp_learning_curves}
\end{figure*}

\Cref{fig:dp_learning_curves} shows the training curves of \algname with DP-FTRL, SyncFL with DP-SGD and DP-FTRL. At low values of $\epsilon$, \algname can achieve the same utility as SyncFL with amplified DP-SGD at the cost of slower convergence. At the same $\epsilon$, \algname with DP-FTRL achieves better utility and faster convergence compared to SyncFL with DP-FTRL. We find the source for this speed-up is from \algname's ability to tolerate much lower clipping norm value $L$. We repeat each experiment for 3 different seeds and take the average. The seeds are 0, 1, and 2.

%% file: main.bbl
\begin{thebibliography}{78}
\providecommand{\natexlab}[1]{#1}
\providecommand{\url}[1]{\texttt{#1}}
\expandafter\ifx\csname urlstyle\endcsname\relax
  \providecommand{\doi}[1]{doi: #1}\else
  \providecommand{\doi}{doi: \begingroup \urlstyle{rm}\Url}\fi

\bibitem[Abadi et~al.(2016)Abadi, Chu, Goodfellow, McMahan, Mironov, Talwar,
  and Zhang]{dp-sgd}
M.~Abadi, A.~Chu, I.~Goodfellow, H.~B. McMahan, I.~Mironov, K.~Talwar, and
  L.~Zhang.
\newblock Deep learning with differential privacy.
\newblock In \emph{Proceedings of the 2016 ACM SIGSAC conference on computer
  and communications security}, pages 308--318, 2016.

\bibitem[Assran et~al.(2020)Assran, Aytekin, Feyzmahdavian, Johansson, and
  Rabbat]{advance_async_Mike}
M.~Assran, A.~Aytekin, H.~R. Feyzmahdavian, M.~Johansson, and M.~G. Rabbat.
\newblock Advances in asynchronous parallel and distributed optimization.
\newblock \emph{Proceedings of the IEEE}, 108\penalty0 (11):\penalty0
  2013--2031, 2020.

\bibitem[Bell et~al.(2020)Bell, Bonawitz, Gasc{\'o}n, Lepoint, and
  Raykova]{bell2020secure}
J.~H. Bell, K.~A. Bonawitz, A.~Gasc{\'o}n, T.~Lepoint, and M.~Raykova.
\newblock Secure single-server aggregation with (poly) logarithmic overhead.
\newblock In \emph{Proceedings of the 2020 ACM SIGSAC Conference on Computer
  and Communications Security}, pages 1253--1269, 2020.

\bibitem[Bertsekas and Tsitsiklis(1989)]{bertsekasTsitsiklis}
D.~P. Bertsekas and J.~N. Tsitsiklis.
\newblock \emph{Parallel and Distributed Computation: Numerical Methods}.
\newblock Prentice-Hall, 1989.

\bibitem[Bittau et~al.(2017)Bittau, Erlingsson, Maniatis, Mironov, Raghunathan,
  Lie, Rudominer, Kode, Tinnes, and Seefeld]{prochlo}
A.~Bittau, {\'U}.~Erlingsson, P.~Maniatis, I.~Mironov, A.~Raghunathan, D.~Lie,
  M.~Rudominer, U.~Kode, J.~Tinnes, and B.~Seefeld.
\newblock Prochlo: Strong privacy for analytics in the crowd.
\newblock In \emph{Proceedings of the 26th Symposium on Operating Systems
  Principles}, pages 441--459, 2017.

\bibitem[Bonawitz et~al.(2019)Bonawitz, Eichner, Grieskamp, Huba, Ingerman,
  Ivanov, Kiddon, Kone{\v{c}}n{\`y}, Mazzocchi, McMahan, et~al.]{google-fl}
K.~Bonawitz, H.~Eichner, W.~Grieskamp, D.~Huba, A.~Ingerman, V.~Ivanov,
  C.~Kiddon, J.~Kone{\v{c}}n{\`y}, S.~Mazzocchi, H.~B. McMahan, et~al.
\newblock Towards federated learning at scale: System design.
\newblock \emph{arXiv preprint arXiv:1902.01046}, 2019.

\bibitem[Bonawitz et~al.(2016)Bonawitz, Ivanov, Kreuter, Marcedone, McMahan,
  Patel, Ramage, Segal, and Seth]{secagg}
K.~A. Bonawitz, V.~Ivanov, B.~Kreuter, A.~Marcedone, H.~B. McMahan, S.~Patel,
  D.~Ramage, A.~Segal, and K.~Seth.
\newblock Practical secure aggregation for federated learning on user-held
  data.
\newblock In \emph{NIPS Workshop on Private Multi-Party Machine Learning},
  2016.
\newblock URL \url{https://arxiv.org/abs/1611.04482}.

\bibitem[Caldas et~al.(2018)Caldas, Duddu, Wu, Li, Kone{\v{c}}n{\`y}, McMahan,
  Smith, and Talwalkar]{leaf}
S.~Caldas, S.~M.~K. Duddu, P.~Wu, T.~Li, J.~Kone{\v{c}}n{\`y}, H.~B. McMahan,
  V.~Smith, and A.~Talwalkar.
\newblock Leaf: A benchmark for federated settings.
\newblock \emph{arXiv preprint arXiv:1812.01097}, 2018.

\bibitem[Carlini et~al.(2020)Carlini, Tram{\`{e}}r, Wallace, Jagielski,
  Herbert{-}Voss, Lee, Roberts, Brown, Song, Erlingsson, Oprea, and
  Raffel]{extraction}
N.~Carlini, F.~Tram{\`{e}}r, E.~Wallace, M.~Jagielski, A.~Herbert{-}Voss,
  K.~Lee, A.~Roberts, T.~B. Brown, D.~Song, {\'{U}}.~Erlingsson, A.~Oprea, and
  C.~Raffel.
\newblock Extracting training data from large language models.
\newblock \emph{CoRR}, abs/2012.07805, 2020.
\newblock URL \url{https://arxiv.org/abs/2012.07805}.

\bibitem[Chai et~al.(2020)Chai, Chen, Zhao, Cheng, and Rangwala]{paper2}
Z.~Chai, Y.~Chen, L.~Zhao, Y.~Cheng, and H.~Rangwala.
\newblock Fedat: A communication-efficient federated learning method with
  asynchronous tiers under non-iid data.
\newblock \emph{arXiv preprint arXiv:2010.05958}, 2020.

\bibitem[Charles et~al.(2021)Charles, Garrett, Huo, Shmulyian, and
  Smith]{charles2021large}
Z.~Charles, Z.~Garrett, Z.~Huo, S.~Shmulyian, and V.~Smith.
\newblock On large-cohort training for federated learning.
\newblock \emph{arXiv preprint arXiv:2106.07820}, 2021.

\bibitem[Chaturapruek et~al.(2015)Chaturapruek, Duchi, and
  R{\'e}]{Duchi_AsyncSGD_convex}
S.~Chaturapruek, J.~C. Duchi, and C.~R{\'e}.
\newblock Asynchronous stochastic convex optimization: the noise is in the
  noise and sgd don't care.
\newblock \emph{Advances in Neural Information Processing Systems},
  28:\penalty0 1531--1539, 2015.

\bibitem[Chen et~al.(2016)Chen, Monga, Bengio, and
  Jozefowicz]{Revisit_dist_SGD_google}
J.~Chen, R.~Monga, S.~Bengio, and R.~Jozefowicz.
\newblock Revisiting distributed synchronous sgd.
\newblock In \emph{International Conference on Learning Representations
  Workshop Track}, 2016.
\newblock URL \url{https://arxiv.org/abs/1604.00981}.

\bibitem[Chen et~al.(2019)Chen, Ning, Slawski, and Rangwala]{paper1}
Y.~Chen, Y.~Ning, M.~Slawski, and H.~Rangwala.
\newblock Asynchronous online federated learning for edge devices with non-iid
  data.
\newblock \emph{arXiv preprint arXiv:1911.02134}, 2019.

\bibitem[Dutta et~al.(2018)Dutta, Joshi, Ghosh, Dube, and
  Nagpurkar]{dutta2018slow}
S.~Dutta, G.~Joshi, S.~Ghosh, P.~Dube, and P.~Nagpurkar.
\newblock Slow and stale gradients can win the race: Error-runtime trade-offs
  in distributed sgd.
\newblock In \emph{International Conference on Artificial Intelligence and
  Statistics}, pages 803--812. PMLR, 2018.

\bibitem[Dutta et~al.(2021)Dutta, Wang, and Joshi]{dutta2020slow}
S.~Dutta, J.~Wang, and G.~Joshi.
\newblock Slow and stale gradients can win the race.
\newblock \emph{IEEE Journal on Selected Areas in Information Theory},
  2\penalty0 (3):\penalty0 1012--1024, Sep. 2021.

\bibitem[Dwork et~al.(2014)Dwork, Roth, et~al.]{dp}
C.~Dwork, A.~Roth, et~al.
\newblock The algorithmic foundations of differential privacy.
\newblock \emph{Foundations and Trends in Theoretical Computer Science},
  9\penalty0 (3-4):\penalty0 211--407, 2014.

\bibitem[Erlingsson et~al.(2020)Erlingsson, Feldman, Mironov, Raghunathan,
  Song, Talwar, and Thakurta]{esa_empirical}
{\'U}.~Erlingsson, V.~Feldman, I.~Mironov, A.~Raghunathan, S.~Song, K.~Talwar,
  and A.~Thakurta.
\newblock Encode, shuffle, analyze privacy revisited: Formalizations and
  empirical evaluation.
\newblock \emph{arXiv preprint arXiv:2001.03618}, 2020.

\bibitem[Geiping et~al.(2020)Geiping, Bauermeister, Dr{\"o}ge, and
  Moeller]{geiping2020inverting}
J.~Geiping, H.~Bauermeister, H.~Dr{\"o}ge, and M.~Moeller.
\newblock Inverting gradients--how easy is it to break privacy in federated
  learning?
\newblock \emph{arXiv preprint arXiv:2003.14053}, 2020.

\bibitem[Go et~al.(2009)Go, Bhayani, and Huang]{sent140}
A.~Go, R.~Bhayani, and L.~Huang.
\newblock Twitter sentiment classification using distant supervision.
\newblock \emph{CS224N project report, Stanford}, 1\penalty0 (12):\penalty0
  2009, 2009.

\bibitem[Goyal et~al.(2017)Goyal, Doll{\'{a}}r, Girshick, Noordhuis,
  Wesolowski, Kyrola, Tulloch, Jia, and He]{imagenet-1hr}
P.~Goyal, P.~Doll{\'{a}}r, R.~B. Girshick, P.~Noordhuis, L.~Wesolowski,
  A.~Kyrola, A.~Tulloch, Y.~Jia, and K.~He.
\newblock Accurate, large minibatch {SGD:} training imagenet in 1 hour.
\newblock \emph{CoRR}, abs/1706.02677, 2017.
\newblock URL \url{http://arxiv.org/abs/1706.02677}.

\bibitem[Haddadpour and Mahdavi(2019)]{haddadpour_localdescent}
F.~Haddadpour and M.~Mahdavi.
\newblock On the convergence of local descent methods in federated learning.
\newblock \emph{arXiv preprint arXiv:1910.14425}, 2019.

\bibitem[Haddadpour et~al.(2019)Haddadpour, Kamani, Mahdavi, and
  Cadambe]{adaptive_sync_localSGD}
F.~Haddadpour, M.~M. Kamani, M.~Mahdavi, and V.~R. Cadambe.
\newblock Local sgd with periodic averaging: Tighter analysis and adaptive
  synchronization.
\newblock \emph{arXiv preprint arXiv:1910.13598}, 2019.

\bibitem[Hsieh et~al.(2020)Hsieh, Phanishayee, Mutlu, and
  Gibbons]{batchnorm_with_groupnorm}
K.~Hsieh, A.~Phanishayee, O.~Mutlu, and P.~Gibbons.
\newblock The non-iid data quagmire of decentralized machine learning.
\newblock In \emph{International Conference on Machine Learning}, pages
  4387--4398. PMLR, 2020.

\bibitem[Hsu et~al.(2019)Hsu, Qi, and Brown]{fedavgm}
T.-M.~H. Hsu, H.~Qi, and M.~Brown.
\newblock Measuring the effects of non-identical data distribution for
  federated visual classification.
\newblock \emph{arXiv preprint arXiv:1909.06335}, 2019.

\bibitem[Jastrzebski et~al.(2017)Jastrzebski, Kenton, Arpit, Ballas, Fischer,
  Bengio, and Storkey]{three-factors}
S.~Jastrzebski, Z.~Kenton, D.~Arpit, N.~Ballas, A.~Fischer, Y.~Bengio, and
  A.~J. Storkey.
\newblock Three factors influencing minima in {SGD}.
\newblock \emph{CoRR}, abs/1711.04623, 2017.
\newblock URL \url{http://arxiv.org/abs/1711.04623}.

\bibitem[Kairouz et~al.(2019)Kairouz, McMahan, Avent, Bellet, Bennis, Bhagoji,
  Bonawitz, Charles, Cormode, Cummings, et~al.]{fl-survey}
P.~Kairouz, H.~B. McMahan, B.~Avent, A.~Bellet, M.~Bennis, A.~N. Bhagoji,
  K.~Bonawitz, Z.~Charles, G.~Cormode, R.~Cummings, et~al.
\newblock Advances and open problems in federated learning.
\newblock \emph{arXiv preprint arXiv:1912.04977}, 2019.

\bibitem[Kairouz et~al.(2021)Kairouz, McMahan, Song, Thakkar, Thakurta, and
  Xu]{dpftrl}
P.~Kairouz, B.~McMahan, S.~Song, O.~Thakkar, A.~Thakurta, and Z.~Xu.
\newblock Practical and private (deep) learning without sampling or shuffling.
\newblock \emph{arXiv preprint arXiv:2103.00039}, 2021.

\bibitem[Karimireddy et~al.(2020)Karimireddy, Kale, Mohri, Reddi, Stich, and
  Suresh]{SCAFFOLD}
S.~P. Karimireddy, S.~Kale, M.~Mohri, S.~Reddi, S.~Stich, and A.~T. Suresh.
\newblock Scaffold: Stochastic controlled averaging for federated learning.
\newblock In \emph{International Conference on Machine Learning}, pages
  5132--5143. PMLR, 2020.

\bibitem[Karl et~al.(2020)Karl, Takeshita, and Jung]{secagg-sgx}
R.~Karl, J.~Takeshita, and T.~Jung.
\newblock Cryptonite: A framework for flexible time-series secure aggregation
  with online fault tolerance.
\newblock 2020.
\newblock \url{https://eprint.iacr.org/2020/1561}.

\bibitem[Krizhevsky et~al.(2009)Krizhevsky, Hinton, et~al.]{cifar10}
A.~Krizhevsky, G.~Hinton, et~al.
\newblock Learning multiple layers of features from tiny images.
\newblock 2009.

\bibitem[Lam et~al.(2021)Lam, Wei, Brooks, Reddi, and
  Mitzenmacher]{lam2021gradient}
M.~Lam, G.-Y. Wei, D.~Brooks, V.~J. Reddi, and M.~Mitzenmacher.
\newblock Gradient disaggregation: Breaking privacy in federated learning by
  reconstructing the user participant matrix.
\newblock \emph{arXiv preprint arXiv:2106.06089}, 2021.

\bibitem[Leblond et~al.(2017)Leblond, Pedregosa, and
  Lacoste-Julien]{leblond17a_ASAGA}
R.~Leblond, F.~Pedregosa, and S.~Lacoste-Julien.
\newblock {ASAGA: Asynchronous Parallel SAGA}.
\newblock In A.~Singh and J.~Zhu, editors, \emph{Proceedings of the 20th
  International Conference on Artificial Intelligence and Statistics},
  volume~54 of \emph{Proceedings of Machine Learning Research}, pages 46--54,
  Fort Lauderdale, FL, USA, 20--22 Apr 2017. PMLR.
\newblock URL \url{http://proceedings.mlr.press/v54/leblond17a.html}.

\bibitem[Lee et~al.(2017)Lee, Lam, Pedarsani, Papailiopoulos, and
  Ramchandran]{straggler_model}
K.~Lee, M.~Lam, R.~Pedarsani, D.~Papailiopoulos, and K.~Ramchandran.
\newblock Speeding up distributed machine learning using codes.
\newblock \emph{IEEE Transactions on Information Theory}, 64\penalty0
  (3):\penalty0 1514--1529, 2017.

\bibitem[Li et~al.(2018)Li, Sahu, Zaheer, Sanjabi, Talwalkar, and
  Smith]{fedprox}
T.~Li, A.~K. Sahu, M.~Zaheer, M.~Sanjabi, A.~Talwalkar, and V.~Smith.
\newblock Federated optimization in heterogeneous networks.
\newblock \emph{arXiv preprint arXiv:1812.06127}, 2018.

\bibitem[Li et~al.(2019)Li, Yang, Wang, and Zhang]{li2019communication}
X.~Li, W.~Yang, S.~Wang, and Z.~Zhang.
\newblock Communication efficient decentralized training with multiple local
  updates.
\newblock 2019.

\bibitem[Li et~al.(2020)Li, Huang, Yang, Wang, and Zhang]{fedavg_conv_Li}
X.~Li, K.~Huang, W.~Yang, S.~Wang, and Z.~Zhang.
\newblock On the convergence of fedavg on non-iid data.
\newblock 2020.

\bibitem[Li et~al.(2021)Li, Qu, Tang, and Lu]{li2021stragglers}
X.~Li, Z.~Qu, B.~Tang, and Z.~Lu.
\newblock Stragglers are not disaster: A hybrid federated learning algorithm
  with delayed gradients.
\newblock \emph{arXiv preprint arXiv:2102.06329}, 2021.

\bibitem[Lian et~al.(2015)Lian, Huang, Li, and
  Liu]{lian2015asynchronous_nonconvex}
X.~Lian, Y.~Huang, Y.~Li, and J.~Liu.
\newblock Asynchronous parallel stochastic gradient for nonconvex optimization.
\newblock \emph{Advances in neural information processing systems}, 2015.

\bibitem[Lian et~al.(2018)Lian, Zhang, Zhang, and Liu]{lian2018asynchronous}
X.~Lian, W.~Zhang, C.~Zhang, and J.~Liu.
\newblock Asynchronous decentralized parallel stochastic gradient descent.
\newblock In \emph{International Conference on Machine Learning}, pages
  3043--3052. PMLR, 2018.

\bibitem[Lin et~al.(2018)Lin, Stich, Patel, and Jaggi]{dont_use_minibatch}
T.~Lin, S.~U. Stich, K.~K. Patel, and M.~Jaggi.
\newblock Don't use large mini-batches, use local sgd.
\newblock \emph{arXiv preprint arXiv:1808.07217}, 2018.

\bibitem[Liu et~al.(2015)Liu, Luo, Wang, and Tang]{celeba}
Z.~Liu, P.~Luo, X.~Wang, and X.~Tang.
\newblock Deep learning face attributes in the wild.
\newblock In \emph{Proceedings of International Conference on Computer Vision
  (ICCV)}, December 2015.

\bibitem[Mania et~al.(2017)Mania, Pan, Papailiopoulos, Recht, Ramchandran, and
  Jordan]{AsyncSGD_purturbed}
H.~Mania, X.~Pan, D.~Papailiopoulos, B.~Recht, K.~Ramchandran, and M.~I.
  Jordan.
\newblock Perturbed iterate analysis for asynchronous stochastic optimization.
\newblock \emph{SIAM Journal on Optimization}, 27\penalty0 (4):\penalty0
  2202--2229, 2017.

\bibitem[McMahan et~al.(2016)McMahan, Moore, Ramage, and
  y~Arcas]{google-fedavg}
H.~B. McMahan, E.~Moore, D.~Ramage, and B.~A. y~Arcas.
\newblock Federated learning of deep networks using model averaging.
\newblock \emph{arXiv preprint arXiv:1602.05629}, 2016.

\bibitem[McMahan et~al.(2018)McMahan, Ramage, Talwar, and Zhang]{google-dp-fl}
H.~B. McMahan, D.~Ramage, K.~Talwar, and L.~Zhang.
\newblock Learning differentially private recurrent language models.
\newblock In \emph{International Conference on Learning Representations}, 2018.

\bibitem[Melis et~al.(2019)Melis, Song, De~Cristofaro, and
  Shmatikov]{melis2019exploiting}
L.~Melis, C.~Song, E.~De~Cristofaro, and V.~Shmatikov.
\newblock Exploiting unintended feature leakage in collaborative learning.
\newblock In \emph{2019 IEEE Symposium on Security and Privacy (SP)}, pages
  691--706. IEEE, 2019.

\bibitem[Mo et~al.(2021)Mo, Haddadi, Katevas, Marin, Perino, and
  Kourtellis]{tee_ppfl}
F.~Mo, H.~Haddadi, K.~Katevas, E.~Marin, D.~Perino, and N.~Kourtellis.
\newblock Ppfl: privacy-preserving federated learning with trusted execution
  environments.
\newblock \emph{arXiv preprint arXiv:2104.14380}, 2021.

\bibitem[Niu et~al.(2011)Niu, Recht, Re, and Wright]{hogwild}
F.~Niu, B.~Recht, C.~Re, and S.~J. Wright.
\newblock Hogwild!: A lock-free approach to parallelizing stochastic gradient
  descent, 2011.

\bibitem[Ott et~al.(2018)Ott, Edunov, Grangier, and Auli]{scaling_nmt}
M.~Ott, S.~Edunov, D.~Grangier, and M.~Auli.
\newblock Scaling neural machine translation.
\newblock \emph{arXiv preprint arXiv:1806.00187}, 2018.

\bibitem[Paszke et~al.(2017)Paszke, Gross, Chintala, Chanan, Yang, DeVito, Lin,
  Desmaison, Antiga, and Lerer]{pytorch}
A.~Paszke, S.~Gross, S.~Chintala, G.~Chanan, E.~Yang, Z.~DeVito, Z.~Lin,
  A.~Desmaison, L.~Antiga, and A.~Lerer.
\newblock Automatic differentiation in pytorch.
\newblock 2017.

\bibitem[Pennington et~al.()Pennington, Socher, and Manning]{glove}
J.~Pennington, R.~Socher, and C.~D. Manning.
\newblock In \emph{EMNLP}.

\bibitem[Reddi et~al.(2020)Reddi, Charles, Zaheer, Garrett, Rush,
  Kone{\v{c}}n{\`y}, Kumar, and McMahan]{adaptive-fl-optimization}
S.~Reddi, Z.~Charles, M.~Zaheer, Z.~Garrett, K.~Rush, J.~Kone{\v{c}}n{\`y},
  S.~Kumar, and H.~B. McMahan.
\newblock Adaptive federated optimization.
\newblock \emph{arXiv preprint arXiv:2003.00295}, 2020.

\bibitem[Reddi et~al.(2015)Reddi, Hefny, Sra, Poczos, and
  Smola]{reddi2015_ASVRG}
S.~J. Reddi, A.~Hefny, S.~Sra, B.~Poczos, and A.~Smola.
\newblock On variance reduction in stochastic gradient descent and its
  asynchronous variants.
\newblock \emph{Advances in neural information processing systems}, 2015.

\bibitem[Reisizadeh et~al.(2020)Reisizadeh, Tziotis, Hassani, Mokhtari, and
  Pedarsani]{reisizadeh2020straggler}
A.~Reisizadeh, I.~Tziotis, H.~Hassani, A.~Mokhtari, and R.~Pedarsani.
\newblock Straggler-resilient federated learning: Leveraging the interplay
  between statistical accuracy and system heterogeneity.
\newblock \emph{arXiv preprint arXiv:2012.14453}, 2020.

\bibitem[Shallue et~al.(2018)Shallue, Lee, Antognini, Sohl-Dickstein, Frostig,
  and Dahl]{measuring_data_parallelism}
C.~J. Shallue, J.~Lee, J.~Antognini, J.~Sohl-Dickstein, R.~Frostig, and G.~E.
  Dahl.
\newblock Measuring the effects of data parallelism on neural network training.
\newblock \emph{arXiv preprint arXiv:1811.03600}, 2018.

\bibitem[Smith et~al.(2017)Smith, Kindermans, Ying, and Le]{increase_batchsize}
S.~L. Smith, P.-J. Kindermans, C.~Ying, and Q.~V. Le.
\newblock Don't decay the learning rate, increase the batch size.
\newblock \emph{arXiv preprint arXiv:1711.00489}, 2017.

\bibitem[Snoek et~al.(2012)Snoek, Larochelle, and Adams]{bayesian-opt}
J.~Snoek, H.~Larochelle, and R.~P. Adams.
\newblock Practical bayesian optimization of machine learning algorithms.
\newblock \emph{Advances in neural information processing systems},
  25:\penalty0 2951--2959, 2012.

\bibitem[So et~al.(2021{\natexlab{a}})So, Ali, G{\"u}ler, and
  Avestimehr]{so_buffsecagg}
J.~So, R.~E. Ali, B.~G{\"u}ler, and A.~S. Avestimehr.
\newblock Secure aggregation for buffered asynchronous federated learning.
\newblock \emph{arXiv preprint arXiv:2110.02177}, 2021{\natexlab{a}}.

\bibitem[So et~al.(2021{\natexlab{b}})So, G{\"u}ler, and
  Avestimehr]{so2021turbo}
J.~So, B.~G{\"u}ler, and A.~S. Avestimehr.
\newblock Turbo-aggregate: Breaking the quadratic aggregation barrier in secure
  federated learning.
\newblock \emph{IEEE Journal on Selected Areas in Information Theory},
  2\penalty0 (1):\penalty0 479--489, 2021{\natexlab{b}}.

\bibitem[Stich(2019)]{Stitch_LocalSGD}
S.~U. Stich.
\newblock Local sgd converges fast and communicates little.
\newblock 2019.

\bibitem[Tandon et~al.(2017)Tandon, Lei, Dimakis, and
  Karampatziakis]{tandon2017gradient}
R.~Tandon, Q.~Lei, A.~G. Dimakis, and N.~Karampatziakis.
\newblock Gradient coding: Avoiding stragglers in distributed learning.
\newblock In \emph{International Conference on Machine Learning}, pages
  3368--3376. PMLR, 2017.

\bibitem[van Dijk et~al.(2020)van Dijk, Nguyen, Nguyen, Nguyen, Tran-Dinh, and
  Nguyen]{paper3}
M.~van Dijk, N.~V. Nguyen, T.~N. Nguyen, L.~M. Nguyen, Q.~Tran-Dinh, and P.~H.
  Nguyen.
\newblock Asynchronous federated learning with reduced number of rounds and
  with differential privacy from less aggregated gaussian noise.
\newblock \emph{arXiv preprint arXiv:2007.09208}, 2020.

\bibitem[Wang et~al.(2021)Wang, Charles, Xu, Joshi, McMahan, Al-Shedivat,
  Andrew, Avestimehr, Daly, Data, et~al.]{fl_field_guide}
J.~Wang, Z.~Charles, Z.~Xu, G.~Joshi, H.~B. McMahan, M.~Al-Shedivat, G.~Andrew,
  S.~Avestimehr, K.~Daly, D.~Data, et~al.
\newblock A field guide to federated optimization.
\newblock \emph{arXiv preprint arXiv:2107.06917}, 2021.

\bibitem[Watson et~al.(2021)Watson, Guo, Cormode, and
  Sablayrolles]{lauren_alex}
L.~Watson, C.~Guo, G.~Cormode, and A.~Sablayrolles.
\newblock On the importance of difficulty calibration in membership inference
  attacks.
\newblock \emph{CoRR}, abs/2111.08440, 2021.
\newblock URL \url{https://arxiv.org/abs/2111.08440}.

\bibitem[Woodworth et~al.(2020)Woodworth, Patel, Stich, Dai, Bullins, Mcmahan,
  Shamir, and Srebro]{is_local_sgd_better}
B.~Woodworth, K.~K. Patel, S.~Stich, Z.~Dai, B.~Bullins, B.~Mcmahan, O.~Shamir,
  and N.~Srebro.
\newblock Is local sgd better than minibatch sgd?
\newblock In \emph{International Conference on Machine Learning}, pages
  10334--10343. PMLR, 2020.

\bibitem[Wu et~al.(2020)Wu, He, Lin, Mao, Maple, and Jarvis]{safa}
W.~Wu, L.~He, W.~Lin, R.~Mao, C.~Maple, and S.~A. Jarvis.
\newblock Safa: a semi-asynchronous protocol for fast federated learning with
  low overhead.
\newblock \emph{IEEE Transactions on Computers}, 2020.

\bibitem[Wu and He(2018)]{groupnorm}
Y.~Wu and K.~He.
\newblock Group normalization.
\newblock In \emph{Proceedings of the European conference on computer vision
  (ECCV)}, pages 3--19, 2018.

\bibitem[Xie et~al.(2019)Xie, Koyejo, and Gupta]{fedasync}
C.~Xie, S.~Koyejo, and I.~Gupta.
\newblock Asynchronous federated optimization.
\newblock \emph{arXiv preprint arXiv:1903.03934}, 2019.

\bibitem[Yin et~al.(2017)Yin, Pananjady, Lam, Papailiopoulos, Ramchandran, and
  Bartlett]{gradient-diversity}
D.~Yin, A.~Pananjady, M.~Lam, D.~S. Papailiopoulos, K.~Ramchandran, and P.~L.
  Bartlett.
\newblock Gradient diversity empowers distributed learning.
\newblock \emph{CoRR}, abs/1706.05699, 2017.
\newblock URL \url{http://arxiv.org/abs/1706.05699}.

\bibitem[You et~al.(2017)You, Gitman, and Ginsburg]{lars}
Y.~You, I.~Gitman, and B.~Ginsburg.
\newblock Large batch training of convolutional networks.
\newblock \emph{arXiv preprint arXiv:1708.03888}, 2017.

\bibitem[You et~al.(2018)You, Zhang, Hsieh, Demmel, and
  Keutzer]{imagenet_minutes}
Y.~You, Z.~Zhang, C.-J. Hsieh, J.~Demmel, and K.~Keutzer.
\newblock Imagenet training in minutes.
\newblock In \emph{Proceedings of the 47th International Conference on Parallel
  Processing}, pages 1--10, 2018.

\bibitem[You et~al.(2019)You, Li, Reddi, Hseu, Kumar, Bhojanapalli, Song,
  Demmel, Keutzer, and Hsieh]{bert_in_76}
Y.~You, J.~Li, S.~Reddi, J.~Hseu, S.~Kumar, S.~Bhojanapalli, X.~Song,
  J.~Demmel, K.~Keutzer, and C.-J. Hsieh.
\newblock Large batch optimization for deep learning: Training bert in 76
  minutes.
\newblock \emph{arXiv preprint arXiv:1904.00962}, 2019.

\bibitem[Yu et~al.(2019{\natexlab{a}})Yu, Jin, and Yang]{yu2019linear}
H.~Yu, R.~Jin, and S.~Yang.
\newblock On the linear speedup analysis of communication efficient momentum
  sgd for distributed non-convex optimization.
\newblock In \emph{International Conference on Machine Learning}, pages
  7184--7193. PMLR, 2019{\natexlab{a}}.

\bibitem[Yu et~al.(2019{\natexlab{b}})Yu, Yang, and Zhu]{restarted_SGD_Yu}
H.~Yu, S.~Yang, and S.~Zhu.
\newblock Parallel restarted sgd with faster convergence and less
  communication: Demystifying why model averaging works for deep learning.
\newblock In \emph{Proceedings of the AAAI Conference on Artificial
  Intelligence}, volume~33, pages 5693--5700, 2019{\natexlab{b}}.

\bibitem[Yu et~al.(2019{\natexlab{c}})Yu, Li, Raviv, Kalan, Soltanolkotabi, and
  Avestimehr]{yu2019lagrange}
Q.~Yu, S.~Li, N.~Raviv, S.~M.~M. Kalan, M.~Soltanolkotabi, and S.~A.
  Avestimehr.
\newblock Lagrange coded computing: Optimal design for resiliency, security,
  and privacy.
\newblock In \emph{The 22nd International Conference on Artificial Intelligence
  and Statistics}, pages 1215--1225. PMLR, 2019{\natexlab{c}}.

\bibitem[Yu et~al.(2020)Yu, Maddah-Ali, and Avestimehr]{yu2020straggler}
Q.~Yu, M.~A. Maddah-Ali, and A.~S. Avestimehr.
\newblock Straggler mitigation in distributed matrix multiplication:
  Fundamental limits and optimal coding.
\newblock \emph{IEEE Transactions on Information Theory}, 66\penalty0
  (3):\penalty0 1920--1933, 2020.

\bibitem[Zheng et~al.(2017)Zheng, Meng, Wang, Chen, Yu, Ma, and
  Liu]{AsyncSGDDelay}
S.~Zheng, Q.~Meng, T.~Wang, W.~Chen, N.~Yu, Z.-M. Ma, and T.-Y. Liu.
\newblock Asynchronous stochastic gradient descent with delay compensation.
\newblock In \emph{International Conference on Machine Learning}, pages
  4120--4129. PMLR, 2017.

\bibitem[Zhu and Han(2020)]{zhu2020deep}
L.~Zhu and S.~Han.
\newblock Deep leakage from gradients.
\newblock In \emph{Federated learning}, pages 17--31. Springer, 2020.

\end{thebibliography}
